\documentclass[11pt]{article}
\usepackage[left=1.05in,right=1.05in,top=1.15in,bottom=1.15in]{geometry}
\usepackage{mathrsfs}
\usepackage{graphicx}
\usepackage{grffile}
\usepackage{comment}
\usepackage{authblk}

\usepackage[round]{natbib}
\usepackage{hyperref}
\hypersetup{hidelinks}
\usepackage{url}
\usepackage{multirow}
\usepackage{xr}

\usepackage{bbm}
\usepackage{todonotes}
\usepackage{listings}
\usepackage{textgreek} 
\usepackage{amsmath} 
\usepackage{amssymb,amsthm,mathtools}

\usepackage[capitalise]{cleveref}
\usepackage{stfloats} 
\usepackage{float}
\usepackage{placeins} 
\usepackage{xcolor}

\usepackage{tablefootnote} 
\usepackage[roman]{parnotes} 
\usepackage{tabulary}
\usepackage{booktabs}
\usepackage{tabularx}

\usepackage{tikz}

\usetikzlibrary{shapes,decorations,arrows,calc,arrows.meta,fit,positioning}
\tikzset{
    -Latex,auto,node distance =1 cm and 1 cm,semithick,
    state/.style ={ellipse, draw, minimum width = 0.7 cm},
    point/.style = {circle, draw, inner sep=0.04cm,fill,node contents={}},
    bidirected/.style={Latex-Latex,dashed},
    el/.style = {inner sep=2pt, align=left, sloped}
}

\usepackage{changepage}

\usepackage[noend]{algpseudocode}
\definecolor{mygray}{gray}{0.5}
\definecolor{cblue}{RGB}{8, 85, 153}
\definecolor{darkblue}{RGB}{31, 64, 96}

\definecolor{cgreen}{RGB}{8, 153, 83}
\definecolor{green}{RGB}{8, 200, 50}
\definecolor{cmaroon}{RGB}{128, 0, 0}
\algdef{SE}[DOWHILE]{Do}{doWhile}{\algorithmicdo}[1]{\algorithmicwhile\ #1}%

\renewcommand\algorithmicdo{}

\Crefname{section}{Sec}{Secs.}
\Crefname{figure}{Fig}{Figs.}
\Crefname{theorem}{Thm}{Thms.}

\usepackage{amsfonts}

\usepackage{tikzsymbols}
\usepackage[definitionLists,hashEnumerators,smartEllipses,hybrid]{markdown}

\theoremstyle{plain}
\newtheorem{theorem}{Theorem}[section]

\newtheorem{lemma}[theorem]{Lemma}

\newtheorem{proposition}[theorem]{Proposition}

\theoremstyle{definition}
\newtheorem{definition}[theorem]{Definition}
\newtheorem{assumption}[theorem]{Assumption}

\theoremstyle{remark}
\newtheorem{remark}[theorem]{Remark}
\newtheorem{example}[theorem]{Example}

\theoremstyle{plain}

\newcommand{\bx}{\mathbf{x}}
\newcommand{\by}{\mathbf{y}}
\newcommand{\bz}{\mathbf{z}}

\newcommand{\bzero}{\mathbf{0}}

\newcommand{\bX}{\mathbf{X}}

\DeclareMathOperator*{\argmax}{arg\,max}

\DeclarePairedDelimiter{\braces}{\lbrace}{\rbrace}
\DeclarePairedDelimiter{\bracks}{[}{]}
\DeclarePairedDelimiter{\paren}{(}{)}

\DeclarePairedDelimiter{\abs}{|}{|}
\DeclarePairedDelimiter{\sgnorm}{\|}{\|_{\psi_2}}

\DeclarePairedDelimiter{\norm}{\|}{\|}
\DeclarePairedDelimiter{\inprod}{\langle}{\rangle}

\newcommand{\R}{\mathbb{R}}
\renewcommand{\P}{\mathbb{P}}
\newcommand{\E}{\mathbb{E}}

\newcommand{\Var}{\textnormal{Var}}
\newcommand{\Cov}{\textnormal{Cov}}
\DeclareMathOperator{\Vol}{Vol}
\DeclareMathOperator{\len}{len}

\newcommand{\tree}{\mathfrak{T}}

\newcommand{\node}{\mathfrak{t}}

\newcommand{\partfunc}{\mathcal{V}}

\newcommand{\data}[1][n]{\mathcal{D}_{#1}}

\newcommand{\funcclass}{\mathcal{F}}

\newcommand{\1}{\mathbbm{1}}

\newcommand{\sample}{\mathcal{X}}

\newcommand{\Rad}[1]{\mathcal{R}_n(#1)}

\newcommand{\Osc}{\operatorname{Osc}}

\usepackage[ruled,vlined]{algorithm2e}
\usepackage{subcaption}

\usepackage{microtype}
\usepackage[T1]{fontenc}
\usepackage{newtxtext}
\usepackage{newtxmath}

\title{On Stopping Rules and Spatial Adaptation for CART}

\author[1,*]{Zineng Xu}
\author[1,*]{Yuchao Cai}
\author[1]{Yan Shuo Tan}
\affil[1]{Department of Statistics and Data Science, National University of Singapore}
\affil[ ]{\texttt{xuzineng@u.nus.edu, yccai1123@gmail.com, yanshuo@nus.edu.sg}}

\begin{document}
\date{}

\maketitle
\begingroup
\renewcommand{\thefootnote}{\fnsymbol{footnote}}
\footnotetext[1]{Equal contribution.}
\endgroup

\begin{abstract}
    The popular CART algorithm for regression trees combines a greedy splitting rule with a stopping rule, but while the splitting rule has been well studied, the statistical role of stopping rules is less well understood.
    Meanwhile, although regression trees fit using Bayesian methods or via empirical risk minimization (ERM) have been shown to be spatially adaptive to local smoothness and anisotropy, it is unknown whether CART can achieve the same adaptation.
    We address these gaps by proving that, under spatially heterogeneous and anisotropic smoothness and appropriate structural assumptions on the regression function and covariate distribution, CART with the minimum impurity decrease (MID) stopping rule and a suitable threshold achieves pointwise rates that are minimax up to logarithmic factors.
    These rates hold simultaneously over all points in the domain.
    Moreover, we prove that spatial adaptation cannot be achieved under the widely used minimum leaf size stopping rule.
    Together, these results establish a precise statistical role for the MID stopping rule and provide a theoretical basis for the empirical success of CART.
\end{abstract}

\section{Introduction}

Regression trees are piecewise constant models obtained by recursively partitioning the covariate space into rectangular regions, or cells. 
They are among the most popular machine learning models because they combine an interpretable structure with competitive predictive performance \citep{hastie2009elements,rudin2019stop,rudin2022interpretable}.
Since the empirical risk minimization problem over tree partitions is computationally difficult \citep{laurent1976constructing}, most practical tree methods grow trees greedily, splitting each cell by optimizing a local split criterion. 
The most widely used criterion is the \emph{impurity decrease}, which measures the reduction in empirical variance from splitting a cell and gives rise to the CART algorithm \citep{breiman1984classification}.
Equally important to the splitting rule is the stopping or pruning rule, which regularizes the size of the fitted tree.
However, whereas the impurity decrease split rule is nearly universal and has been the subject of extensive theoretical analysis \citep{scornet2016random,klusowski2020sparse,syrgkanis2020estimation,chi2022asymptotic,klusowski2024large,Mazumder2024,cattaneo2025honest,tan2024statistical}, the stopping rule is more variable across implementations and its statistical role is less well understood.

Beyond this gap concerning stopping rules, we emphasize two further limitations of the existing theory.
First, most theoretical work on CART and random forests has focused on global $L^2$ risk rather than pointwise risk at a given query point. 
Pointwise risk is relevant when trees are applied to personalized decision-making problems, especially in high-stakes settings.
\citet{cattaneo2025honest} recently showed that CART can be pointwise inconsistent even when the regression function is constant, but did not address how this behavior interacts with stopping rules or whether it can be avoided under structural assumptions.
Second, much of the CART literature has emphasized global smoothness conditions on the regression function.
By contrast, ERM and Bayesian decision trees have been shown to be spatially adaptive, that is, adaptive to heterogeneous smoothness and anisotropy \citep{donoho1997cart,xu2026statistical,rovckova2024ideal,jeong2023art}.
These works suggest that spatial adaptation may help explain the empirical success of tree-based methods, but they do not address whether greedy CART can achieve the same adaptation, or what role stopping rules play in making it possible.
This leads to the following intertwined questions:

\begin{quote}
    \emph{Can CART attain pointwise rates governed by the local smoothness near the query point? What role does early stopping play in the answer?} 
\end{quote}

In this work, we answer the first question affirmatively and offer a nuanced answer to the second, which we summarize here and make precise in the rest of the paper.
To begin, we focus on a single query point $\bx$ and consider its root-to-leaf path as constructed by the impurity decrease splitting rule.
In choosing where to terminate this path, the stopping rule must achieve two objectives in order to maximize statistical performance at $\bx$:
\begin{enumerate}
    \item[(a)] (Preventing noise-driven splits.) The path must be stopped before the split criterion is dominated by noise, otherwise splits will be unreliable and unhelpful.
    \item[(b)] (Bandwidth selection.) The fitted CART model is a local averaging estimator whose neighborhood around $\bx$ is determined by the terminal cell selected by the stopping rule.
    Hence, the path must stop at a cell that gives the right bias-variance tradeoff for pointwise estimation at each query point.
\end{enumerate}

Since the stochastic fluctuations of the empirical impurity decrease are of essentially uniform order over the cell classes considered (Lemma~\ref{lemma:unif_conc_id}), the minimum impurity decrease (MID) stopping rule, which stops when the best available empirical impurity decrease falls below a threshold, can stop the tree before empirical split selection becomes noise-dominated. This addresses split reliability. 
There is no a priori reason, however, for the scale at which splitting ceases to be reliable to coincide with the scale that optimizes local averaging. The latter is governed by the local smoothness of the regression function and may vary in depth, sample size, and shape across query points. Surprisingly, under the sufficient impurity decrease (SID) condition of \citet{chi2022asymptotic}, these two scales are linked: stopping at the impurity noise floor also selects a terminal cell with the correct local bias-variance tradeoff.



We make this claim rigorous by proving pointwise rates for CART-MID for anisotropic locally H\"older continuous functions in both one-dimensional and sparse high-dimensional settings. 
These rates are minimax optimal up to log factors, hold simultaneously over all points in the domain on the same high-probability event, and do not require knowledge of the local smoothness at each point.
Other stopping rules need not provide the same local adaptation.
To illustrate this, we prove in one dimension that the minimum leaf size (MLS) stopping rule cannot simultaneously select the correct local bandwidth at two points with different smoothness. 
Thus minimum leaf size is statistically mismatched to the local tasks that stopping must perform for pointwise adaptation.

Apart from the SID condition, the assumptions required by our high-dimensional upper bound echo structural conditions provable in the one-dimensional case, but they are not automatically satisfied by multivariate trees because of their increased flexibility.
We therefore view them as relatively mild, albeit not fully intuitive.
We also provide sufficient conditions that are more interpretable and easier to verify, and we give examples of multivariate regression functions satisfying these conditions.

Taken together, our results give a pointwise account of early stopping in CART.
They show that greedy CART can attain locally adaptive pointwise rates, but that the stopping rule is not a secondary implementation detail.
Minimum impurity decrease succeeds because it regularizes the tree in the same empirical scale used for splitting; under our structural assumptions, that detected scale both preserves split reliability and selects the bandwidth needed for local averaging.

\section{Setup and preliminaries}

\subsection{Regression model}
\label{sec:regression-model}

We observe a training dataset $\mathcal D_n=\{(\bX_i,Y_i):1\le i\le n\}$ consisting of independent copies of a random variable $(\bX,Y) \in [0,1]^d\times\mathbb R$. 
We assume the standard regression model
\begin{equation} \label{eq:regression}
Y=f^*(\bX)+\xi,
\end{equation}
where $\E[\xi\mid \bX]=0$. 
Unless stated otherwise, the noise $\xi$ may be heteroskedastic.
Our goal is to study the pointwise behavior of estimators of $f^*$.
To this end, we impose the following standard regularity assumptions on the regression model. To simplify our notation, we will assume for the rest of this paper that $n\geq 2$.



\begin{assumption}[Bounded density] \label{assum:bounded-density}
    The distribution $P$ admits a density $p_X$ on $[0,1]^d$ satisfying
    \begin{equation}
    \label{eq:design-regularity}
    0<p_{\min}\le p_X(\bx)\le p_{\max}<\infty
    \qquad\text{for all }\bx\in[0,1]^d.
\end{equation}
\end{assumption}

\begin{assumption}[Sub-Gaussian response]
    \label{assum:subgaussian-response}
There exists a constant $K<\infty$ such that
\[
\sup_{\bx\in[0,1]^d}\sgnorm{Y\mid \bX=\bx}\le K.
\]
\end{assumption}

\begin{remark}
    If $f^*$ is bounded and the noise $\xi$ is conditionally sub-Gaussian, then Assumption~\ref{assum:subgaussian-response} holds with $K = \|f^*\|_\infty + \sup_{\bx}\sgnorm{\xi\mid \bX=\bx}$.
\end{remark}


\subsection{Notation}
\label{sec:further_notation}

\paragraph{Vectors, random variables, and indexing.}
We use uppercase letters for random variables and lowercase letters for fixed values. Vectors are denoted in boldface and scalars in regular font; for example, $\bx=(x_1,\ldots,x_d)$ and $\bX=(X_1,\ldots,X_d)$. 
For an indexed vector $\bX_i$, write $X_{ij}$ for its $j$th coordinate. 
For an integer $k\ge 1$, write $[k]=\{1,\ldots,k\}$. 
For a set of coordinates $S\subseteq[d]$, write $S^c=[d]\setminus S$ and $\bx_S=(x_j)_{j\in S}$.

\paragraph{Probability and expectation.}
We write $\P$ and $\E$ for probability and expectation under the joint distribution of all random variables in context. 
We use parentheses for function or operator arguments, square brackets for moments, and braces for events and sets. Thus we write $\P(A)$ for the mass of a deterministic cell, $\P\{E\}$ for the probability of an event, and $\E[Z]$, $\Var[Z]$, and $\Cov[Z,W]$ for moments, with conditioning indicated by $\mid$, for example $\E[Z\mid \bX\in A]$. 
For a measurable function $g$, $\|g\|_\infty$ denotes its essential supremum.

\paragraph{Constants and comparison notation.}
We will use $C$ to denote a universal constant (not depending on any parameters) whose value will be allowed to vary from line to line.
Given any two functions of a vector of input parameters ($n$, $d$, etc.), $F$ and $G$, we say that $F \lesssim G$ (equivalently $G \gtrsim F$) if there is a universal constant $C > 0$ such that we have the functional inequality $F \leq CG$.
If $C$ depends on a specific parameter (e.g. $\rho$), we decorate it (or the asymptotic notation) with the parameter as the subscript (e.g. $C_\rho$ or $F \lesssim_\rho G$).
If $F \lesssim G$ and $F \gtrsim G$, we say that $F \asymp G$.

\subsection{Cells}
\label{sec:impurity-decrease}

Let $\mathcal R_d$ denote the collection of axis-aligned rectangles in $[0,1]^d$. We call elements of $\mathcal R_d$ \emph{cells}. For $A=\prod_{j=1}^d A_j\in\mathcal R_d$, write $\len_j(A)$ for the side length of $A$ in coordinate $j$ and $\Vol(A)=\prod_{j=1}^d \len_j(A)$ for its Lebesgue volume. For a coordinate set $S\subseteq[d]$, write $A_S=\prod_{j\in S}A_j$.
For $r\in[d]$, define
\begin{equation}\label{eq:Apd}
    \mathcal{A}_{d,r}
    :=
    \Bigl\{
    A=\prod_{j=1}^d [\ell_j,u_j]\subseteq [0,1]^d
    \;\Big|\;
    \#\bigl\{j\in[d]:[\ell_j,u_j]\neq[0,1]\bigr\}\le r
    \Bigr\}.
\end{equation}
Thus $\mathcal{A}_{d,r}$ consists of cells whose side lengths differ from $[0,1]$ in at most $r$ coordinates. Given a relevant set $S\subseteq[d]$, define the associated relevant cell class
\[
\mathcal A_S
:=
\{A\in\mathcal R_d:A_j=[0,1]\ \text{for all }j\notin S\}.
\]
Then $\mathcal A_S\subseteq\mathcal A_{d,s}$ when $s=|S|$.
For a cell $A\in\mathcal R_d$, let $N(A)=\sum_{i=1}^n\1\{\bX_i\in A\}$. For any measurable function $g$, set $Z_i=g(\bX_i,Y_i)$ and let $\overline Z_A
=
\frac{1}{N(A)}
\sum_{i:\bX_i\in A} Z_i$
when $N(A)>0$, and set $\overline Z_A=0$ when $N(A)=0$.

\subsection{Splits and impurity decrease}

For a coordinate $j\in[d]$ and a threshold $b\in[0,1]$, the \emph{split} $(A,j,b)$ partitions $A$ into two children cells and
\[
A_L=\{\bx\in A:x_j\le b\},
\qquad
A_R=\{\bx\in A:x_j>b\}.
\]
We call the split \emph{empirically admissible} if both children contain
at least one observation.
For an empirically admissible split, the \emph{empirical impurity decrease}
is the weighted reduction in variance of the responses from making the split:
\begin{align}
\label{eq:empirical-impurity-decrease}
\widehat\Delta(A,j,b)
&=
\frac{1}{n}\biggl[
\sum_{i:\bX_i\in A}(Y_i-\overline Y_A)^2
-
\sum_{i:\bX_i\in A_L}(Y_i-\overline Y_{A_L})^2
-
\sum_{i:\bX_i\in A_R}(Y_i-\overline Y_{A_R})^2
\biggr].
\end{align}
Equivalently,
\[
\widehat\Delta(A,j,b)
=
\frac{N(A_L)N(A_R)}{N(A)n}
\bigl(\overline Y_{A_L}-\overline Y_{A_R}\bigr)^2.
\]
For a split that is not empirically admissible, we set
$\widehat\Delta(A,j,b)=0$.

The \emph{population impurity decrease} of the split is the corresponding weighted reduction in conditional variance of the regression function:
\begin{equation}
\label{eq:population-impurity-decrease}
\Delta(A,j,b)
=
V(A) - V(A_L) - V(A_R),
\end{equation}
where we call
\begin{equation} \label{eq:residual-variation}
    V(A)=\P(A)\Var[f^*(\bX)\mid\bX\in A]
\end{equation}
the \emph{residual variation} in cell $A$, with $V(A)=0$ when
$\P(A)=0$.
Empirical and population impurity decreases can also be interpreted as the reduction in training error and expected test error respectively from making the split.
Note that while both quantities depend on $f^*$ and $P$, this is suppressed in the notation. 
We also define the following split related quantities:
\[
\widehat\Delta_{\max}(A,j)
:=
\sup_{b\in[0,1]}\widehat\Delta(A,j,b),
\qquad
\widehat\Delta_{\max}(A)
:=
\max_{j\in[d]}\widehat\Delta_{\max}(A,j),
\]
so $\widehat\Delta_{\max}(A,j)=0$ if no empirically admissible split is
available along coordinate $j$. Their
population counterparts are
\[
\Delta_{\max}(A,j)
:=
\sup_{b\in[0,1]}\Delta(A,j,b),
\qquad
\Delta_{\max}(A)
:=
\max_{j\in[d]}\Delta_{\max}(A,j).
\]
We refer to $\Delta_{\max}(A)$ as the \emph{split signal} in cell $A$, and to $\widehat\Delta_{\max}(A)$ as the \emph{empirical split signal}.
We say that a split $(A,j,b)$ is \emph{competitive} if\footnote{The choice of $1/2$ in the definition is arbitrary; any choice of constant in $(0,1)$ will lead to the same statements for the results in our paper, just with different constants.}
\begin{equation}
\label{eq:competitive-split}
\Delta(A,j,b)\ge \frac{1}{2}\,\Delta_{\max}(A).
\end{equation}
Lastly, for a coordinate set $S\subseteq[d]$, write
\[
\Delta_{\max}(A,S):=\max_{j\in S}\Delta_{\max}(A,j),
\qquad
\Delta_{\max}(A,S^c):=\max_{j\in S^c}\Delta_{\max}(A,j).
\]
We use the convention $\Delta_{\max}(A,\varnothing)=0$.

\subsection{CART and stopping rules}
\label{section:early-CART}

CART is a two-stage algorithm for fitting a regression tree model. 
In the first stage, a binary tree partition $\tree = \lbrace A_1,A_2,\ldots, A_M \rbrace$ of $[0,1]^d$ is constructed via recursive partitioning:
Starting with the whole covariate space as the root node of the tree, each cell/node $A$ is split by maximizing \eqref{eq:empirical-impurity-decrease}, i.e. via selecting
\begin{equation}
\label{equ::jstar}
(\hat j,\hat b)
\in
\argmax_{\substack{j\in[d],\,b\in[0,1]\\
N(A_L)>0,\,N(A_R)>0}}
\widehat\Delta(A,j,b),
\end{equation}
unless a stopping rule declares $A$ terminal or no empirically admissible
split exists. In the second stage, the estimator predicts at a query point $\bx$ by averaging all responses in the terminal cell $A(\bx)$ (called a \emph{leaf}) containing $\bx$:
\[
\hat f(\bx) = \overline Y_{A(\bx)}.
\]
The terminal cell thus acts as a data-adaptive local bandwidth for pointwise estimation.

We compare two stopping rules.

\paragraph{CART-MID$(\delta)$.}
Given $\delta>0$, CART with minimum impurity decrease (MID) stopping rule declares $A$ terminal when
\[
\widehat\Delta_{\max}(A)<\delta.
\]
The resulting estimator is denoted by $\hat f_\delta$.

\paragraph{CART-MLS$(N)$.}
Given $N\in\mathbb N$, CART with minimum leaf size (MLS) stopping rule permits only splits for which both children contain at least $N$ observations. If no such admissible split is available, the cell is terminal. The resulting estimator is denoted by $\hat f_N$.\footnote{This corresponds to setting \texttt{min\_samples\_leaf=N} in the CART implementation in the \texttt{scikit-learn} package \citep{pedregosa2011scikit}.}


\begin{remark}[Honesty]
    While in practice the tree construction and leaf-based averaging use the same training sample $\mathcal D_n$, some results in the literature, as well as Theorem~\ref{thm::lowerright} in our work, use an \emph{honest} version of CART, in which one sample constructs the tree partition and an independent sample of the same size estimates the leaf averages \citep{athey2016recursive}. Honesty decouples the randomness of the two stages and is especially useful for lower-bound arguments \citep{tan2022cautionary,cattaneo2025honest}.
    The results in this paper will focus mostly on the regular, non-honest version of CART.
    Where we use honesty, we will explicitly state it, and assume that independent copies of $\data$ are used for each of the two stages.
    If an honest averaging leaf is empty, we set its fitted value to zero. This convention only makes the estimator globally well-defined and does not affect our results, whose proofs work on events where the relevant leaves are nonempty.
\end{remark}

\subsection{Concentration for cell averages and impurity decrease}
\label{section:unif-conc}

The analysis uses two main concentration inequalities, which are proved in Appendix~\ref{app:concentration-tools}. The first controls the empirical impurity decrease values used in the greedy splitting rule, while the second controls the leaf averages that appear in the averaging stage of CART.

\begin{lemma}[Uniform concentration of impurity decrease]
\label{lemma:unif_conc_id}
Assume sub-Gaussian response (Assumption~\ref{assum:subgaussian-response}).
Then there exists a universal constant $C>0$ such that, for any $u>0$, with probability at least $1-e^{-u}$, the following inequality holds uniformly for every $r\in [d]$, $A\in\mathcal{A}_{d,r}$, $j\in[d]$, and $b\in[0,1]$:
\begin{align}
\bigl|\Delta(A,j,b)^{1/2}
-\widehat\Delta(A,j,b)^{1/2}\bigr|\lesssim K
\paren*{\frac{r\log(nd)+u}{n}}^{1/2}. \label{eq:id-conc-conclude-1-1}
\end{align}
\end{lemma}

\begin{lemma}[Uniform concentration of cell averages]
\label{lemma:unif_conc_cell_ave}
Assume sub-Gaussian response (Assumption~\ref{assum:subgaussian-response}). Then there exist universal constants $C,c>0$ such that, for any $u>0$, with probability at least $1-e^{-u}$, the following inequality holds uniformly for every $r\in[d]$ and every $A\in\mathcal{A}_{d,r}$ satisfying either $A=[0,1]^d$ or $\P(A)\ge c(r\log(nd)+u)/n$:
\begin{equation}
\label{eq:conclude-cell-mean}
\P(A)^{1/2}
\bigl|\overline Y_A-\E[f^*(\bX)\mid \bX\in A]\bigr|
\lesssim
K
\left(\frac{r\log(nd)+u}{n}\right)^{1/2},
\end{equation}
where $\overline Y_A=N(A)^{-1}\sum_{i:\bX_i\in A} Y_i$.
\end{lemma}

Because of the mass weighting in the definition of impurity decrease (cf. \eqref{eq:empirical-impurity-decrease} and \eqref{eq:population-impurity-decrease}), the bound on the right-hand side of \eqref{eq:id-conc-conclude-1-1} is independent of the cell size, depth, or any other cell characteristics. 
In other words, the stochastic fluctuation level for $\widehat\Delta(A,j,b)$ is of order
\begin{equation} \label{eq:noise-floor}
    K^2\frac{r\log(nd)+u}{n}
\end{equation}
uniformly over cells in $\mathcal{A}_{d,r}$ and candidate splits. 
Generally, CART may select irrelevant features or inappropriate thresholds if its impurity decrease values are dominated by noise.
We may avoid unreliable splits by stopping once the empirical split signal $\widehat\Delta_{\max}(A)$ drops below this noise floor, such as via CART-MID.
As shown by the following lemma, this also ensures that the children of the split have enough mass for the cell-average concentration bound in Lemma~\ref{lemma:unif_conc_cell_ave} to apply.

\begin{lemma}[Impurity-certified child mass]
\label{lem:mid-child-mass}
Assume sub-Gaussian response (Assumption~\ref{assum:subgaussian-response}).
For every fixed $c>0$, there is a constant $C>0$ depending only on $c$ such
that, for any $u>0$, with probability at least $1-e^{-u}$, the following
holds simultaneously for every $r\in[d]$ and every split $(A,j,b)$ with
$A\in\mathcal A_{d,r}$. If
\begin{equation}\label{eq:mid-child-mass-threshold}
\widehat\Delta(A,j,b)
\ge CK^2\frac{r\log(nd)+u}{n},
\end{equation}
then both children $B\in\{A_L,A_R\}$ satisfy
\begin{equation}\label{eq:mid-child-mass-conclusion}
\P(B)\ge c\frac{r\log(nd)+u}{n}.
\end{equation}
\end{lemma}


\begin{remark}[Related work]
The uniform concentration inequalities are based on technical results in \citet{xu2026statistical}.
They do not require bounded-density assumptions and hold uniformly over all cells in $[0,1]^d$.
In these and other respects, they sharpen related concentration results in \citet{Mazumder2024} for cell means, impurity decreases, and cell counts; see, in particular, Lemma~B.5 of that paper.

\end{remark}
\subsection{Sufficient impurity decrease}
\label{sec:shared-assumptions}

The sufficient impurity decrease (SID) condition was introduced by
\citet{chi2022asymptotic} in their consistency analysis of high-dimensional CART
and random forests. In our notation, it can be stated as follows.

\begin{assumption}[Sufficient impurity decrease]
\label{assum:sid}
Given a collection of cells $\mathcal A\subseteq\mathcal R_d$, there exists
$\lambda>0$ such that, for every $A \in\mathcal A$,
\begin{align}
\label{eq:sid}
\Delta_{\max}(A)
\ge
\lambda\,\P(A)\,
\Var[f^*(\bX)\mid\bX\in A].
\end{align}
\end{assumption}

SID says that every cell has an axis-aligned split that can remove a non-negligible fraction of that variation in that cell.
Thus SID is a compatibility condition between the regression
function and the greedy CART splitting rule, ruling out adversarial settings such as pure interactions under which CART is known to fail \citep{tan2024statistical}.

\citet{chi2022asymptotic} apply SID in a depth-wise contraction
argument. 
At the population level, if every cell $A$ is split by the maximizer of $\Delta_{\max}(A)$, then Assumption~\ref{assum:sid} implies that the
tree approximation error
\[
\sum_{A\in \tree}\P(A)\Var[f^*(\bX)\mid \bX\in A]
\]
contracts geometrically with tree depth with contraction factor $1 - \lambda$. 
Uniform concentration of empirical
impurity decreases then shows that empirical CART approximately inherits this
population contraction and is furthermore used to bound the variance term. 
\citet{Mazumder2024} later refined this analysis via tighter control on the uniform concentration.

While this contraction perspective is powerful for proving global consistency, the resulting $L^2$ estimation error rates depend directly on the SID constant: when $\lambda$ is small, the provable geometric contraction is slow and
the rate exponent becomes pessimistic. 
Our use of SID is different. Rather than using it to prove a global depth-wise contraction, we use it locally along the root-to-leaf path of a query point. Combined with MID stopping, SID converts a bound on the empirical split signal into a bound on the residual variation inside the terminal cell.

\begin{remark}
    While \citet{tan2024statistical} show that some conditions are necessary in order for CART to be consistent, SID is not the only route for obtaining consistency results.
    \citet{scornet2016random,klusowski2020sparse,klusowski2024large} do so by assuming additive functional structure, while \citet{syrgkanis2020estimation} show consistency over binary features via a modularity-type assumption.
\end{remark}

\subsection{Spatial adaptation}

There are two related, but distinct, ways to formalize spatial adaptation. 
One is through global losses over spatially inhomogeneous function classes, such as Besov, piecewise H\"older, or anisotropic classes.
Wavelet shrinkage \citep{donoho1994ideal,donoho1995wavelet} and, more recently, Bayesian tree posteriors \citep{jeong2023art} and ERM decision trees \citep{xu2026statistical} have been shown to achieve optimal rates in this global sense.
The second formulation is via pointwise estimation rates that vary with the local regularity at the query point. 
This pointwise perspective is classical for kernel and local-polynomial estimators, which attain pointwise rates by choosing a local bandwidth, with adaptive procedures selecting this bandwidth without knowing the smoothness in advance \citep{hoffman2002random,lepski2015adaptive}. 
\citet{rovckova2024ideal} showed that the Bayesian CART posterior could achieve similarly optimal pointwise performance.

The pointwise perspective is especially natural for CART because of its nature as an adaptive local averaging estimator.
On the other hand, our setting is different from classical bandwidth selection in kernel methods because CART does not choose a bandwidth directly.
Instead, the local neighborhood is the terminal cell produced by greedy splitting and potentially myopic early stopping. 
In our work, we formalize local spatial heterogeneity and anisotropy as follows.

\begin{definition}[Local anisotropic H\"older regularity]
\label{def::localholder}
Let $\alpha_j:[0,1]\to(0,1]$, $j\in[d]$, be measurable functions satisfying
$0<\alpha_{\min}\le\alpha_j(t)\le\alpha_{\max}\le1$ for all $j$ and $t$.
Let $L\ge1$ be a H\"older radius.
We say that a bounded regression function $f^*$ is locally anisotropic H\"older with exponents $\{\alpha_j\}_{j=1}^d$ and radius $L$ if, for all $\bx,\bx'\in[0,1]^d$,
\[
|f^*(\bx)-f^*(\bx')|
\le
L\sum_{j=1}^d |x_j-x_j'|^{\alpha_j(x_j)}.
\]
\end{definition}

\begin{remark}[Locality radii]\label{rmk:local-holder}
More generally, one may require the bound in Definition~\ref{def::localholder}
only for pairs $\bx,\bx'$ satisfying $|x_j-x_j'|<\gamma_j(x_j)$ for all
$j\in[d]$, where $\gamma_j:[0,1]\to(0,1]$ are locality radius functions
(Definition~\ref{def:local-holder-general}). Under mild conditions, however,
the radii can be taken to be one at the price of a constant factor in $L$
(Lemma~\ref{lem:radius-one} in Appendix~\ref{appendix:proof-of-holder}), so we
use Definition~\ref{def::localholder} throughout.
\end{remark}

\begin{remark}[Further extensions]
The proofs in our work allow coordinate-wise constants $L_j$, and furthermore for the constants and not just the exponents to be locally varying. We use a single constant $L$ in the main statements to keep the rates and structural assumptions readable.
\end{remark}


\section{Warm-up: 1D pointwise spatial adaptation for CART-MID}
\label{sec:one-dimensional-upper}



In this section, we give a detailed account of how CART-MID achieves pointwise adaptation in one dimension.
In one dimension, the terminal cell containing a query point $x$ is an interval whose length $h$ acts as the local bandwidth for leaf averaging.
By standard calculations \citep{tsybakov2008introduction}, for a function $f^*$ with local H\"older smoothness $\alpha(x)$, the averaging bias is of order $h^{\alpha(x)}$, while the stochastic error is of order $(nh)^{-1/2}$.\footnote{For notational simplicity, we ignore all dependence on the parameters $L$, $p_{\min}$, and $p_{\max}$ in this informal discussion. We also slightly abuse notation and denote $x=x_1$ and $\alpha = \alpha_1$.} Balancing the two terms gives the optimal bandwidth choice
\[
h_x\asymp n^{-1/(2\alpha(x)+1)}
\]
and the pointwise rate
\[
n^{-\alpha(x)/(2\alpha(x)+1)}.
\]
Unlike regressogram or local kernel estimators, however, CART does not choose $h$ directly. It constructs the interval through greedy empirical splits and its choice of stopping condition, and therefore it cannot be taken for granted that CART can find the correct bandwidth at a given query point, let alone uniformly over the entire domain.

We now argue heuristically how CART-MID can indeed find the correct bandwidth without knowing $\alpha(x)$.
For an interval $A$ of length $h$ near $x$, bounded density gives $\P(A)\asymp h$, and local H\"older smoothness upper bounds the conditional variance of $f^*$ on $A$. 
Together, this gives the residual variation bound (cf. \eqref{eq:residual-variation}):
\[
V(A)\lesssim h^{2\alpha(x)+1}.
\]
Meanwhile, one can show that this bound is sharp over the local H\"older class.
In other words, the bandwidth is in one-to-one correspondence with the worst case residual variation scale, with the optimal bandwidth $h_x$ corresponding to the residual variation value
\[
h_x^{2\alpha(x)+1} \asymp \frac{1}{n}.
\]
Notably, this value is independent of $x$.

While CART-MID uses $\widehat\Delta_{\max}(A)$ rather than $V(A)$ to decide when to stop, the SID condition links the two quantities up to constant factors whenever they are larger than the noise floor \eqref{eq:noise-floor}.
As such, if we set the MID threshold $\delta$ to be roughly the noise floor value, CART-MID is able to achieve the correct bandwidth scale at every query point, up to log factors.
The precise statement of this claim is the main theorem of this section:




\begin{theorem}[One-dimensional pointwise upper bound for CART-MID]
\label{thm:cart-mid-one-dimensional}
Assume the regression model in Section~\ref{sec:regression-model}. Assume bounded density (Assumption~\ref{assum:bounded-density}) and sub-Gaussian response (Assumption~\ref{assum:subgaussian-response}). Suppose $f^*$ is locally H\"older as in Definition~\ref{def::localholder} with $d=1$, and suppose sufficient impurity decrease (Assumption~\ref{assum:sid}) holds on intervals. 
Let $u\ge0$ and choose
\[
\delta
=
C
K^2
\frac{\log n+u}{n},
\]
where $C>0$ is a sufficiently large universal constant. Assume that this threshold satisfies $\delta\le1$.
Then, with probability at least $1-e^{-u}$, uniformly for all $x\in[0,1]$, CART-MID$(\delta)$ satisfies
\begin{equation}\label{eq:main-thm-point-eid-cart-1}
    |\hat f_\delta(x)-f^*(x)|
\lesssim_{p_{\min},p_{\max}} 
\left(
L^{1/\alpha(x)}
K^2
\frac{\log n+u}{\lambda n}
\right)^{
\alpha(x)/(2\alpha(x)+1)
}.
\end{equation}
\end{theorem}



\subsection{Proof of Theorem~\ref{thm:cart-mid-one-dimensional}}
\label{sec:one-dimensional-upper-proof-strategy}

We first state two lemmas that address the technical difficulties in making the earlier heuristic argument rigorous. 
Their proofs are given in Appendix~\ref{sec:one-dimensional-technical-lemmas}. 

\begin{lemma}[Competitive split-scale regularity]
\label{lemma:population-good-split-balance-1d}
Assume the conditions of Theorem~\ref{thm:cart-mid-one-dimensional}. Let $A'\subseteq[0,1]$ be an interval, and consider a split of $A'$ into two children at a threshold $b$. 
Denote either child by $A$ and fix $x \in A$.
Suppose the split $(A',1,b)$ is competitive in the sense of \eqref{eq:competitive-split}. Then
\[
\Vol(A)
\gtrsim_{p_{\min},p_{\max}}
\lambda^{2\alpha(x)/(1+2\alpha(x))}
\bigl(L^{-2}\Delta(A',1,b)\bigr)^{1/(1+2\alpha(x))}.
\]
\end{lemma}

Although we can easily translate a lower bound on the residual variation in $A$ into a lower bound on its length, the problem is that MID gives a lower bound on the parent split signal. The first lemma shows that if the parent split is competitive, then the child interval containing $x$ cannot be too small relative to that split signal.
In other words, it is used to guarantee that the terminal split for CART-MID is sufficiently regular.

\begin{lemma}[Bias visibility]
\label{lemma:LH-upper-1d}
Assume the regression model in Section~\ref{sec:regression-model}, including bounded density (Assumption~\ref{assum:bounded-density}). Suppose $f^*$ is locally H\"older as in Definition~\ref{def::localholder} with $d=1$.
Then, for every interval $A$ containing $x$,
\begin{equation}
\label{eq:LH-upper-1d}
\bigl|f^*(x)-\E[f^*(X)\mid X\in A]\bigr|
\lesssim_{p_{\min},p_{\max}}
\left(L^{1/\alpha(x)}V(A)\right)^{\alpha(x)/(1+2\alpha(x))}.
\end{equation}
\end{lemma}

The second lemma shows how the pointwise bias can be controlled by the residual variation on an interval, which we call the bias visibility property.
This is a more precise form of the bound $h \lesssim V(A)^{1/(1+2\alpha(x))}$ described in the heuristic argument and is used to turn the upper bound on $\Delta_{\max}(A)$ from the MID stopping criterion into a bound on the bias of the terminal interval.

\begin{proof}[Proof of Theorem~\ref{thm:cart-mid-one-dimensional}]
Let $u_0:=u+\log 3$.
Define
\[
\varepsilon_n
:=
K
\paren*{\frac{\log n+u}{n}}^{1/2},
\qquad
\delta=C\varepsilon_n^2,
\]
where $C$ is the universal constant in the theorem. Fix $x\in[0,1]$,
and let $A=A(x)$ denote the terminal interval containing $x$. We begin
with the decomposition
\[
|\hat f_\delta(x)-f^*(x)|
\le
\bigl|\overline Y_A-\E[f^*(X)\mid X\in A]\bigr|
+
\bigl|\E[f^*(X)\mid X\in A]-f^*(x)\bigr|.
\]

\paragraph{Step 1: Concentration.}
Apply Lemmas~\ref{lemma:unif_conc_id},
\ref{lemma:unif_conc_cell_ave}
and~\ref{lem:mid-child-mass} with confidence parameter $u_0$, and
consider the event on which all three conclusions hold. In
Lemma~\ref{lem:mid-child-mass}, choose $c$ to be the constant required
by Lemma~\ref{lemma:unif_conc_cell_ave}; the constant $C$ in $\delta$
can be chosen large enough for the resulting impurity threshold.
A union bound shows that this common event has probability at least
$1-e^{-u}$. Since $u_0$ differs from $u$ only by a constant, its
contribution can be absorbed into the universal constants below. On
this event,
\begin{equation} \label{eq:1d-imp_dec_conc}
\sup_{A,b}
\bigl|\widehat\Delta(A,1,b)^{1/2}-\Delta(A,1,b)^{1/2}\bigr|
\lesssim \varepsilon_n
\end{equation}
and, for every interval $A$ with $\P(A)$ above a sufficiently large multiple of $(\log n+u_0)/n$,
\begin{equation} \label{eq:1d-cell-ave-conc}
    \bigl|\overline Y_A-\E[f^*(X)\mid X\in A]\bigr|
\lesssim
\frac{\varepsilon_n}{\P(A)^{1/2}}.
\end{equation}

\paragraph{Step 2: The parent split signal controls the leaf length.}
If $A$ is the root cell, the desired lower bound on its length is
immediate. Otherwise, let $A'$ be the parent of $A$. Let $b^*$ maximize
the population impurity decrease on $A'$, and let $\hat b$ maximize its
empirical counterpart. Since the algorithm split $A'$, the MID
criterion gives
\[
\widehat\Delta(A',1,\hat b)\ge\delta.
\]
Recall that
\[
\Delta_{\max}(A)=\sup_{b\in[0,1]}\Delta(A,1,b)
\quad\text{and}\quad
\widehat\Delta_{\max}(A)=
\sup_{b\in[0,1]}\widehat\Delta(A,1,b).
\]
Because \eqref{eq:1d-imp_dec_conc} holds uniformly over $b$, it also
holds for the maxima:
\begin{align}\label{eq::conc-max-id}
\left|
\Delta_{\max}(A)^{1/2}
-
\widehat\Delta_{\max}(A)^{1/2}
\right|
\lesssim \varepsilon_n .
\end{align}
Uniform concentration and the empirical optimality of $\hat b$ show
that passing from the population maximizer to $\hat b$ loses at most a
universal multiple of $\varepsilon_n$ on the square-root scale. The MID
criterion likewise gives
$\Delta(A',1,\hat b)^{1/2}\ge\delta^{1/2}-O(\varepsilon_n)$.
Since $\delta=C\varepsilon_n^2$, choosing $C$ sufficiently large therefore
gives
\[
\Delta(A',1,\hat b)
\ge \frac{1}{2}\Delta_{\max}(A'),
\qquad
\Delta(A',1,\hat b)\gtrsim\delta.
\]
Thus the split at $\hat b$ is competitive, and
Lemma~\ref{lemma:population-good-split-balance-1d} yields
\begin{equation} \label{eq:1d_step2}
\Vol(A)
\gtrsim_{p_{\min},p_{\max}}
\lambda^{2\alpha(x)/(1+2\alpha(x))}
L^{-2/(1+2\alpha(x))}
\delta^{1/(1+2\alpha(x))}.
\end{equation}

\paragraph{Step 3: Control the leaf-average error.}
If $A$ is the root cell, Lemma~\ref{lemma:unif_conc_cell_ave} gives
\eqref{eq:1d-cell-ave-conc} directly. Otherwise, $A$ was created by a split
whose empirical impurity decrease is at least $\delta$. Hence
Lemma~\ref{lem:mid-child-mass} gives
\[
\P(A)\ge c\frac{\log n+u_0}{n}
\]
with $c$ chosen as in Step 1. Consequently,
Lemma~\ref{lemma:unif_conc_cell_ave}, and thus
\eqref{eq:1d-cell-ave-conc}, applies to $A$. Moreover,
Assumption~\ref{assum:bounded-density} implies
$\P(A) \geq p_{\min}\Vol(A)$.
Combining this inequality with \eqref{eq:1d-cell-ave-conc} and
\eqref{eq:1d_step2} yields
\[
\bigl|\overline Y_A-\E[f^*(X)\mid X\in A]\bigr|
\lesssim_{p_{\min},p_{\max}}
\lambda^{-\alpha(x)/(1+2\alpha(x))}
L^{1/(1+2\alpha(x))}
\delta^{\alpha(x)/(1+2\alpha(x))},
\]
where we also used $\delta=C\varepsilon_n^2$.

\paragraph{Step 4: Control the residual variation.}
Because $A$ is terminal, $\widehat\Delta_{\max}(A)<\delta$.
Equation~\eqref{eq::conc-max-id} therefore gives
$\Delta_{\max}(A)\lesssim\delta$. By SID,
\[
V(A)\leq \frac{\Delta_{\max}(A)}{\lambda}
\lesssim \frac{\delta}{\lambda}.
\]

\paragraph{Step 5: Control the bias.}
Applying Lemma~\ref{lemma:LH-upper-1d} to $A$ and using the bound from
Step 4 gives
\[
\bigl|\E[f^*(X)\mid X\in A]-f^*(x)\bigr|
\lesssim_{p_{\min},p_{\max}}
\lambda^{-\alpha(x)/(1+2\alpha(x))}
L^{1/(1+2\alpha(x))}
\delta^{\alpha(x)/(1+2\alpha(x))}.
\]
Combining the leaf-average and bias bounds gives
\[
|\hat f_\delta(x)-f^*(x)|
\lesssim_{p_{\min},p_{\max}}
\lambda^{-\alpha(x)/(1+2\alpha(x))}
L^{1/(1+2\alpha(x))}
\delta^{\alpha(x)/(1+2\alpha(x))}.
\]
Substituting the chosen value of $\delta$ proves
\eqref{eq:main-thm-point-eid-cart-1}. Since the concentration event is
uniform over all intervals and splits, the conclusion holds
simultaneously for every $x\in[0,1]$.
\end{proof}

\begin{remark}[Local proof mechanism]
\label{remark:local-proof-mechanism}
The proof uses MID and SID only along the root-to-leaf path containing $x$. MID supplies two local certificates: the terminal cell has small empirical split signal, whereas its parent has large empirical split signal. After the uniform empirical-population comparison, the first certificate and SID give a bound on the residual variation in the terminal cell. The second certificate, together with competitiveness and Lemma~\ref{lemma:population-good-split-balance-1d}, gives a lower bound on the mass of that cell. These are precisely the local bias and leaf-average variance controls needed for pointwise estimation. Thus, the rate is determined by a local bias-variance balance rather than by iterating a contraction over the entire tree (see also Remark~\ref{rem:comparison-with-previous-work}).
\end{remark}

\begin{remark}[Unbalanced splits] \label{remark:unbalanced-splits}
Unlike previous analyses of CART, the proof does not enforce CART to make balanced splits. Lemma~\ref{lemma:population-good-split-balance-1d} shows only that, when a split is competitive, the child containing the query point cannot be too small relative to the residual variation in its parent. This local lower bound is the only balance property used in the proof.
\end{remark}

\section{Pointwise spatial adaptation for CART-MID}
\label{section:multidim-results}
\label{section:upper-bound}




We now state the general version of our pointwise upper bound, which applies to CART-MID in sparse high-dimensional settings.
To formalize this setting, we assume there exists a nonempty set $S\subseteq[d]$ with $|S|=s$ (so $1\le s\le d$) and a function $f_S^*:[0,1]^s\to\mathbb R$ such that
\begin{equation} \label{eq:sparsity-set}
f^*(\bx)=f_S^*(\bx_S).
\end{equation}
The relevant aggregate smoothness at $\bx$ is the harmonic mean over the relevant coordinates:
\[
\bar\alpha_S(\bx)
:=
\left(
\frac{1}{s}\sum_{j\in S}\frac{1}{\alpha_j(x_j)}
\right)^{-1}.
\]

Even without high-dimensionality, achieving an optimal rate in multiple dimensions is more demanding than in one dimension because the terminal cell must have not only the right mass but also the right shape: under anisotropic smoothness, rougher coordinates should be refined more than smoother coordinates.
It turns out that the impurity decrease splitting rule is indeed able to perform this adaptive refinement at a local scale.
On the other hand, the bias visibility and competitive split-scale regularity conditions that were provable in one dimension no longer hold automatically in multiple dimensions.
We therefore first state multidimensional analogues of these two lemmas as structural assumptions.


\begin{assumption}[Pathwise competitive split-scale regularity]
\label{assum:reliable-split-balance}
There exists a constant $\omega\in(0,1]$ such that the following holds for every nested path of relevant cells
\[
A^0\supset A^1\supset\cdots\supset A^D=A
\]
with $D\ge1$ and for every point $\bx\in A$.
Suppose that each nonterminal transition $A^i\to A^{i+1}$ is generated by a competitive split $(A^i,j_i,b_i)$ along a relevant coordinate. Write
\[
\Delta_{\mathrm{pa}}
:=
\Delta(A^{D-1},j_{D-1},b_{D-1})
\]
for the split signal at the terminal parent. Then, for every coordinate $j$ split somewhere along the path,
\begin{align}
\label{eq:reliable-split-scale}
\len_j(A)
\ge
\omega\,
L^{-1/\alpha_j(x_j)}
\left(
\Vol(A)^{-1}\Delta_{\mathrm{pa}}
\right)^{1/(2\alpha_j(x_j))},
\end{align}
where $\omega$ does not depend on the path or $\bx$.
\end{assumption}

When $d=1$, the last split in the only coordinate is the terminal-parent split, $\len_1(A)=\Vol(A)$, and \eqref{eq:reliable-split-scale} has the same form as Lemma~\ref{lemma:population-good-split-balance-1d}.
For $d>1$, the last split in a coordinate may occur strictly above the terminal leaf, and later splits in other coordinates shrink the leaf volume. The assumption therefore must be stated pathwise.
To see how this implies a lower bound on the terminal cell volume, apply the assumption to each coordinate split along the path and multiply the resulting inequalities together.
In the regime used below, after simplifying exponents, this gives the bound
\[
\Vol(A)
\geq
\omega^{2s/(s+2)}
\paren*{L^{-2}\Delta_{\mathrm{pa}}}^{
s/(2\bar\alpha_S(\bx)+s)}.
\]

\begin{assumption}[Bias visibility on large cells]
\label{assum:variance-oscillation}
There exist constants $\zeta\ge1$ and $\chi\in(0,1]$ such that, for every relevant cell $A\in\mathcal A_S$ and every $\bx\in A$, the following implication holds.
If
\[
\P(A)\ge \chi\paren*{L^{-2}V(A)}^{
s/(2\bar\alpha_S(\bx)+s)
}
\]
then
\begin{align}
\label{eq:bias-visibility}
\bigl|f^*(\bx)-\E[f^*(\bX)\mid \bX\in A]\bigr|
\le
\zeta\,
L^{s/(2\bar\alpha_S(\bx)+s)}
V(A)^{
\bar\alpha_S(\bx)/(2\bar\alpha_S(\bx)+s)}.
\end{align}
\end{assumption}

By applying bias visibility only to cells with mass at least the target bias-variance scale, this assumption is somewhat weaker than the one-dimensional analogue, Lemma~\ref{lemma:LH-upper-1d}.
However, it is sufficient for generalizing the bias step of the one-dimensional proof to multiple dimensions while remaining more plausible.
Indeed, we will show in Section~\ref{section:special} that both Assumptions~\ref{assum:reliable-split-balance} and \ref{assum:variance-oscillation} hold under more intuitive sufficient conditions on the regression function and design distribution.

With these two assumptions replacing Lemma~\ref{lemma:population-good-split-balance-1d} and Lemma~\ref{lemma:LH-upper-1d}, the one-dimensional argument extends to an optimal pointwise upper bound for CART-MID in multiple dimensions when all coordinates are relevant.
To attain an optimal rate in the sparse high-dimensional setting, we add one further assumption to ensure that the greedy split rule only selects relevant coordinates.


\begin{assumption}[Sparsity separation]
\label{assum:sparse-routing}
There exists $\vartheta\in(0,1]$ such that, for every cell $A\in\mathcal A_S$,
\begin{align}
\label{eq:sparse-routing}
\Delta_{\max}(A,S^c)
\le
(1-\vartheta)^2\,\Delta_{\max}(A,S).
\end{align}
\end{assumption}

This assumption is vacuous when there are no irrelevant coordinates and holds trivially when the irrelevant coordinates are independent of the relevant coordinates. In general, it is a population separation condition requiring relevant split signals to dominate irrelevant ones.
We now state the main result of this section.




\begin{theorem}[Pointwise upper bound for CART-MID]
\label{thm:main-thm-point-id-cart}
Assume the regression model in Section~\ref{sec:regression-model}. Assume bounded density (Assumption~\ref{assum:bounded-density}) and sub-Gaussian response (Assumption~\ref{assum:subgaussian-response}). Suppose $f^*$ is locally anisotropic H\"older as in Definition~\ref{def::localholder} and has a nonempty sparse support set $S\subseteq[d]$ with $|S|=s$.
Suppose sufficient impurity decrease (Assumption~\ref{assum:sid}), pathwise competitive split-scale regularity (Assumption~\ref{assum:reliable-split-balance}), bias visibility on large cells (Assumption~\ref{assum:variance-oscillation}), and sparsity separation (Assumption~\ref{assum:sparse-routing}) hold on the relevant cell class $\mathcal A_S$, with parameters satisfying $\chi
\lesssim_{p_{\min},\alpha_{\min}}
\omega^{2s/(s+2)}
\lambda^{s/(2\alpha_{\min}+s)}$. 
Let $u\ge0$ and choose
\begin{align}
\label{equ::deltathm}
\delta
=
C
\vartheta^{-2}
K^2
\frac{s\log(nd)+u}{n},
\end{align}
where $C>0$ is a sufficiently large universal constant. Assume that this threshold satisfies $\delta\le1$.
Then, with probability at least $1-e^{-u}$, uniformly for all $\bx\in[0,1]^d$,
\begin{align}
\label{eq:main-thm-point-eid-cart}
\bigl|\hat f_\delta(\bx)-f^*(\bx)\bigr|
\lesssim_{p_{\min},\alpha_{\min}}
\zeta\,
\omega^{-1}
L^{s/(2\bar\alpha_S(\bx)+s)}
\left(
K^2
\frac{s\log(nd)+u}{\lambda\vartheta^2 n}
\right)^{
\bar\alpha_S(\bx)/(2\bar\alpha_S(\bx)+s)
}.
\end{align}
\end{theorem}

The proof of Theorem~\ref{thm:main-thm-point-id-cart} is given in Appendix~\ref{app:cart-mid-upper-proof}.

\begin{remark}[Minimax optimality]
Up to logarithmic factors, Theorem~\ref{thm:main-thm-point-id-cart} attains the pointwise minimax rate over an $s$-dimensional anisotropic H\"older class with local harmonic smoothness $\bar\alpha_S(\bx)$ and radius $L$:
\[
L^{s/(2\bar\alpha_S(\bx)+s)}
n^{-\bar\alpha_S(\bx)/(2\bar\alpha_S(\bx)+s)}.
\]
This is the rate obtained by balancing local bias and leaf-average variance in the relevant coordinates. The result is adaptive in the sense that CART-MID does not know $\bar\alpha_S(\bx)$ or the relevant set $S$; the ambient dimension enters only through the split-search logarithm. In the one-dimensional case, this reduces to the classical pointwise rate $L^{1/(2\alpha(x)+1)}n^{-\alpha(x)/(2\alpha(x)+1)}$; see \citet{ibragimov1982bounds}, \citet{stone1982optimal}, and \citet{lepski2015adaptive}.
\end{remark}

\begin{remark}[Comparison with previous work] \label{rem:comparison-with-previous-work}
\citet{chi2022asymptotic} introduced SID to prove high-dimensional consistency rates for random forests using the original CART split criterion, and \citet{Mazumder2024} later gave a sharper single-tree CART analysis under the same condition.
Both analyses use SID in a global, depth-wise contraction argument: the tree depth is chosen to balance a geometrically decaying integrated squared bias term with a leaf-estimation term.
Consequently, the resulting $L^2(P)$ rates depend on the SID constant through the contraction exponent; for example, \citet{Mazumder2024} obtain a rate with exponent
\[
\phi(\lambda)
=
\frac{-\log_2(1-\lambda)}{1-\log_2(1-\lambda)}.
\]
By contrast, Theorem~\ref{thm:main-thm-point-id-cart} attains the usual pointwise smoothness exponent $\bar\alpha_S(\bx)/(2\bar\alpha_S(\bx)+s)$, while the SID constant enters polynomially through the scale factor. The local role of MID and SID underlying this distinction is described in Remark~\ref{remark:local-proof-mechanism}. In multiple dimensions, however, the argument additionally requires pathwise competitive split-scale regularity, bias visibility, and sparsity separation. These local geometric conditions are not required by the earlier global $L^2(P)$ contraction analyses.
\end{remark}

\begin{remark}[Choice of $\delta$]
    In practice, the threshold $\delta$ can be selected by cross-validation. For a finite candidate grid, a standard held-out validation argument shows that, if the grid contains a threshold of the order specified in \eqref{equ::deltathm}, a balanced sample split preserves the corresponding integrated $L^2(P)$ rate, up to the usual validation remainder. Obtaining an analogous guarantee for the uniform pointwise rate in Theorem~\ref{thm:main-thm-point-id-cart} is more difficult, because validation controls global prediction risk and may not detect errors localized to small terminal cells.\end{remark}

\section{Sufficient conditions for the structural assumptions}
\label{section:special}

The previous section stated a pointwise upper bound for CART-MID under assumptions deliberately framed as extensions of structural properties that are provable in one dimension from SID alone.
They can therefore be viewed as close to necessary conditions for our pointwise-adaptation argument.
The price to pay for this generality is that the assumptions are somewhat abstract and difficult to verify directly.
In this section, we give more interpretable sufficient conditions for these assumptions as well as examples satisfying them.


\subsection{Local norm equivalence and bias visibility}
\label{section:special:SID,VC}

For a cell $A$, write
\[
\Osc(A):=\sup_{\by,\bz\in A}|f^*(\by)-f^*(\bz)|
\]
for the oscillation of the regression function over $A$.

\begin{assumption}[Local norm equivalence]
\label{assum:local-norm-equivalence}
There exists a constant $\rho\ge1$ such that, for every relevant cell $A\in\mathcal A_S$,
\begin{align}
\label{eq:local-norm-osc}
\Osc(A)^2
&\le
\rho\,\Var[f^*(\bX)\mid\bX\in A],
\end{align}
\end{assumption}

Local norm equivalence (LNE) compares the conditional sup-norm variation of the regression function on each relevant cell with its conditional $L^2(P)$ norm.
While the opposite inequality holds trivially, the $L^2(P)$ norm can be much smaller than the sup-norm variation if $f^*$ is spiky.
Such spikiness can create a large bias in leaf averaging without creating sizeable residual variation, and hence LNE is a natural, albeit strong condition for bias visibility, as shown in the following proposition.

\begin{proposition}[Local norm equivalence implies bias visibility]
\label{prop:lne-implies-bias-visibility}
Assume the regression model in Section~\ref{sec:regression-model}, including bounded density (Assumption~\ref{assum:bounded-density}).
Suppose local norm equivalence (Assumption~\ref{assum:local-norm-equivalence}) holds on $\mathcal A_S$.
Then bias visibility on large cells (Assumption~\ref{assum:variance-oscillation}) holds on $\mathcal A_S$ for any $\zeta\ge1$ and $\chi\in(0,1]$ satisfying $\zeta^2\ge\rho\chi^{-1}$.
\end{proposition}
The proof is given in Appendix~\ref{app:sufficient-condition-proofs}.


\subsection{Pathwise competitive split-scale regularity}
\label{sec::examplessplit}

In addition to implying bias visibility, LNE together with SID and smoothness of $f^*$ and $p_X$ also implies pathwise competitive split-scale regularity.
The next proposition makes this implication explicit.

\begin{proposition}[Verifying pathwise competitive split-scale regularity]
\label{prop:pathwise-split-scale-from-primitives}
Assume the regression model in Section~\ref{sec:regression-model}, including bounded density (Assumption~\ref{assum:bounded-density}).
Suppose both $f^*$ and $p_X$ are locally anisotropic H\"older as in Definition~\ref{def::localholder}. 
Suppose sufficient impurity decrease (Assumption~\ref{assum:sid}) and local norm equivalence (Assumption~\ref{assum:local-norm-equivalence}) hold on $\mathcal A_S$.
Then there exists a constant
$C_{p_{\min},p_{\max},\rho,\alpha_{\min},\|f^*\|_\infty}>0$ such that
pathwise competitive split-scale regularity (Assumption~\ref{assum:reliable-split-balance}) holds on $\mathcal A_S$ with
\[
\omega
=
\min\left\{1,\,
C_{p_{\min},p_{\max},\rho,\alpha_{\min},\|f^*\|_\infty}
\lambda^{1+1/(2\alpha_{\min})}
\right\}.
\]
\end{proposition}


The proof of the proposition uses three lemmas that disentangle the respective roles played by the assumptions.
Lemma~\ref{lem:smooth-split-signal-upper} controls the signal of an ancestor split by the H\"older scale of the parent side length.
Lemma~\ref{lem:reliable-split-balance-from-lne} then prevents a competitive split from discarding too much mass in the split coordinate, while Lemma~\ref{lem:pathwise-variance-comparison-from-lne} lets us compare residual variation along the path.
Together, these three ingredients convert ancestor split signal into a lower bound on the terminal side length.
When Proposition~\ref{prop:lne-implies-bias-visibility} is used to verify Assumption~\ref{assum:variance-oscillation}, the compatibility condition on $\chi$ in Theorem~\ref{thm:main-thm-point-id-cart} can be met by choosing $\zeta$ large enough relative to the local norm-equivalence, SID, and path-regularity constants.
The proofs of the lemmas as well as the proposition are given in Appendix~\ref{app:sufficient-condition-proofs}.

\begin{lemma}[Smooth split-signal bound]
\label{lem:smooth-split-signal-upper}
Assume the regression model in Section~\ref{sec:regression-model}, including bounded density (Assumption~\ref{assum:bounded-density}). 
Suppose both $f^*$ and $p_X$ are locally anisotropic H\"older as in Definition~\ref{def::localholder}. Then, for every $A\in\mathcal A_S$, $\bx\in A$, $j\in S$, and every admissible $b\in A_j$,
\begin{equation}
\label{eq:directional-split-smoothness}
\frac{\Delta(A,j,b)}{\P(A)}
\lesssim_{p_{\min},p_{\max},\|f^*\|_\infty}
L^2\len_j(A)^{2\alpha_j(x_j)}.
\end{equation}
\end{lemma}

\begin{lemma}[Competitive splits are balanced]
\label{lem:reliable-split-balance-from-lne}
Assume the regression model in Section~\ref{sec:regression-model}, including bounded density (Assumption~\ref{assum:bounded-density}).
Suppose sufficient impurity decrease (Assumption~\ref{assum:sid}) and local norm equivalence (Assumption~\ref{assum:local-norm-equivalence}) hold on $\mathcal A_S$.
Then any competitive split of $A\in\mathcal A_S$ with $\Delta_{\max}(A)>0$ satisfies
\[
\min\{\P(A_L), \P(A_R)\}
\ge
\frac{\lambda}{2\rho}\P(A).
\]
\end{lemma}

\begin{lemma}[Pathwise variance comparison]
\label{lem:pathwise-variance-comparison-from-lne}
Assume the regression model in Section~\ref{sec:regression-model}, including bounded density (Assumption~\ref{assum:bounded-density}).
Suppose local norm equivalence (Assumption~\ref{assum:local-norm-equivalence}) holds on $\mathcal A_S$.
Then relevant cells $A'\subseteq A$ satisfy
\[
\Var[f^*(\bX)\mid\bX\in A']
\le
\rho\,\Var[f^*(\bX)\mid\bX\in A].
\]
\end{lemma}


\subsection{Locally reverse Poincar\'e as a building block for SID and LNE}


To construct multivariate regression functions satisfying both SID and LNE, we use the locally reverse Poincar\'e (LRP) condition on univariate functions.

\begin{definition}[Locally reverse Poincar\'e] \label{def:LRP}
Given $R>0$, an absolutely continuous function
$g\colon[0,1]\to\mathbb R$ satisfies the locally reverse
Poincar\'e condition with constant $R$ if
\begin{equation}
\label{eq:assum-TV}
\left(\int_a^b |g'(t)|\,\mathrm{d}t\right)^2
\le
\frac{R}{b-a}
\inf_{w\in\mathbb R}
\int_a^b |g(t)-w|^2\,\mathrm{d}t
\qquad
\text{for every }0\leq a<b\leq 1.
\end{equation}
Here and below, $g'$ denotes the almost-everywhere derivative of the
absolutely continuous function $g$.
\end{definition}

For any interval $[a,b]$, absolute continuity implies that the left-hand side is the square of the total variation of $g$ on $[a,b]$.
The total variation of a continuous function is at least as large as its oscillation, and hence controls its conditional variance on $[a,b]$.
What LRP asserts is that the reverse inequality also holds, up to a constant factor $R$.

LRP was introduced by \citet{Mazumder2024} as a sufficient condition for verifying SID in one dimension.
They showed that LRP is a broad class of functions, containing all
strictly increasing functions whose derivatives are bounded above and bounded away from zero,
smooth strongly convex functions,
and polynomials of bounded degree \citep[Examples~3.1--3.3]{Mazumder2024}.
In addition, they showed that LRP functions can be used as building blocks for constructing examples of SID functions in higher dimensions: specifically, additive models with LRP components satisfy SID.

We complement this by proving LNE for additive models with LRP components, and then extend both SID and LNE to sums of products of LRP components.




\begin{proposition}[LRP-based multivariate models]\label{prop:sid-suffi}
    Assume bounded density (Assumption~\ref{assum:bounded-density}), and let $R>0$.
    Then the following statements hold.
    \begin{enumerate}
        \item[\textup{(i)}] \textup{(Additive models.)}
        Suppose $f(\bx)=f_1(x_1)+\cdots+f_s(x_s)$, where
        each $f_k$ satisfies the locally reverse Poincar\'e condition
        with constant $R$.
        Then $(f,P)$ satisfies sufficient impurity decrease
        (Assumption~\ref{assum:sid}) and local norm equivalence
        (Assumption~\ref{assum:local-norm-equivalence}), with constants $\lambda$ and $\rho$ depending only on
        $s,p_{\min},p_{\max},R$.

        \item[\textup{(ii)}] \textup{(Sums of products.)}
        Let $s,M\geq 1$ and let
        $\mathcal{J}_1,\ldots,\mathcal{J}_M$ be (not
        necessarily disjoint) nonempty subsets of $[s]$ with
        $\bigcup_{j=1}^{M}\mathcal{J}_j=[s]$. Suppose that
        $X_1,\ldots,X_s$ are mutually independent and
        \[
        f(\bx)=\sum_{j=1}^{M}\prod_{k\in\mathcal{J}_j}h_k(x_k),
        \]
        where each $h_k\colon[0,1]\to\mathbb{R}$ satisfies the locally
        reverse Poincar\'e condition with constant $R$ and takes values in
        $[h_{\min},h_{\max}]$, with
        $0<h_{\min}\le h_{\max}<\infty$.
        Then $(f,P)$ satisfies sufficient impurity decrease (Assumption~\ref{assum:sid}) and local norm equivalence (Assumption~\ref{assum:local-norm-equivalence}), with constants $\lambda$ and $\rho$ depending only on $s,M,p_{\min},p_{\max},R,h_{\min},h_{\max}$.
    \end{enumerate}
\end{proposition}

The next proposition shows that SID and LNE are preserved when functions satisfying
them are added across independent coordinate blocks, which further expands the class of examples satisfying our assumptions.

\begin{proposition}[Closure under independent blocks]\label{prop:additive}
    Assume $f\colon[0,1]^s\to\mathbb{R}$ admits, for some $M\ge1$, the decomposition
    $f(\bx) = g_1(\bx_{\mathcal{J}_1}) + \cdots
              + g_M(\bx_{\mathcal{J}_M})$,
    where $\mathcal{J}_1,\ldots,\mathcal{J}_M$ form a disjoint
    partition of $[s]$, and let $P_{\mathcal{J}_j}$ denote
    the marginal distribution of $\bX_{\mathcal{J}_j}$.
    Suppose
    the subvectors
    $\bX_{\mathcal{J}_1},\ldots,\bX_{\mathcal{J}_M}$ are
    mutually independent.
    Suppose each pair $(g_j,P_{\mathcal{J}_j})$ satisfies sufficient impurity decrease (Assumption~\ref{assum:sid}) and local norm equivalence (Assumption~\ref{assum:local-norm-equivalence}) on its coordinate block.
    Then $(f,P)$ satisfies sufficient impurity decrease and local norm equivalence, with constants depending only on $s$ and on the corresponding constants of the individual pairs.
\end{proposition}

The proofs of Propositions~\ref{prop:sid-suffi} and \ref{prop:additive} are given in Appendix~\ref{app:sufficient-condition-proofs}.




\begin{remark}[Global versus piecewise LRP]
    \citet[Proposition~3.2]{Mazumder2024} showed that additive
    models with piecewise-LRP components satisfy SID, even when
    the components are discontinuous at the joining points.
    Piecewise LRP does not, however, imply local norm
    equivalence.  For example,
    $g(x)=\1\{x>1/2\}$ is LRP on each of the two pieces and its
    split at $1/2$ removes all residual variance.  But on
    $I_\varepsilon=[1/2-\varepsilon,1/2+\varepsilon^2]$ under
    the uniform law,
    $\operatorname{Osc}_{I_\varepsilon}(g)^2=1$ while
    $\Var[g(X)\mid X\in I_\varepsilon]=\varepsilon/(1+\varepsilon)^2$.
    Thus, in the additive setting, piecewise LRP is enough for
    SID but not for the sup-norm-to-$L^2$ control supplied by
    LNE.
\end{remark}


\section{Pointwise lower bound for CART-MLS}
\label{sec::motivation-lower-bound}

The positive results in the previous sections show that CART-MID can achieve pointwise adaptation.
As argued at the start of Section~\ref{sec:one-dimensional-upper}, this is partly because the optimal stopping value for the split signal is independent of the local smoothness, so a single global threshold can simultaneously select the correct local bandwidth at all points.
On the other hand, the optimal stopping value for the leaf size depends on the local smoothness, and a single global leaf size cannot simultaneously achieve the optimal bias-variance tradeoff at two points with different smoothness.
We formalize this intuition in this section by proving a pointwise lower bound for CART-MLS for a one-dimensional example.

\begin{theorem}[Lower bound for CART with minimum leaf size]
\label{thm::lowerright}
Assume the regression model in Section~\ref{sec:regression-model} with
$d=1$ and bounded density (Assumption~\ref{assum:bounded-density}),
$f^*(x)=x^{1/2}$, $\xi_1, \ldots, \xi_n \stackrel{\mathrm{iid}}{\sim}
\mathcal N(0, \sigma^2)$ independently of $X_1,\ldots,X_n$, and $\sigma^2>0$.
Let $\widehat f_N$ denote either the ordinary or honest CART-MLS estimator.
There exist constants $C,c,p_0>0$, depending only on
$p_{\min}$ and $p_{\max}$, and an integer
$n_0=n_0(\sigma,p_{\min},p_{\max})\in\mathbb{N}$ such that, for all
$n\ge n_0$ and every fixed $N\in[n]$, the following statements hold.

If $N\ge C\log n$, then, for either estimator,
\begin{equation}\label{eq:mls-lower-main-regime}
\P\left\{
\begin{aligned}
|\widehat f_N(0)-f^*(0)|
&\ge c\left(\sqrt{\frac{N}{n}}+\frac{\sigma}{\sqrt N}\right),\\
|\widehat f_N(1)-f^*(1)|
&\ge c\left(\frac{N}{n}+\frac{\sigma}{\sqrt N}\right)
\end{aligned}
\right\}
\ge p_0.
\end{equation}
If $N<C\log n$, then the same joint probability is at least $p_0$ when
both lower bounds in \eqref{eq:mls-lower-main-regime} are replaced by
\begin{equation}\label{eq:mls-lower-small-regime}
\frac{c\sigma}{\sqrt{\log n}}
\quad\text{for honest CART-MLS},
\qquad
\frac{c\sigma}{\sqrt N\,\log n\,\log\log n}
\quad\text{for ordinary CART-MLS}.
\end{equation}
\end{theorem}

The proof of Theorem~\ref{thm::lowerright} is given in
Appendix~\ref{app:cart-mls-lower-bound}.
To interpret the lower bound, consider $N_n=\lfloor n^\gamma\rfloor$ for a fixed
$\gamma\in(0,1]$. The two lower bounds in \eqref{eq:mls-lower-main-regime} are polynomial in $n$ with exponents
\[
    \beta_0(\gamma)=\min\{\gamma/2,(1-\gamma)/2\},
    \qquad
    \beta_1(\gamma)=\min\{\gamma/2,1-\gamma\}.
\]
These exponents are plotted in
Figure~\ref{fig:mls-bandwidth-mismatch}.
We see that $\beta_0$, corresponding to the rough endpoint $x=0$, is optimized at $\gamma=1/2$, while $\beta_1$, corresponding to the smoother endpoint $x=1$, is optimized at $\gamma=2/3$. Thus a single $\gamma$ cannot optimize both exponents simultaneously, and a single global minimum leaf size cannot simultaneously achieve optimal pointwise rates at both endpoints. This illustrates the bandwidth-selection failure of CART-MLS.

\begin{figure}[H]
\centering
\begin{tikzpicture}[x=7.2cm,y=7.2cm]
    \draw[->] (0,0) -- (1.06,0) node[right] {$\gamma$};
    \draw[->] (0,0) -- (0,0.39);

    \draw[-, cblue, very thick]
        (0,0) -- (0.5,0.25) -- (1,0);
    \draw[-, cgreen!70!black, very thick, dashed]
        (0,0) -- (0.666,0.333) -- (1,0);

    \draw[-, densely dotted] (0.5,0) -- (0.5,0.25);
    \draw[-, densely dotted] (0.666,0) -- (0.666,0.333);
    \draw[-, densely dotted] (0,0.25) -- (0.5,0.25);
    \draw[-, densely dotted] (0,0.333) -- (0.666,0.333);

    \fill[cblue] (0.5,0.25) circle (1.4pt);
    \fill[cgreen!70!black] (0.666,0.333) circle (1.4pt);

    \node[below] at (0,0) {$0$};
    \node[below] at (0.5,0) {$1/2$};
    \node[below] at (0.666,0) {$2/3$};
    \node[below] at (1,0) {$1$};
    \node[left] at (0,0.25) {$1/4$};
    \node[left] at (0,0.333) {$1/3$};

    \node[cblue, anchor=west] at (0.73,0.15) {$\beta_0(\gamma)$};
    \node[cgreen!70!black, anchor=west] at (0.73,0.25) {$\beta_1(\gamma)$};
\end{tikzpicture}
\caption{The exponent functions $\beta_0(\gamma)$ and $\beta_1(\gamma)$.}
\label{fig:mls-bandwidth-mismatch}
\end{figure}
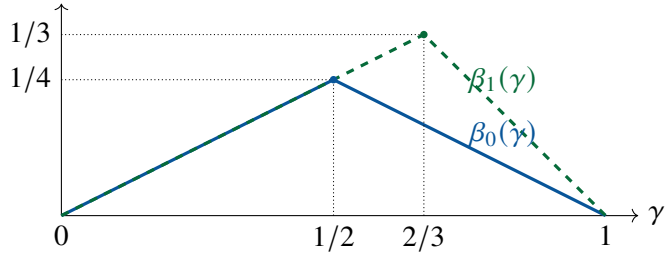

\begin{remark}[Isolating the bandwidth-selection obstruction]
    The proof uses only the endpoint geometry and terminal cell sizes forced by the MLS rule, together with the bias and noise of averaging over those cells. Thus the obstruction would remain even if all splits along the relevant paths were reliable: the failure comes from using one global leaf-size parameter to choose local bandwidths.
\end{remark}

\begin{remark}[Role of honesty]
    Honesty changes only the stochastic part of the proof. Conditional on the two designs, an honest leaf average has a fixed Gaussian noise term. For ordinary CART, the selected leaf depends on the same responses that are averaged, so the proof instead uses a noise-persistence event that holds uniformly over every admissible endpoint leaf size.
\end{remark}

\begin{remark}[Performance gap]
    Proposition~\ref{prop:sqrt-lrp} verifies that $f^*(x)=x^{1/2}$ satisfies LRP.
    Proposition~\ref{prop:sid-suffi}(i) then implies SID under bounded density.
    Moreover, $f^*$ satisfies Definition~\ref{def::localholder} with $L=2$ and
    \[
        \alpha(x)=
        \begin{cases}
            1/2, & 0\le x\le 1/4,\\
            1, & 1/4<x\le1.
        \end{cases}
    \]
    Indeed, $|\sqrt{x}-\sqrt{x'}|\le |x-x'|^{1/2}$ in the first case,
    while $|\sqrt{x}-\sqrt{x'}|\le 2|x-x'|$ when $x>1/4$.
    Contrasting Theorems~\ref{thm::lowerright} and~\ref{thm:cart-mid-one-dimensional} thus establishes a rigorous performance gap between CART-MID and CART-MLS.
\end{remark}

\section{Noise-driven splits and end-cut preference}
\label{sec:split-reliability-lower-bound}

Having seen the importance of stopping rules for bandwidth selection, we now turn to their role in preventing noise-driven splits, that is, splits whose empirical impurity decrease is dominated by noise rather than signal.
The question we address in this section is whether preventing noise-driven splits is necessary for optimal pointwise rates, or good prediction more generally.
In a multidimensional setting, noise-driven splits can select irrelevant features, thereby leading to poor generalization in sparse high-dimensional settings \citep{tan2024statistical}.
In one dimension, the issue is more subtle as it revolves around the location of the split threshold.

It has long been known that the CART splitting rule tends to favor split thresholds close to the boundaries of cells~\citep{morgan1973thaid}.
This phenomenon was called ``end-cut preference'' by \citet{breiman1984classification}, who provided a theoretical explanation.
To describe the idea, consider the noise-only problem (i.e. $f^* \equiv 0$) on the
interval $[0,1]$ containing $n$ observations.
Order the observations by their covariate values and write $\hat\iota$ for the CART split index, so that the two children
contain $\hat\iota$ and $n-\hat\iota$ observations. 
The classical result
of \citet[Chapter~11.8]{breiman1984classification} says that, for each fixed
$\epsilon\in(0,1/2)$, the probability that $\hat\iota$ lies in the endpoint neighborhood
$\{\hat\iota\le \epsilon n\}\cup\{\hat\iota\ge (1-\epsilon)n\}$ tends to one as $n \to\infty$.

End-cut preference, however, need not necessarily preclude optimal pointwise rates.
Our proof of Theorem~\ref{thm:cart-mid-one-dimensional} does not require balanced splits, only the much weaker competitive split-scale regularity property.
Indeed, a concrete example of optimal pointwise rates despite unbalanced splits is given in the following example.

\begin{example}[Split-scale regular end cuts for ReLU]
\label{ex:relu}
    Let $X\sim\mathrm{Unif}[0,1]$ and consider the shifted ReLU
    \[
        f^*(x)=(2x-1)_+,
        \qquad x\in[0,1],
    \]
    To isolate the split geometry, consider the population CART recursion along the cells containing the endpoint $0$.
    Such a cell has the form $A'=[0,a]$.
    Appendix~\ref{app:relu-end-cut-calculation} shows that, for $a>1/2$ with $a\downarrow1/2$, the population-optimal split point $b$ satisfies
    \[
        a-b\asymp a-1/2.
    \]
    Thus the right child contains a vanishing fraction of the parent cell, even though its length remains commensurate with the residual signal scale.
    This is a statement about the population split geometry; a finite-sample empirical CART path need not follow the population recursion exactly.
    Nonetheless, this function satisfies the SID condition by a direct one-dimensional calculation, so Theorem~\ref{thm:cart-mid-one-dimensional} guarantees that CART-MID achieves the optimal pointwise rate at $0$.
    The supporting calculations are given in Appendix~\ref{app:relu-end-cut-calculation}.
\end{example}


More recently, \citet[Theorem~SA-1]{cattaneo2025honest} gave a more precise quantification of end-cut preference that implies more severe consequences for prediction.
They showed that, in the same noise-only setting considered by \citet{breiman1984classification}, the CART split index has the following behavior for any $0 < a < b < 1$:
\begin{equation}\label{eq:cattaneo}
  \liminf_{n\to\infty}\P\bigl\{n^a<\hat\iota<n^b\bigr\}
  \;\ge\;
  \frac{b-a}{2e},
  \qquad
  \liminf_{n\to\infty}\P\bigl\{n-n^b<\hat\iota<n-n^a\bigr\}
  \;\ge\;
  \frac{b-a}{2e}.
\end{equation}
In other words, the selected split has nonvanishing
probability of lying at distance $n^{-\epsilon}$ from either endpoint for any $0 < \epsilon < 1$.
\citet{cattaneo2025honest} used this result to obtain a pointwise lower bound for the estimation error of CART at the endpoints of the interval, assuming that the algorithm makes at least one split.


We extend \citet{cattaneo2025honest}'s result to show that an analogous quantitative form of end-cut preference persists even when the regression function is nonconstant, provided that the local residual variation is beneath the noise level (cf. \eqref{eq:noise-floor}).
Together with Remark~\ref{rem:pointwise-lower-bound}, this argues more strongly that preventing noise-driven splits is indeed necessary for optimal pointwise rates.

\begin{theorem}[End-cut preference under weak split signal]
\label{thm:noise-driven-split}
Assume the regression model in Section~\ref{sec:regression-model} with $d=1$,
bounded density (Assumption~\ref{assum:bounded-density}),
$\xi_1,\ldots,\xi_n\stackrel{\mathrm{iid}}{\sim}\mathcal N(0,\sigma^2)$ independently of $X_1,\ldots,X_n$, and
$\sigma^2>0$. Suppose $f^*$ satisfies the locally reverse Poincar\'e condition (Definition~\ref{def:LRP})
with some constant $R>0$.
Let $A_n\subset[0,1]$ be a deterministic sequence of intervals such that
\begin{equation}\label{eq:weak-signal-variation}
    n\log\log (n)V(A_n)\to 0,
    \qquad
    \frac{n\P(A_n)}{\log n}\to \infty .
\end{equation}
Let
$X_{A_n}^{(1)}\le X_{A_n}^{(2)}\le\cdots\le X_{A_n}^{(m_n)}$
denote the order statistics of the sample points falling in $A_n$, where
$m_n=N(A_n)$. Let
\[
    \hat\iota_n\in
    \argmax_{1\le \iota<m_n}
    \widehat\Delta\bigl(A_n,1,X_{A_n}^{(\iota)}\bigr)
\]
be the empirical CART split index on $A_n$. Then, for every fixed
$0<a<b<1$,
\begin{equation}\label{eq:unreliable}
\begin{aligned}
  \liminf_{n\to\infty}\P\bigl\{m_n^a<\hat\iota_n<m_n^b\bigr\}
  &\ge
  \frac{b-a}{2e},\\
  \liminf_{n\to\infty}\P\bigl\{m_n-m_n^b<\hat\iota_n<m_n-m_n^a\bigr\}
  &\ge
  \frac{b-a}{2e}.
\end{aligned}
\end{equation}
\end{theorem}


The proof of Theorem~\ref{thm:noise-driven-split} is given in
Appendix~\ref{app:noise-driven-splits}.

\begin{remark}[From split locations to pointwise lower bounds] \label{rem:pointwise-lower-bound}
    Theorem~\ref{thm:noise-driven-split} suggests that whenever we set the MID threshold to be $o\paren*{(n\log\log n)^{-1}}$, there is a nonvanishing probability that CART will produce a leaf with only $n^b$ samples for any $0 < b<1$.
    This inflates the variance of the leaf average to $n^{-b/2}$, which, for $b$ small enough, can be worse than the minimax rate.
    Turning this argument into a rigorous pointwise lower bound, however, requires bridging the theoretical gap between the deterministic sequence assumption in Theorem~\ref{thm:noise-driven-split} and the random sequence of cells produced by CART.
    We leave this as an open problem for future work.
\end{remark}

\section{Experiments}

We conduct two experiments illustrating complementary forms of
local adaptation by CART-MID.  The first asks whether a single
stopping threshold can adapt across rough and smooth regions of
a one-dimensional signal.  The second examines whether the
terminal cells adapt to different smoothness levels across
coordinates.

\subsection{Adaptation across rough and smooth regions}

\subsubsection*{Experimental setup}

Let $W$ be a fixed realization of standard Brownian motion on
$[0,1/2]$, and set $B(x)=5W(x)$ for $0\leq x\leq 1/2$.
We join this scaled Brownian path continuously to a linear
segment and define the hybrid Brownian--linear signal
\[
f_B^*(x)=
\begin{cases}
    B(x), & 0\leq x\leq 1/2,\\
    B(1/2)+4(x-1/2), & 1/2<x\leq 1.
\end{cases}
\]
This modifies the Brownian-motion example of
\citet{rovckova2024ideal}.  Almost surely, $B$ is
$\alpha$-H\"older continuous for every $\alpha<1/2$, whereas
the linear segment is Lipschitz.  The multiplicative factor only
changes the signal scale and not the local smoothness exponent.
We observe
\[
Y=f_B^*(X)+\xi,
\qquad
X\sim\operatorname{Unif}[0,1],
\qquad
\xi\sim\mathcal N(0,0.5),
\]
where $0.5$ denotes the noise variance.  The scaled Brownian path
is held fixed while the design points and errors are redrawn
across Monte Carlo repetitions.

For the visualization, we use $n=3000$ and set the CART-MID
threshold to $\delta=0.4\log(n)/n$.  We then compare
CART-MID and CART-MLS at $x=0.25$, in the rough region, and
$x=0.75$, in the smooth region.  For each stopping-parameter
value, we use 400 Monte Carlo repetitions with $n=3000$,
varying the minimum impurity-decrease threshold for CART-MID and
the minimum leaf size for CART-MLS.  We estimate the pointwise
mean squared error, its squared-bias and variance components,
and the mean length of the terminal cell containing each
evaluation point.
The tuning grids are
\[
\delta\in10^{-3}\{0.53,0.80,1.07,1.33,1.60,1.87,2.14,2.40\}
\]
for CART-MID and
\[
N\in\{15,45,75,105,135,165,195,225\}
\]
for CART-MLS.

\subsubsection*{Results}

Figure~\ref{fig:Brownian_visualization} shows the fixed
regression function in black, one training sample in gray, and
the CART-MID fit in red.  CART-MID creates smaller cells in the
rough Brownian region and stops earlier in the smooth linear
region.  Thus, even with one global threshold, the fitted tree
uses different local resolutions across the domain.

\begin{figure}[H]
    \centering
    \includegraphics[width=0.62\textwidth]{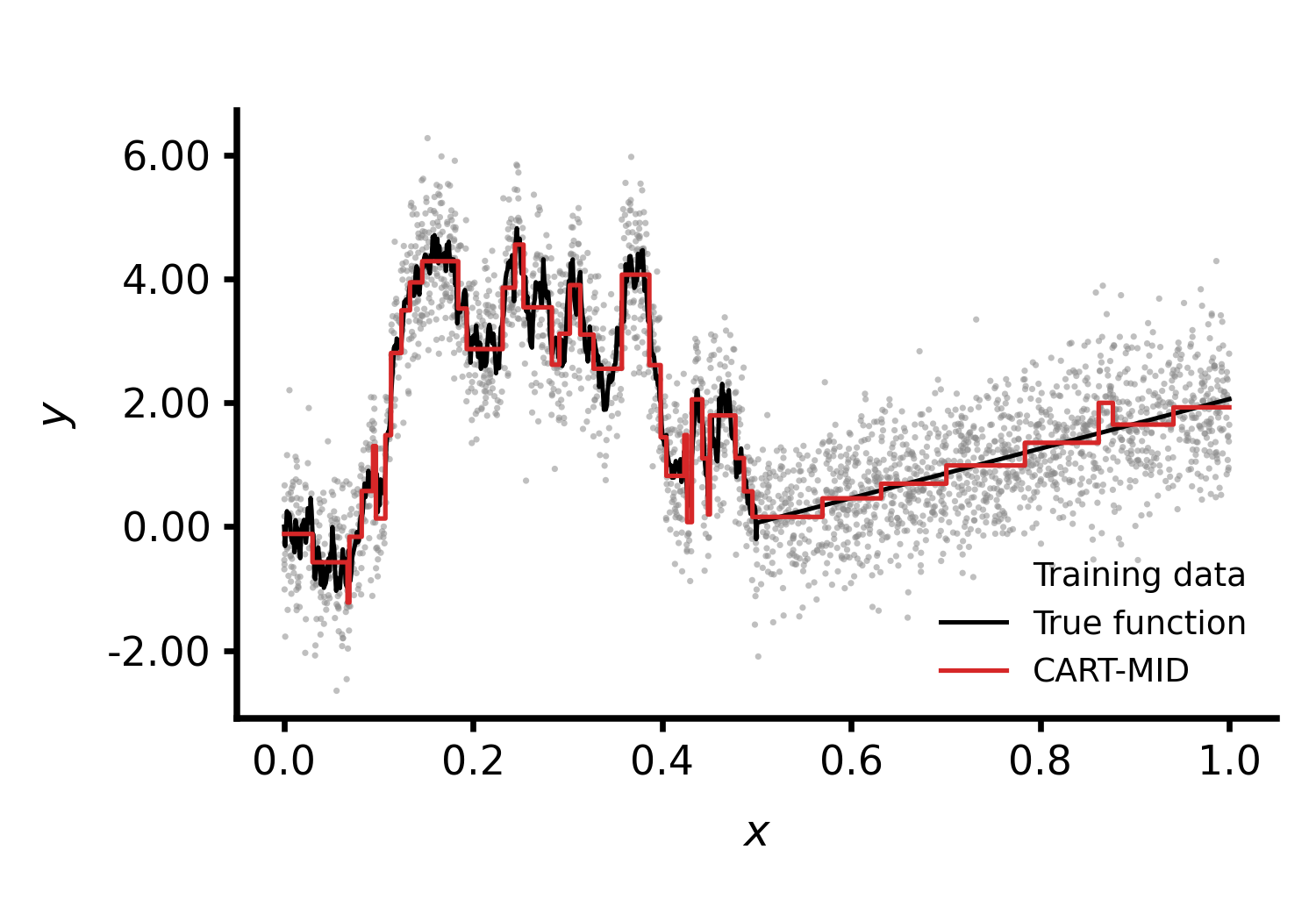}
    \caption{A CART-MID fit to the hybrid Brownian--linear
    signal.  The regression function is shown in black, the
    training sample in gray, and the fitted tree in red.}
    \label{fig:Brownian_visualization}
\end{figure}

Figure~\ref{fig:brownian-mse-decomposition} reports the
pointwise Monte Carlo bias--variance decomposition.  For
CART-MID, the mean squared error is minimized at MID thresholds around
$1.6\times10^{-3}$ at $x=0.25$ and around
$1.07\times10^{-3}$ at $x=0.75$.  Although the minimizers are
not identical, the curves are relatively flat over a common
range: near $1.07\times10^{-3}$, the mean squared error remains
stable at both locations.  A single impurity-decrease threshold
therefore gives a favorable trade-off in both the rough and
smooth regions.

CART-MLS exhibits a sharper tuning conflict.  Its optimal
minimum leaf size is about $45$ at $x=0.25$ and between $105$
and $135$ at $x=0.75$.  Smaller leaves are needed to control
bias in the rough region, whereas larger leaves reduce variance
in the linear region with only a modest increase in bias.  Since
the minimum-leaf-size constraint applies globally, it cannot
select both local scales with one tuning value.

\begin{figure}[H]
    \centering
    \begin{tabular}{cc}
        \textbf{CART-MID} & \textbf{CART-MLS} \\[0.2em]

        \begin{subfigure}[t]{0.47\textwidth}
            \centering
            \includegraphics[width=\linewidth]{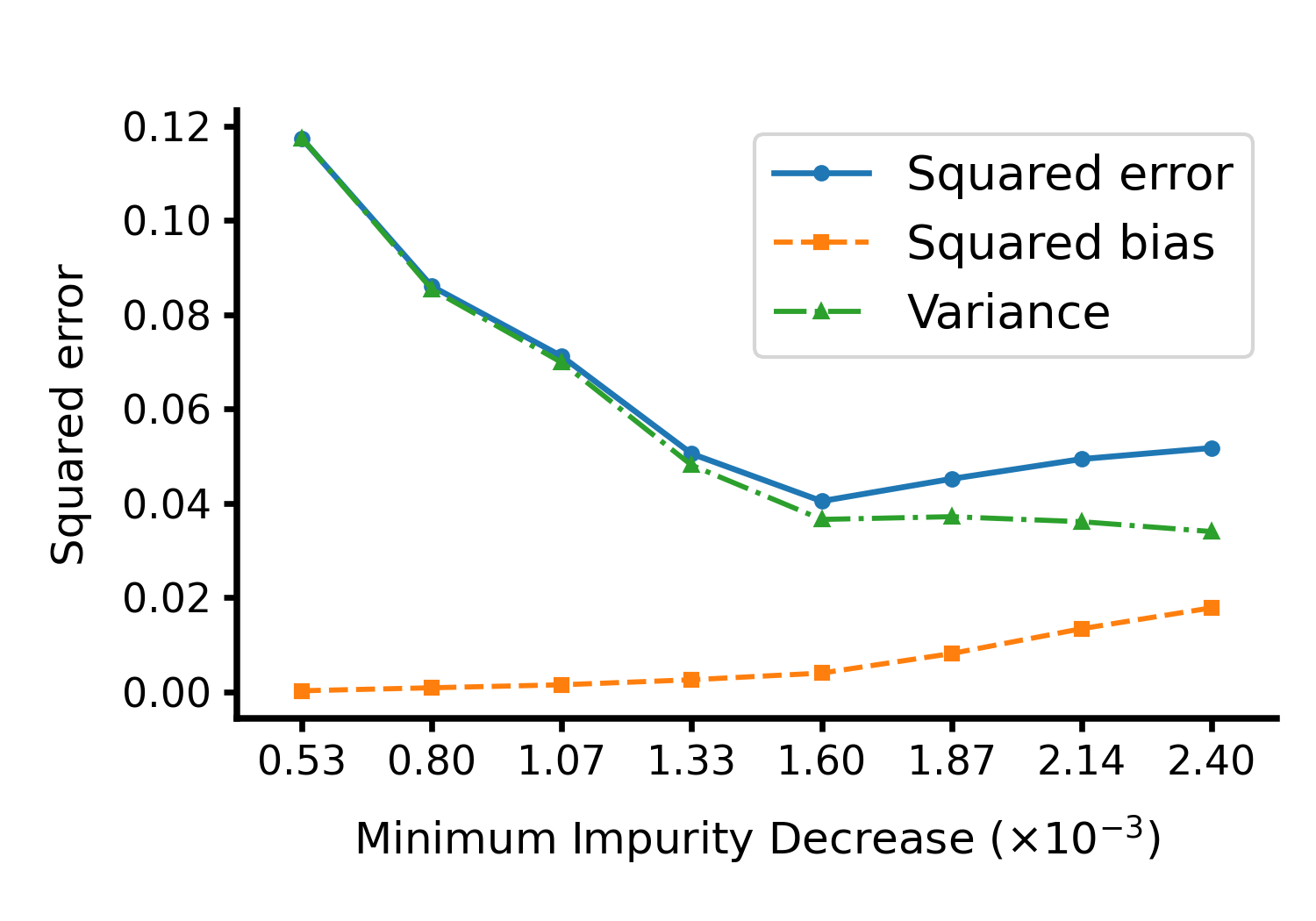}
            \caption{\(x=0.25\)}
            \label{fig:brownian-mid-mse-x025}
        \end{subfigure}
        &
        \begin{subfigure}[t]{0.47\textwidth}
            \centering
            \includegraphics[width=\linewidth]{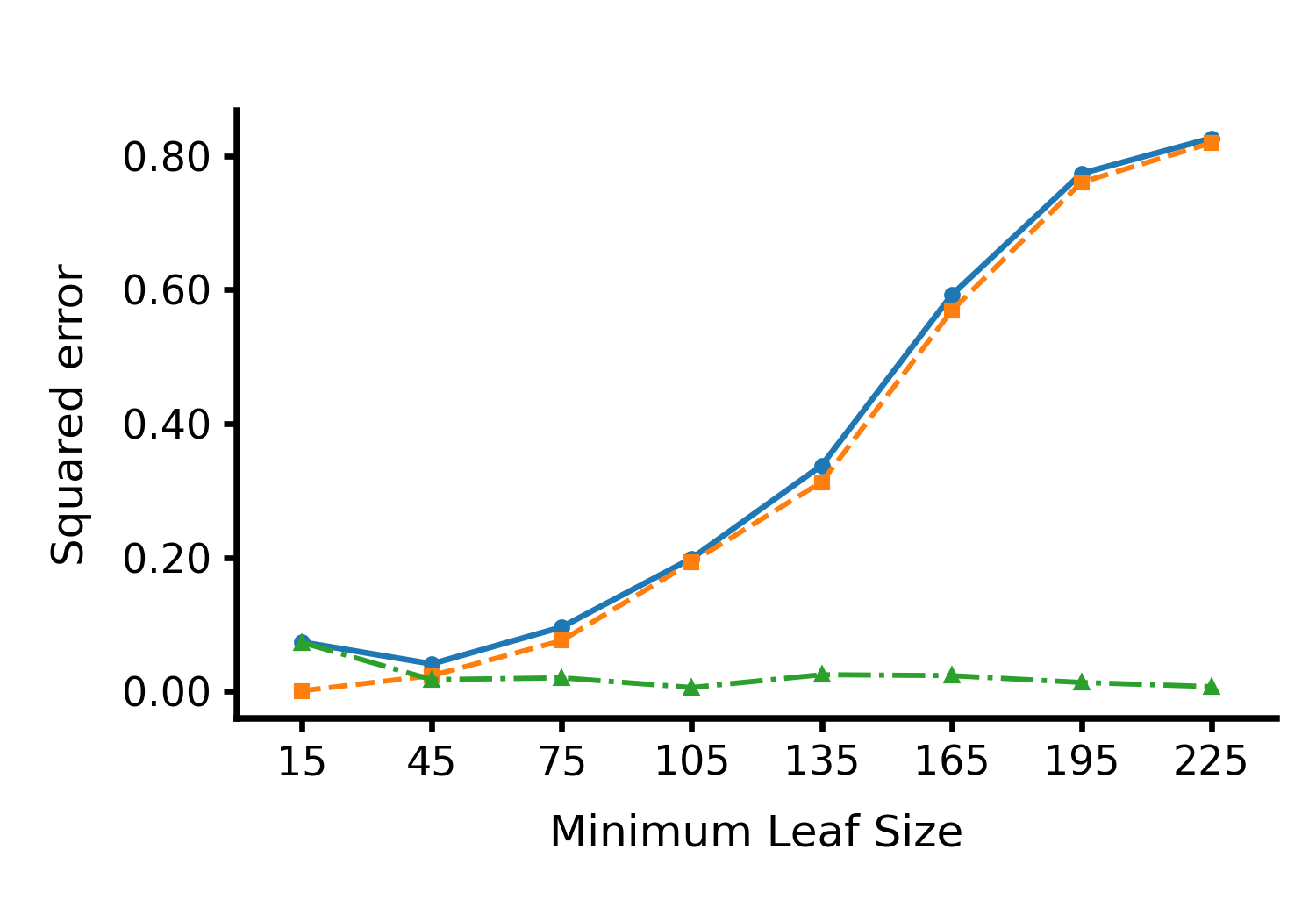}
            \caption{\(x=0.25\)}
            \label{fig:brownian-mls-mse-x025}
        \end{subfigure}
        \\[-0.2em]

        \begin{subfigure}[t]{0.47\textwidth}
            \centering
            \includegraphics[width=\linewidth]{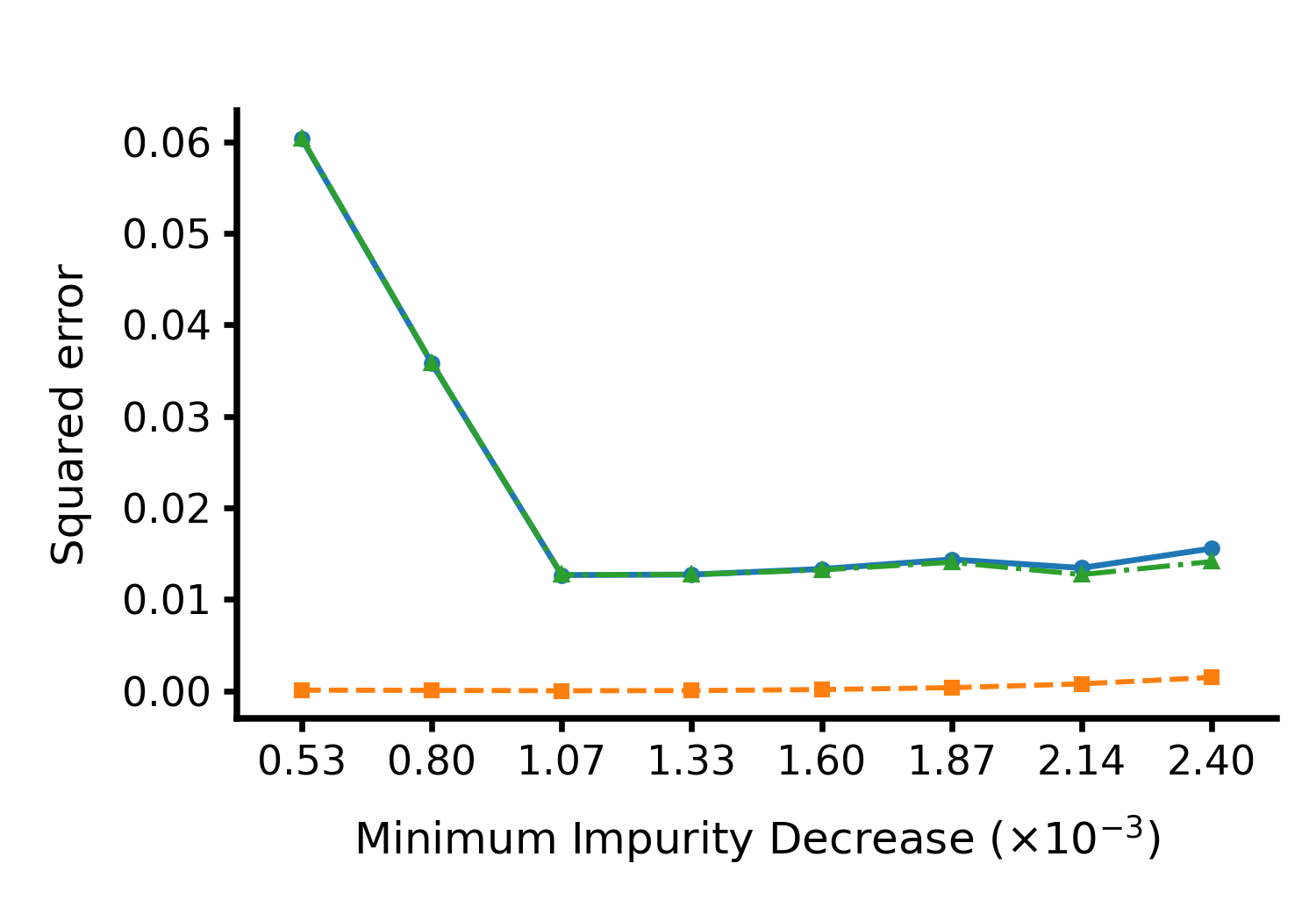}
            \caption{\(x=0.75\)}
            \label{fig:brownian-mid-mse-x075}
        \end{subfigure}
        &
        \begin{subfigure}[t]{0.47\textwidth}
            \centering
            \includegraphics[width=\linewidth]{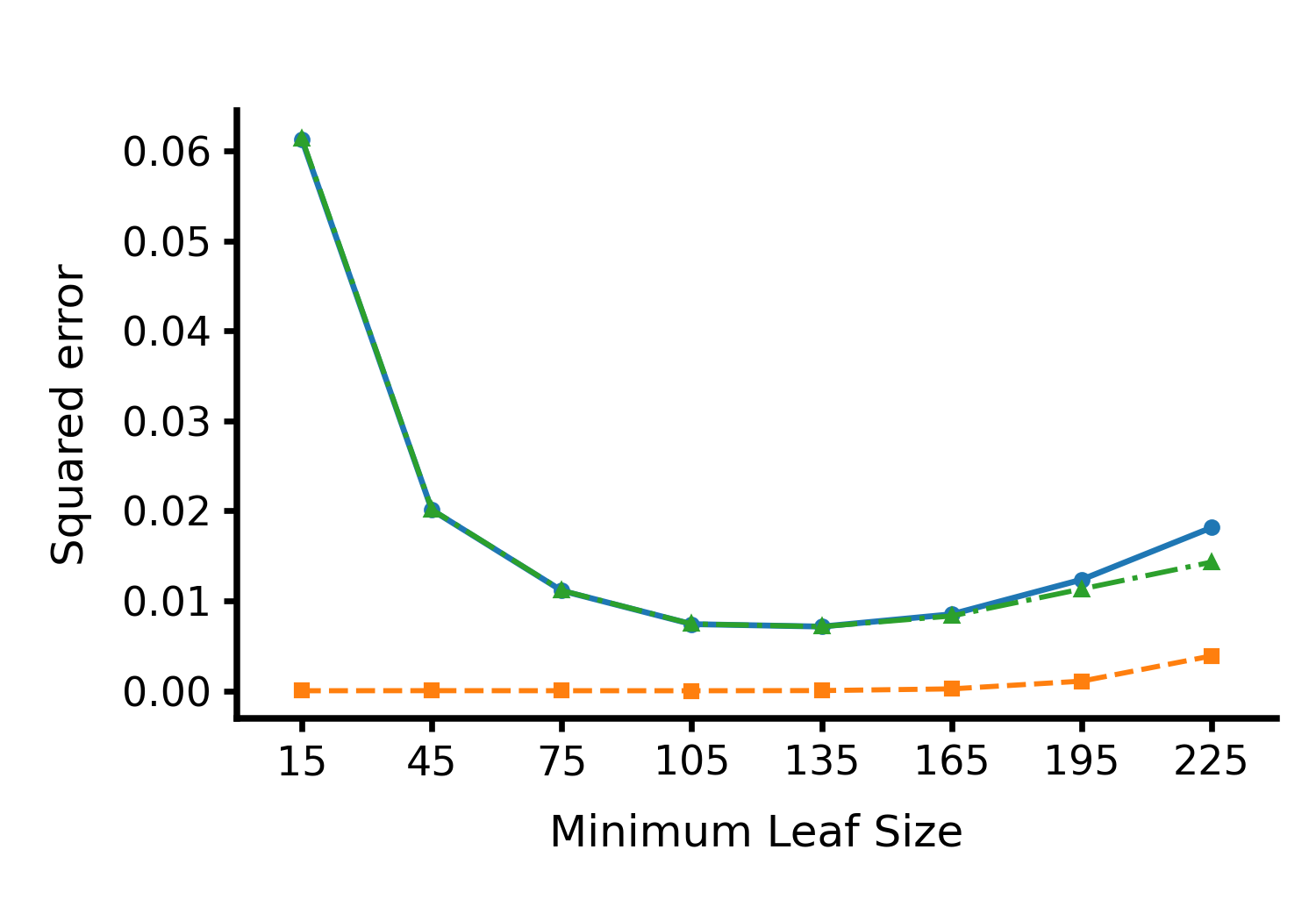}
            \caption{\(x=0.75\)}
            \label{fig:brownian-mls-mse-x075}
        \end{subfigure}
    \end{tabular}

    \caption{Pointwise Monte Carlo bias--variance decomposition
    for the hybrid Brownian--linear signal.  The columns compare
    CART-MID and CART-MLS, and the rows correspond to
    $x=0.25$ and $x=0.75$.  The vertical scale is allowed to
    differ across panels so that all three components remain
    visible.}
    \label{fig:brownian-mse-decomposition}
\end{figure}

Figure~\ref{fig:brownian-cell-size} explains this difference in
terms of the selected local bandwidths.  Under the same
impurity-decrease threshold, CART-MID produces much smaller
terminal cells at $x=0.25$ than at $x=0.75$.  By contrast,
CART-MLS produces nearly equal cell lengths at the two points
for a fixed minimum leaf size.  Under the uniform design, a
common leaf-size constraint naturally enforces comparable cell
lengths throughout the domain, irrespective of local
smoothness.  This difference in terminal-cell geometry explains
the tuning conflict observed in
Figure~\ref{fig:brownian-mse-decomposition}.

\begin{figure}[H]
    \centering
    \begin{tabular}{cc}
        \textbf{CART-MID} & \textbf{CART-MLS} \\[0.3em]

        \begin{subfigure}[t]{0.47\textwidth}
            \centering
            \includegraphics[width=\linewidth]{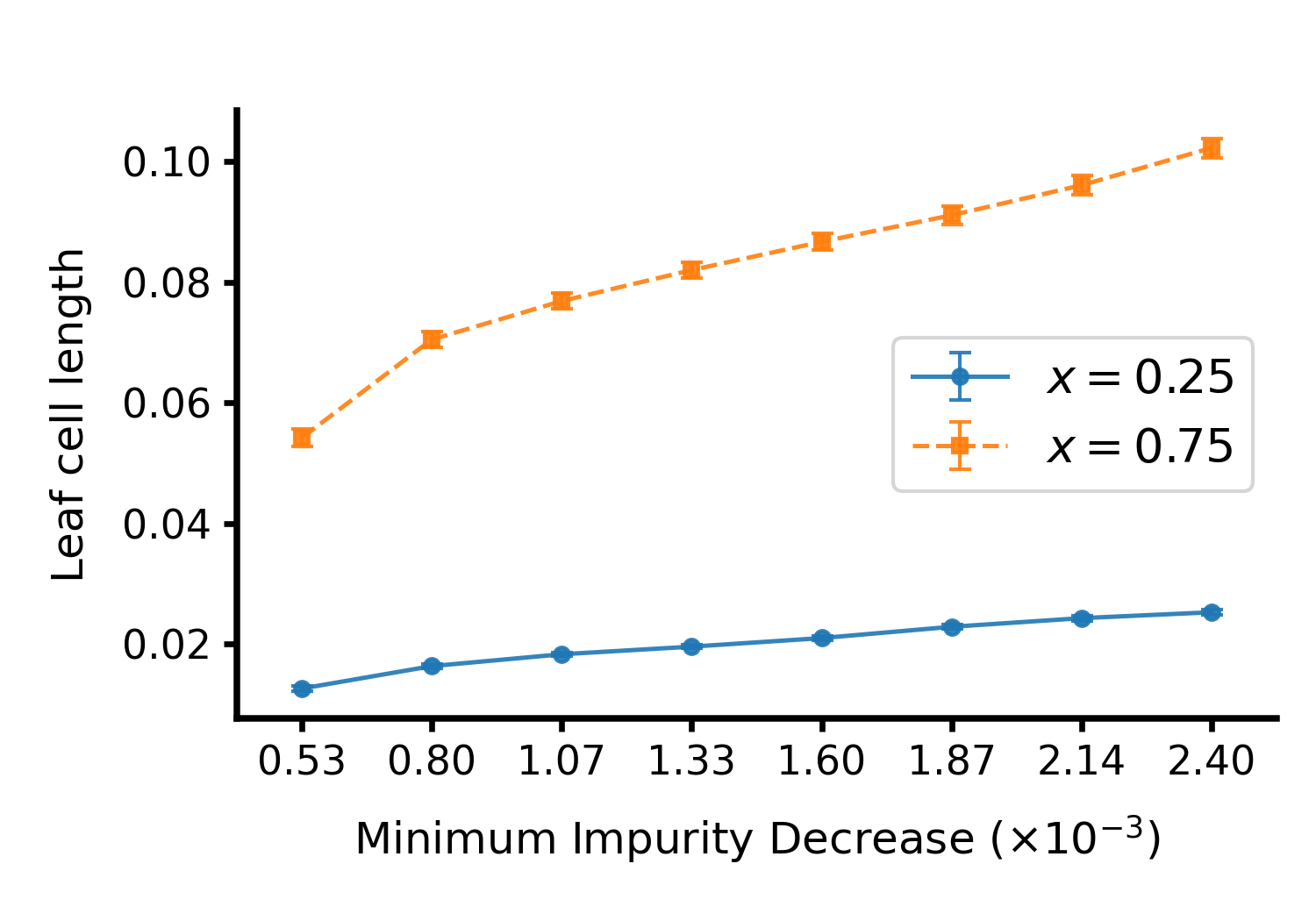}
            \label{fig:brownian-mid-cell-x025}
        \end{subfigure}
        &
        \begin{subfigure}[t]{0.47\textwidth}
            \centering
            \includegraphics[width=\linewidth]{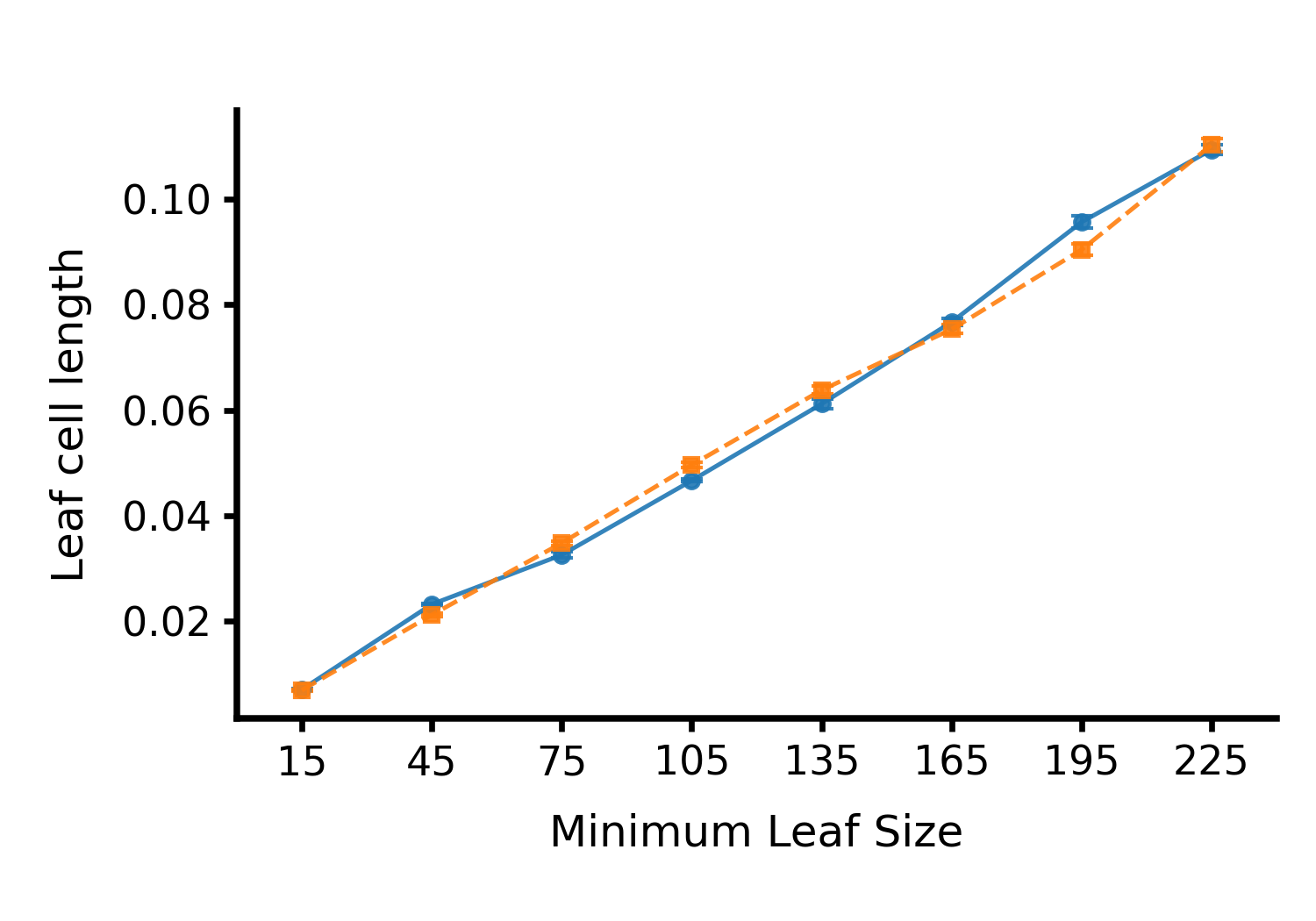}
            \label{fig:brownian-mls-cell-x025}
        \end{subfigure}
    \end{tabular}

    \caption{Mean terminal-cell lengths for the hybrid
    Brownian--linear signal.  Points are Monte Carlo averages
    over 400 repetitions, with error bars indicating
    $\pm1$ standard error.  The panels compare CART-MID and
    CART-MLS; each panel shows both $x=0.25$ and $x=0.75$.}
    \label{fig:brownian-cell-size}
\end{figure}

\FloatBarrier

\subsection{Adaptation across coordinate directions}

\subsubsection*{Experimental setup}

We next consider the two-dimensional additive signal
\[
f^*(\bx)=x_1^{1/2}+x_2^{1/4},
\qquad \bx\in[0,1]^2.
\]
At $\bx=\bzero$, its local H\"older exponents are
$\alpha_1=1/2$ and $\alpha_2=1/4$, so the signal is less
smooth along the second coordinate.  We generate
\[
\bX\sim\operatorname{Unif}([0,1]^2),
\qquad
Y=f^*(\bX)+\xi,
\qquad
\xi\sim\mathcal N(0,0.5),
\]
again taking $0.5$ to be the noise variance.  We use the sample
sizes
\[
n\in\bigl\{\lfloor1000\times1.5^k\rfloor:
k=0,\ldots,11\bigr\}.
\]
For each $n$, we fit CART-MID with
$\delta=0.1\log(n)/n$ over 400 Monte Carlo repetitions.  If
$\ell_j(\bx)$ is the side length, along coordinate $j$, of the
terminal cell containing $\bx$, we regress the Monte Carlo average
of $\log\ell_j(\bzero)$ on $\log n$ separately for $j=1,2$.

\subsubsection*{Results}

Figure~\ref{fig:additive-cell-size-origin} shows that terminal
cells shrink more rapidly along the less smooth coordinate.  The
fitted slopes are $-0.320$ for $j=1$ and $-0.529$ for $j=2$.
Their ordering, and particularly the second-coordinate slope,
are consistent with the theoretical scaling.

\begin{figure}[H]
    \centering
    \includegraphics[width=0.55\textwidth]{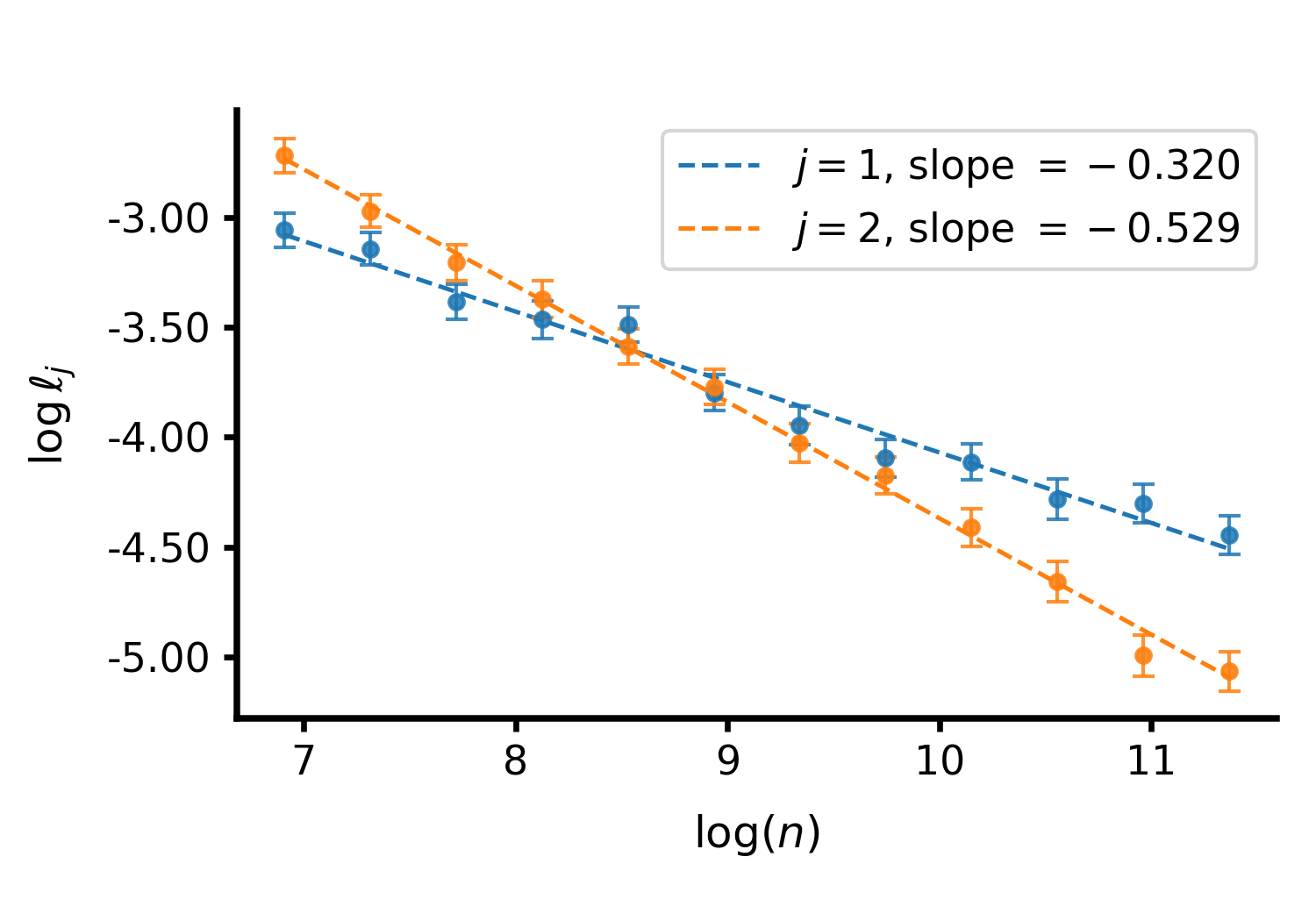}
    \caption{Coordinate-wise terminal-cell side lengths for
    $f^*(\bx)=x_1^{1/2}+x_2^{1/4}$ at $\bx=\bzero$. Points show Monte Carlo averages of $\log \ell_j(\bzero)$ over 400 repetitions,
    error bars show $\pm1$ standard error on the log scale, and
    dashed lines show least-squares fits of $\log \ell_j(\bzero)$ on
    $\log n$.}
    \label{fig:additive-cell-size-origin}
\end{figure}

Ignoring logarithmic factors, the theoretical benchmark follows
by balancing approximation and stochastic error.  If the
terminal side lengths at the origin are $\ell_1$ and $\ell_2$,
then
\[
\ell_1^{1/2}
\asymp
\ell_2^{1/4}
\asymp
(n\ell_1\ell_2)^{-1/2},
\]
which gives
\[
\ell_1\asymp n^{-1/4},
\qquad
\ell_2\asymp n^{-1/2}.
\]
Thus the benchmark slopes are $-1/4$ and $-1/2$.  The observed
finite-sample slopes provide direct evidence that CART-MID
adapts the shape of its terminal cells to anisotropic local
smoothness.

\FloatBarrier

\section{Discussion}
\label{sec:discussion}

In this paper, we have answered affirmatively the question of whether CART can achieve spatial adaptation similar to known results for ERM and Bayesian trees, and we have shown that the MID stopping rule plays a crucial role in enabling this adaptation.

Given the popularity of cost-complexity pruning (CCP) for regularizing CART and randomized ensembling via random forests, it is natural to ask whether our proofs can be extended to these settings.
We recall that our argument is based on establishing two local certificates $V(A)\lesssim\delta$ for each terminal cell and $\Delta_{\mathrm{pa}}\gtrsim\delta$ for its parent split. Standard cost-complexity pruning supplies the first certificate, but not necessarily the second: its optimality condition controls the aggregate gain of a retained branch, so strong descendant splits may subsidize a weak parent split. Random forests present a different obstruction. With feature subsampling, a small best split among the sampled coordinates need not imply small residual variation, because a relevant coordinate may be unavailable. Extending the theory to CCP or random forests therefore requires modifying the present parent--child certificate argument.

Also of interest is whether other sufficient conditions can be found to guarantee spatial adaptation for CART.
In particular, LNE is an extremely strong condition that can probably be weakened.

\subsection*{Acknowledgements}
The authors would like to thank Jason Klusowski and Li Ma for insightful discussions.
This work was supported by NUS Startup Grant A-8000448-00-00 and the Singapore Ministry of Education (MOE) AcRF Tier 1 Grants A-8002498-00-00 and A-8004458-00-00.

\bibliography{ref}

@inproceedings{Mazumder2024,
 author = {Mazumder, Rahul and Wang, Haoyue},
 booktitle = {Advances in Neural Information Processing Systems},
 editor = {A. Oh and T. Naumann and A. Globerson and K. Saenko and M. Hardt and S. Levine},
 pages = {57754--57782},
 publisher = {Curran Associates, Inc.},
 title = {On the Convergence of {CART} under Sufficient Impurity Decrease Condition},
 url = {https://proceedings.neurips.cc/paper_files/paper/2023/file/b418964bafb4fdd9aef9017301323a8a-Paper-Conference.pdf},
 volume = {36},
 year = {2023}
}

@article{rovckova2024ideal,
  title={Ideal {Bayesian} spatial adaptation},
  author={Ro{\v{c}}kov{\'a}, Veronika and Rousseau, Judith},
  journal={Journal of the American Statistical Association},
  volume={119},
  number={547},
  pages={2078--2091},
  year={2024},
  publisher={Taylor \& Francis}
}

@article{donoho1994ideal,
  title={Ideal spatial adaptation by wavelet shrinkage},
  author={Donoho, David L and Johnstone, Iain M},
  journal={{Biometrika}},
  volume={81},
  number={3},
  pages={425--455},
  year={1994},
  publisher={Oxford University Press}
}

@article{donoho1995wavelet,
  title={Wavelet shrinkage: asymptopia?},
  author={Donoho, David L and Johnstone, Iain M and Kerkyacharian, G{\'e}rard and Picard, Dominique},
  journal={Journal of the Royal Statistical Society: Series B (Methodological)},
  volume={57},
  number={2},
  pages={301--337},
  year={1995},
  publisher={Wiley Online Library}
}

@article{chi2022asymptotic,
  title={Asymptotic properties of high-dimensional random forests},
  author={Chi, Chien-Ming and Vossler, Patrick and Fan, Yingying and Lv, Jinchi},
  journal={The Annals of Statistics},
  volume={50},
  number={6},
  pages={3415--3438},
  year={2022},
  publisher={Institute of Mathematical Statistics}
}

@book{breiman1984classification,
  author    = {Leo Breiman and Jerome Friedman and Charles J. Stone and Richard A. Olshen},
  title     = {Classification and Regression Trees},
  year      = {1984},
  publisher = {CRC Press},
  address   = {Belmont, CA}
}

@article{jeong2023art,
  title     = {The Art of BART: Minimax Optimality over Nonhomogeneous Smoothness in High Dimension},
  author    = {Jeong, Seonghyun and Rockova, Veronika},
  journal   = {Journal of Machine Learning Research},
  volume    = {24},
  number    = {337},
  pages     = {1--65},
  year      = {2023},
  url       = {https://www.jmlr.org/papers/volume24/22-0382/22-0382.pdf}
}

@article{hoffman2002random,
  title={Random rates in anisotropic regression (with a discussion and a rejoinder by the authors)},
  author={Hoffman, M and Lepski, Oleg},
  journal={The Annals of Statistics},
  volume={30},
  number={2},
  pages={325--396},
  year={2002},
  publisher={Institute of Mathematical Statistics}
}

@article{lepski2015adaptive,
  title={ADAPTIVE ESTIMATION OVER ANISOTROPIC FUNCTIONAL CLASSES VIA ORACLE APPROACH},
  author={Lepski, Oleg},
  journal={The Annals of Statistics},
  volume={43},
  number={3},
  pages={1178--1242},
  year={2015}
}

@article{stone1982optimal,
  title={Optimal global rates of convergence for nonparametric regression},
  author={Stone, Charles J},
  journal={{The Annals of Statistics}},
  pages={1040--1053},
  year={1982},
  publisher={JSTOR}
}

@article{ibragimov1982bounds,
  title={Bounds for the risks of non-parametric regression estimates},
  author={Ibragimov, IA and Khas’ minskii, RZ},
  journal={Theory of Probability \& Its Applications},
  volume={27},
  number={1},
  pages={84--99},
  year={1982},
  publisher={SIAM}
}

@article{scornet2016random,
  title={Random forests and kernel methods},
  author={Scornet, Erwan},
  journal={IEEE Transactions on Information Theory},
  volume={62},
  number={3},
  pages={1485--1500},
  year={2016},
  publisher={IEEE}
}

@article{tan2024statistical,
  title={Statistical-computational trade-offs for recursive adaptive partitioning estimators},
  author={Tan, Yan Shuo and Klusowski, Jason M and Balasubramanian, Krishnakumar},
  journal={The Annals of Statistics},
  volume={54},
  number={3},
  pages={1262--1288},
  year={2026},
  publisher={Institute of Mathematical Statistics}
}

@inproceedings{tan2022cautionary,
  title={A cautionary tale on fitting decision trees to data from additive models: generalization lower bounds},
  author={Tan, Yan Shuo and Agarwal, Abhineet and Yu, Bin},
  booktitle={International Conference on Artificial Intelligence and Statistics},
  pages={9663--9685},
  year={2022},
  organization={PMLR}
}

@inproceedings{syrgkanis2020estimation,
  title={Estimation and inference with trees and forests in high dimensions},
  author={Syrgkanis, Vasilis and Zampetakis, Manolis},
  booktitle={Conference on {Learning Theory}},
  pages={3453--3454},
  year={2020},
  organization={PMLR}
}

@inproceedings{NIPS2010_refined_rade,
 author = {Srebro, Nathan and Sridharan, Karthik and Tewari, Ambuj},
 booktitle = {Advances in Neural Information Processing Systems},
 editor = {J. Lafferty and C. Williams and J. Shawe-Taylor and R. Zemel and A. Culotta},
 pages = {},
 publisher = {Curran Associates, Inc.},
 title = {Smoothness, Low Noise and Fast Rates},
 url = {https://proceedings.neurips.cc/paper_files/paper/2010/file/76cf99d3614e23eabab16fb27e944bf9-Paper.pdf},
 volume = {23},
 year = {2010}
}

@article{klusowski2024large,
  title={Large scale prediction with decision trees},
  author={Klusowski, Jason M and Tian, Peter M},
  journal={Journal of the American Statistical Association},
  volume={119},
  number={545},
  pages={525--537},
  year={2024},
  publisher={Taylor \& Francis}
}

@article{donoho1997cart,
  title={{CART} and best-ortho-basis: a connection},
  author={Donoho, David L},
  journal={{The Annals of Statistics}},
  volume={25},
  number={5},
  pages={1870--1911},
  year={1997},
  publisher={Institute of Mathematical Statistics}
}

@article{rudin2022interpretable,
author = {Cynthia Rudin and Chaofan Chen and Zhi Chen and Haiyang Huang and Lesia Semenova and Chudi Zhong},
title = {{Interpretable machine learning: Fundamental principles and 10 grand challenges}},
volume = {16},
journal = {Statistics Surveys},
number = {none},
publisher = {Amer. Statist. Assoc., the Bernoulli Soc., the Inst. Math. Statist., and the Statist. Soc. Canada},
pages = {1 -- 85},
year = {2022},
doi = {10.1214/21-SS133},
URL = {https://doi.org/10.1214/21-SS133}
}

@article{klusowski2020sparse,
  title={Sparse learning with {CART}},
  author={Klusowski, Jason},
  journal={Advances in Neural Information Processing Systems},
  volume={33},
  pages={11612--11622},
  year={2020}
}

@article{cattaneo2025honest,
  title={Accuracy Limits of Causal Trees for Individualized Treatment Effects},
  author={Cattaneo, Matias D and Klusowski, Jason M and Yu, Ruiqi Rae},
  journal={arXiv preprint arXiv:2509.11381},
  year={2026}
}

@article{athey2016recursive,
  title={Recursive partitioning for heterogeneous causal effects},
  author={Athey, Susan and Imbens, Guido},
  journal={Proceedings of the National Academy of Sciences},
  volume={113},
  number={27},
  pages={7353--7360},
  year={2016},
  publisher={National Academy of Sciences}
}

@book{hastie2009elements,
  title={The elements of statistical learning: data mining, inference, and prediction},
  author={Hastie, Trevor and Tibshirani, Robert and Friedman, Jerome H and Friedman, Jerome H},
  volume={2},
  year={2009},
  publisher={Springer}
}

@article{rudin2019stop,
  title={Stop explaining black box machine learning models for high stakes decisions and use interpretable models instead},
  author={Rudin, Cynthia},
  journal={Nature Machine Intelligence},
  volume={1},
  number={5},
  pages={206--215},
  year={2019},
  publisher={Nature Publishing Group}
}

@article{laurent1976constructing,
  title={{Constructing optimal binary decision trees is NP-complete}},
  author={Hyafil, Laurent and Rivest, Ronald L},
  journal={Information Processing Letters},
  volume={5},
  number={1},
  pages={15--17},
  year={1976}
}

@article{xu2026statistical,
  title={On the Statistical Optimality of Optimal Decision Trees},
  author={Xu, Zineng and Ghosh, Subhro and Tan, Yan Shuo},
  journal={arXiv preprint arXiv:2603.05340},
  year={2026},
  eprint={2603.05340},
  archivePrefix={arXiv},
  primaryClass={stat.ML},
  doi={10.48550/arXiv.2603.05340},
  url={https://arxiv.org/abs/2603.05340}
}

@book{tsybakov2008introduction,
  title={Introduction to Nonparametric Estimation},
  author={Tsybakov, A.B.},
  isbn={9780387790527},
  lccn={2008939894},
  series={Springer Series in Statistics},
  url={https://books.google.com.sg/books?id=mwB8rUBsbqoC},
  year={2008},
  publisher={Springer New York}
}

@techreport{morgan1973thaid,
  author      = {Morgan, J. N. and Messenger, R. C.},
  title       = {THAID, a sequential analysis program for the analysis of nominal scale dependent variables},
  institution = {Institute for Social Research, University of Michigan},
  year        = {1973},
  address     = {Ann Arbor},
  type        = {Technical report}
}

@article{pedregosa2011scikit,
  title={Scikit-learn: Machine Learning in {P}ython},
  author={Pedregosa, F. and others},
  journal={Journal of Machine Learning Research},
  volume={12},
  pages={2825--2830},
  year={2011}
}
\newpage
\appendix

\section{Algorithmic details}
\begin{algorithm}[h]
	\caption{Early-stopped CART with Minimum Impurity Decrease for Regression}
	\label{alg:early-stopping-cart}
	\KwIn{Data $\mathcal D_n=\{(\bX_i,Y_i):i\in[n]\}$; threshold $\delta>0$.}
	
	Initialize the root node $A_0=[0,1]^d$ and set the active leaf set $\mathcal{A}\leftarrow\{A_0\}$\;
	
	\While{there exists $A\in\mathcal{A}$ with an admissible split and $\widehat{\Delta}_{\max}(A)\ge \delta$}{
		Set $\mathcal{A}_{\mathrm{new}}\leftarrow \mathcal{A}$\;
		\ForEach{$A\in\mathcal{A}$ with an admissible split and $\widehat{\Delta}_{\max}(A)\ge \delta$}{
			Select $(\hat{j},\hat{b})$ by \eqref{equ::jstar}\;
			Split $A$ into $A_L$ and $A_R$ at $(\hat{j},\hat{b})$\;
			Update $\mathcal{A}_{\mathrm{new}}\leftarrow (\mathcal{A}_{\mathrm{new}}\setminus\{A\})\cup\{A_L,A_R\}$\;
		}
		Set $\mathcal{A}\leftarrow \mathcal{A}_{\mathrm{new}}$\;
	}
	
	\ForEach{leaf cell $A\in\mathcal{A}$}{
		Set $\bar Y_A \leftarrow \frac{1}{N(A)}\sum_{i\in\mathcal{I}_A} Y_i$,
		where $\mathcal{I}_A:=\{i\in[n]:\bX_i\in A\}$ and $N(A):=|\mathcal{I}_A|$\;
	}
	\KwOut{Estimator $\hat f_\delta(\bx)=\sum_{A\in\mathcal{A}}\bar Y_A\,\1\{\bx\in A\}$.}
\end{algorithm}

\section{Concentration tools}
\label{app:concentration-tools}

\subsection{Proofs for uniform concentration with empirical Rademacher complexity}
In this section, we prove the concentration results stated in
Section~\ref{section:unif-conc}. Our strategy adapts the arguments of
Section~7 in \citet{xu2026statistical}. Throughout, let
$\partfunc_r \coloneqq \{\1_A(\bx) : A \in \mathcal{A}_{d,r}\}$ denote
the class of indicators of cells in $\mathcal{A}_{d,r}$, and, for any
function class $\funcclass$ and $a > 0$, define the localized class
$\funcclass_a \coloneqq \{f \in \funcclass : \norm{f}_n \leq a\}$.

The key step is to establish analogues of the self-normalized deviation
bounds of Lemma~7.2 in \citet{xu2026statistical}, now holding uniformly
over $\partfunc_r$. As in their analysis, the starting point is a bound
on the empirical Rademacher complexity of the localized class
$\partfunc_{r,a}$, in the spirit of equation~(A.15) in
\citet{xu2026statistical}. Whereas their bound is taken over
decision-tree functions, with the number of leaves controlling the
complexity, here the number of active coordinates $r$ plays that role,
and the resulting complexity parameter is $r\log(nd)$, as shown in
Lemma~\ref{lem:subG_multiplier_sup_exp}.
The sample-trace argument in Lemma~\ref{lem:subG_multiplier_sup_cond_exp}
and the bound in Lemma~\ref{lem:subG_multiplier_sup_exp} provide the
analogues of their Lemmas~A.1 and A.2, respectively;
Lemma~\ref{lem:final_conc} adapts their high-probability and peeling steps.

\begin{lemma}
\label{lem:subG_multiplier_sup_cond_exp}
    For any fixed $\sample$, suppose $\boldsymbol{\epsilon} \coloneqq \braces{\epsilon_1,\epsilon_2,\ldots,\epsilon_n}$ are independent Rademacher random variables.
    For any $r\in[d]$ and $0 < a \leq 1$, conditioned on $\sample$,
    \begin{equation}
        \E_{\boldsymbol{\epsilon}}\bracks*{\sup_{\substack{f \in \partfunc_r \\
        \norm{f}_n \leq a}} \inprod{f,\epsilon}_n} \leq Ca\paren*{\frac{r\log(nd)}{n}}^{1/2}
    \end{equation}
    where $C > 0$ is a universal constant.
\end{lemma}
\begin{proof}

Let $\partfunc_{r,a}\coloneqq\{\1_A(\bx):A\in\mathcal{A}_{d,r},\norm{\1_A}_n\leq a\}$ and
write $\operatorname{tr}_{\sample}(A):=(\1_A(\bX_i))_{i=1}^n$ for the
trace of a cell on the sample. In any fixed coordinate, an interval
selects a contiguous block after the sample coordinates are sorted, so
there are at most $\binom{n+1}{2}+1$ distinct one-dimensional traces,
including the empty trace. A cell with exactly $k$ active coordinates
therefore has at most
$\binom{d}{k}(\binom{n+1}{2}+1)^k$ distinct traces. Since the empirical
metric depends only on these traces, it follows that
 \begin{equation*}
     \mathcal{N}(t,\partfunc_{r,a},\norm{\cdot}_n)
     \leq \sum_{k=0}^r\binom{d}{k}
     \paren*{\binom{n+1}{2}+1}^k
     \leq 2d^rn^{2r}
     \leq \exp\{3r\log(nd)\},
 \end{equation*}
where the last two inequalities use $n\ge2$ and $r\ge1$.
 Using Lemma A.3 in \citet{NIPS2010_refined_rade}, we complete the proof:
 \begin{equation}
    \begin{split}
         \E_{\boldsymbol{\epsilon}}\bracks*{\sup_{f \in \partfunc_{r,a} } \inprod{f,\epsilon}_n}
        &\leq \frac{12}{\sqrt{n}}\int_{0}^{a/2}
        \sqrt{\log\mathcal{N}(t,\partfunc_{r,a},\norm{\cdot}_n )}dt\\
        &\leq Ca\paren*{\frac{r\log(nd)}{n}}^{1/2}.
    \end{split}
\end{equation}
\end{proof}

\begin{lemma}
\label{lem:subG_multiplier_sup_exp}
    Let $\braces*{\bX_1,\bX_2,\ldots,\bX_n}$ be IID random variables and suppose $\boldsymbol{\epsilon} \coloneqq \braces{\epsilon_1,\epsilon_2,\ldots,\epsilon_n}$ are independent Rademacher random variables.
    For any $r\in[d]$ and $0 < a \leq 1$,
    we have the bound
    \begin{equation} \label{eq:subG_multiplier_sup_exp}
        \E\bracks*{\sup_{\substack{f \in \partfunc_r \\
        \norm{f}_2 \leq a,\norm{f}_\infty \leq 1}} \inprod{f,\epsilon}_n} \leq Ca\paren*{\frac{r\log(nd)}{n}}^{1/2} + \frac{Cr\log(nd)}{n}.
    \end{equation}
    where $C > 0$ is a universal constant.
\end{lemma}

\begin{proof}   
    For convenience, let us use $\mathcal{G}$ to denote the set over which the supremum is taken on the left-hand side of \eqref{eq:subG_multiplier_sup_exp} and denote the whole quantity as $\Rad{\mathcal{G}}$.
    Notice that for each fixed $\sample$, 
    we have the inclusion
    \begin{equation}
        \mathcal{G} \subset \braces*{f \in \partfunc_r \colon \norm{f}_n \leq \sup_{f \in \mathcal{G}} \norm{f}_n}.
    \end{equation}
    Using Lemma~\ref{lem:subG_multiplier_sup_cond_exp}, the conditional expectation satisfies
    \begin{equation}
    \label{eq:subG_multiplier_sup_exp_helper}
        \E\bracks*{\sup_{f \in \mathcal{G}} \inprod{f,\epsilon}_n~\vline~ \sample} \leq C\paren*{\frac{r\log(nd)}{n}}^{1/2} \sup_{f \in \mathcal{G}} \norm{f}_n
    \end{equation}
    for some universal constant $C > 0$.
    Next, it is easy to compute
    \begin{equation}
        \E\bracks*{\sup_{f \in \mathcal{G}} \norm{f}_n} \leq \E\bracks*{\sup_{f \in \mathcal{G}} \abs*{\norm{f}_n^2 - \norm{f}_2^2}}^{1/2} + \sup_{f \in \mathcal{G}} \norm{f}_2.
    \end{equation}
    The second term on the right-hand side is at most $a$ by the definition of $\mathcal{G}$, while the first term can be bounded as
    \begin{equation}
    \label{eq:subG_multiplier_sup_exp_helper2}
        \E\bracks*{\sup_{f \in \mathcal{G}} \abs*{\norm{f}_n^2 - \norm{f}_2^2}} \leq C\Rad{\braces*{f^2 \colon f \in \mathcal{G}}} \leq C\Rad{\mathcal{G}}.
    \end{equation}
    Here, the first inequality is by symmetrization, while the second inequality uses the Ledoux-Talagrand contraction inequality and the assumption that all functions in $\mathcal{G}$ have supremum norm bounded by $1$.
    Taking a further expectation over $\sample$ in \eqref{eq:subG_multiplier_sup_exp_helper} and plugging these bounds back into the resulting inequality, we get
    \begin{equation}
        \Rad{\mathcal{G}} \leq C\paren*{\frac{r\log(nd)}{n}}^{1/2}\paren*{\Rad{\mathcal{G}}^{1/2} + a}.
    \end{equation}
    This is a quadratic inequality in $\Rad{\mathcal{G}}^{1/2}$, which can be solved (and squared) to get \eqref{eq:subG_multiplier_sup_exp}
    \begin{equation}
    \label{eq:subG_multiplier_sup_exp_helper3}
        \Rad{\mathcal{G}} \leq 2Ca\paren*{\frac{r\log(nd)}{n}}^{1/2} + 4C^2\paren*{\frac{r\log(nd)}{n}}.
    \end{equation}   
\end{proof}

\begin{lemma}\label{lem:final_conc}
    Assume sub-Gaussian response
    (Assumption~\ref{assum:subgaussian-response}), and let
    $\braces*{(\bX_1,g_1),(\bX_2,g_2),\ldots,(\bX_n,g_n)}$ comprise
    IID pairs from a joint distribution on $[0,1]^d\times\R$. Assume
    $|g_i|\le \norm{g}_\infty$ almost surely.
    For any $u>0$, with probability at least $1-e^{-u}$, the following holds for every $r\in [d]$ and every $f \in \partfunc_r$:
    \begin{equation}
    \label{eq:function_multiplier_final_conc}
        \abs*{\inprod{f,g}_n - \inprod{f,g}_{P}} \leq C\norm{g}_\infty\left[\norm{f}_2\paren*{\frac{r\log(nd) + u}{n}}^{1/2} + \paren*{\frac{r\log(nd) + u}{n}}\right].
	\end{equation}
    \begin{equation}
    \label{eq:function_multiplier_final_conc_noise}
        \abs*{\inprod{f,\xi}_n} \leq CK\left[\norm{f}_2\paren*{\frac{r\log(nd) + u}{n}}^{1/2} + \paren*{\frac{r\log(nd) + u}{n}}\right].
	\end{equation}
    where $C>0$ is a universal constant.
\end{lemma}
\begin{proof}
Let $\partfunc_r^\pm:=\partfunc_r\cup(-\partfunc_r)$. Passing to this
signed class changes the covering numbers by at most a factor of two
and lets us treat absolute suprema as ordinary suprema. Fix $r\in[d]$,
$0<a\leq 1$, and $u>0$, and write $\varepsilon_{n,r,u}\coloneqq (r\log(nd)+u)/n$.

The argument used to prove Lemma~7.1 of \citet{xu2026statistical}, with
Lemmas~\ref{lem:subG_multiplier_sup_cond_exp} and
\ref{lem:subG_multiplier_sup_exp} in place of their Lemmas~A.1 and A.2,
gives the localized bounds
\begin{align}
\sup_{\substack{f\in\partfunc_r\\ \norm{f}_2\leq a}}
\abs*{\inprod{f,g}_n-\inprod{f,g}_P}
&\leq C\norm{g}_\infty\bigl(a\varepsilon_{n,r,u}^{1/2}+\varepsilon_{n,r,u}\bigr),
\label{eq:function-multiplier-localized}\\
\sup_{\substack{f\in\partfunc_r\\ \norm{f}_2\leq a}}
\abs*{\inprod{f,\xi}_n}
&\leq CK\bigl(a\varepsilon_{n,r,u}^{1/2}+\varepsilon_{n,r,u}\bigr),
\label{eq:noise-multiplier-localized}
\end{align}
each with probability at least $1-Ce^{-u}$.

For completeness, the first bound follows by applying symmetrization
and contraction to the class $(\bx,z)\mapsto f(\bx)z$, followed by
Bousquet's inequality; its envelope is $\norm{g}_\infty$ and its
variance is at most $\norm{g}_\infty^2a^2$. This argument remains valid
when $g_i$ is a bounded random multiplier jointly distributed with
$\bX_i$. For the second bound, the norm-process part of the same
argument first gives
$\sup_{f\in\partfunc_r:\,\norm{f}_2\leq a}\norm{f}_n
\leq C(a+\varepsilon_{n,r,u}^{1/2})$.
Conditional on the covariates, the sample-trace count from the proof of
Lemma~\ref{lem:subG_multiplier_sup_cond_exp} and the conditional
sub-Gaussianity of $\xi_i$ then give
\eqref{eq:noise-multiplier-localized}. The use of
$\partfunc_r^\pm$ accounts for the symmetry invoked in the proof of
Lemma~7.1 of \citet{xu2026statistical}.

It remains to remove the localization radius. Follow the peeling
argument of Lemma~7.2 in \citet{xu2026statistical}, applying
\eqref{eq:function-multiplier-localized} and
\eqref{eq:noise-multiplier-localized} at radii $a_k=e^{-k+1}$ for
$k\in[n]$ and $r\in[d]$, with the tail parameter in those bounds set to
$v=u+C\log(nd)$ for a sufficiently large universal $C$. A union bound makes the event simultaneous in $k$ and
$r$, and the additional logarithm is absorbed by $r\log(nd)$. Choosing
the shell containing $\norm{f}_2$ (and using the linear $\varepsilon_{n,r,u}$ term
when $\norm{f}_2\leq e^{-n+1}$) now gives
\eqref{eq:function_multiplier_final_conc} and
\eqref{eq:function_multiplier_final_conc_noise}.
\end{proof}

\begin{lemma}[Sub-Gaussian response multiplier over cells]\label{lem:subg_response_multiplier_cells}
Assume sub-Gaussian response (Assumption~\ref{assum:subgaussian-response}).
For any $u>0$, with probability at least $1-e^{-u}$, the following holds for every $r\in[d]$ and every $f\in\partfunc_r$:
\begin{equation}\label{eq:subg-response-multiplier-cells}
\abs*{\inprod{f,Y}_n-\E[f(\bX)f^*(\bX)]}
\leq
C K
\left[
\norm{f}_2\left(\frac{r\log(nd)+u}{n}\right)^{1/2}
+
\frac{r\log(nd)+u}{n}
\right].
\end{equation}
\end{lemma}

\begin{proof}
Write $Y=f^*(\bX)+\xi$.
Assumption~\ref{assum:subgaussian-response} implies
$\norm{f^*}_\infty\lesssim K$ and
$\sup_{\bx}\sgnorm{\xi\mid \bX=\bx}\lesssim K$.
Work on the event guaranteed by Lemma~\ref{lem:final_conc}.
For the deterministic part, Lemma~\ref{lem:final_conc} with $g=f^*$
gives
\[
\abs*{\inprod{f,f^*}_n-\E[f(\bX)f^*(\bX)]}
\lesssim
K\left[
\norm{f}_2\left(\frac{r\log(nd)+u}{n}\right)^{1/2}
+
\frac{r\log(nd)+u}{n}
\right].
\]
Combining this bound with
\eqref{eq:function_multiplier_final_conc_noise} proves the lemma.
\end{proof}

\begin{lemma}[Uniform concentration of cell counts]
\label{lemma:uniform_cell_count}
There exists a universal constant $C>0$ such that, for any $u>0$, with probability at least $1-e^{-u}$, for every $r\in [d]$, the following statements hold uniformly for all $A\in\mathcal A_{d,r}$:
\begin{enumerate}
\item
\begin{equation}
\label{eq:cell-conc-1}
\left|
\left(\frac{N(A)}{n}\right)^{1/2}
-\P(A)^{1/2}
\right|
\leq
C
\left(\frac{r\log(nd)+u}{n}\right)^{1/2}.
\end{equation}
\item There exists a universal constant $C_1>0$ such that, if $\P(A)\vee (N(A)/n)\geq C_1(r\log(nd)+u)/n$, then
\begin{equation}
\label{eq:cell-conc-2}
\frac{1}{2}\leq \frac{\P(A)n}{N(A)}\leq 2.
\end{equation}
\item For any constant $C_2>0$, if $\P(A)\wedge (N(A)/n)\leq C_2(r\log(nd)+u)/n$, then
\begin{equation}
\label{eq:cell-conc-3}
\P(A)\vee\frac{N(A)}{n}
\leq
(2C_2\vee C_1)\frac{r\log(nd)+u}{n}.
\end{equation}
\end{enumerate}
\end{lemma}

\begin{lemma}[Uniform concentration of bounded cell averages]
\label{lemma:unif_conc_bounded_cell_ave}
Suppose $Z$ is a bounded random variable defined on the same probability
space as $\bX$ and $Y$, not necessarily independent of $\bX$ and $Y$.
There exists a constant $C$ such that, for any $u>0$, with probability at
least $1-e^{-u}$, the following inequality holds uniformly for every
$r\in[d]$ and $A\in\mathcal{A}_{d,r}$:
\begin{equation}
\label{eq:conclude-bounded-cell-mean}
\P(A)^{1/2}
\bigl|\overline Z_A-\E[Z\mid \bX\in A]\bigr|
\lesssim
\|Z\|_\infty
\left(\frac{r\log(nd)+u}{n}\right)^{1/2},
\end{equation}
where $\overline Z_A=N(A)^{-1}\sum_{i:\bX_i\in A} Z_i$ whenever
$N(A)>0$, and $\overline Z_A=0$ otherwise.
\end{lemma}

\begin{proof}
Apply Lemmas~\ref{lemma:uniform_cell_count} and \ref{lem:final_conc}
with $u+\log 2$ in place of $u$, and work on the intersection of the
resulting events.  This intersection has probability at least
$1-e^{-u}$ after increasing constants.
Fix $r$ and $A\in\mathcal{A}_{d,r}$, and write
$\varepsilon_{n,r,u}\coloneqq (r\log(nd)+u)/n$.

If $\P(A)\leq c\varepsilon_{n,r,u}$, then boundedness gives
\[
\P(A)^{1/2}
\bigl|\overline Z_A-\E[Z\mid \bX\in A]\bigr|
\leq
2\norm{Z}_\infty \P(A)^{1/2}
\leq
C\norm{Z}_\infty\varepsilon_{n,r,u}^{1/2}.
\]
Now suppose $\P(A)>c\varepsilon_{n,r,u}$, with $c$ large enough that
Lemma~\ref{lemma:uniform_cell_count} gives
$P_n(A)\coloneqq N(A)/n\asymp \P(A)$.
With $f=\1_A$, Lemma~\ref{lem:final_conc} gives
\[
\abs*{P_n(Z\1_A)-\E[Z\1_A(\bX)]}
\leq
C\norm{Z}_\infty
\left[
\P(A)^{1/2}\varepsilon_{n,r,u}^{1/2}
+\varepsilon_{n,r,u}
\right]
\leq
C\norm{Z}_\infty
\P(A)^{1/2}\varepsilon_{n,r,u}^{1/2},
\]
where $P_n(Z\1_A)=n^{-1}\sum_i Z_i\1_A(\bX_i)$ and the last step uses the
lower bound on $\P(A)$.
Since
\[
\overline Z_A-\E[Z\mid \bX\in A]
=
\frac{P_n(Z\1_A)-\E[Z\1_A(\bX)]}{P_n(A)}
+\E[Z\mid \bX\in A]\frac{\P(A)-P_n(A)}{P_n(A)},
\]
and Lemma~\ref{lemma:uniform_cell_count} also gives
$|\P(A)-P_n(A)|\lesssim \P(A)^{1/2}\varepsilon_{n,r,u}^{1/2}$ in this
large-cell regime, the desired bound follows.
\end{proof}

\begin{proof}[Proof of Lemma~\ref{lemma:unif_conc_cell_ave}]
Let $u_1:=u+\log 3$.
Fix $r\in[d]$ and write
$\varepsilon_{n,r,u}\coloneqq(r\log(nd)+u)/n$.
Work on the intersection of the events in
Lemmas~\ref{lemma:uniform_cell_count} and
\ref{lem:subg_response_multiplier_cells}, each applied with $u_1$ in
place of $u$, together with the standard root-cell concentration bound
$|\overline Y_{[0,1]^d}-\E[Y]|\lesssim K(u_1/n)^{1/2}$.
A union bound shows that this event has probability at least $1-e^{-u}$,
and the additive constant in $u_1$ can be absorbed into the universal
constants below.
For cells with
$\P(A)\ge c\varepsilon_{n,r,u}$, where $c$ is a sufficiently large
universal constant, Lemma~\ref{lemma:uniform_cell_count} gives
$P_n(A)=N(A)/n\asymp \P(A)$ uniformly.
Lemma~\ref{lem:subg_response_multiplier_cells}, applied with
$f=\1_A$, gives
\[
\left|
\frac1n\sum_{i=1}^n Y_i\1_A(\bX_i)
-
\E[f^*(\bX)\1_A(\bX)]
\right|
\lesssim
K(\P(A)\varepsilon_{n,r,u})^{1/2},
\]
where the lower-order $\varepsilon_{n,r,u}$ term is absorbed by the
assumed lower bound on $\P(A)$.
Also,
\[
\left|\frac{N(A)}{n}-\P(A)\right|
\lesssim
(\P(A)\varepsilon_{n,r,u})^{1/2}.
\]
Using
\begin{align*}
\bar Y_A-\E[f^*(\bX)\mid \bX\in A]
&=
\frac{n^{-1}\sum_{i=1}^n Y_i\1_A(\bX_i)-\E[f^*(\bX)\1_A(\bX)]}{P_n(A)}\\
&\quad+
\E[f^*(\bX)\mid \bX\in A]\,\frac{\P(A)-P_n(A)}{P_n(A)},
\end{align*}
and $\norm{f^*}_\infty\lesssim K$, we obtain
\[
\P(A)^{1/2}\bigl|\bar Y_A-\E[f^*(\bX)\mid \bX\in A]\bigr|
\lesssim
K \varepsilon_{n,r,u}^{1/2}.
\]
Taking the union over $r\in[d]$ is already included in the stated complexity, and the claim follows.
For the root cell, the claimed inequality follows directly from the
root-cell concentration bound and $\E[Y]=\E[f^*(\bX)]$.
\end{proof}

\begin{proof}[Proof of Lemma~\ref{lemma:uniform_cell_count}]
    Fix $r\in[d]$ and $A\in\mathcal A_{d,r}$.
    To prove \eqref{eq:cell-conc-1}, by Lemma~\ref{lem:final_conc} with $g\equiv 1$, we have
    \begin{equation}\label{eq:proof-cell-conc}
        \abs*{\frac{N(A)}{n}-\P(A)}\leq C\paren*{\paren*{\P(A)\frac{r\log(nd) + u}{n}}^{1/2} + \paren*{\frac{r\log(nd) + u}{n}}}.
    \end{equation}
     This is a quadratic inequality in $\P(A)$, which can be solved (and squared) to get \eqref{eq:cell-conc-1}.

    To prove \eqref{eq:cell-conc-2}, set $c_1=2C^2/(\sqrt{2}-1)^2$, where $C$ is the constant in \eqref{eq:cell-conc-1}. We consider only the case $\P(A)>c(r\log(nd)+u)/n$, since the case $N(A)/n>c(r\log(nd)+u)/n$ follows by a similar argument. By \eqref{eq:cell-conc-1}, one direction of \eqref{eq:cell-conc-2} follows
     \begin{equation*}
        \begin{split}
            \paren*{\frac{N(A)}{n}}^{1/2}&\leq \paren*{\P(A)}^{1/2}+C\paren*{\frac{r\log(nd) + u}{n}}^{1/2}\\
            &\leq \paren*{\P(A)}^{1/2}+(\sqrt{2}-1)\paren*{\P(A)}^{1/2}\\
            &\leq \sqrt{2}\paren*{\P(A)}^{1/2}.
        \end{split}
     \end{equation*}
     Similarly from \eqref{eq:cell-conc-1}, we finish the proof of \eqref{eq:cell-conc-2}:
     \begin{equation*}
         \begin{split}
             \paren*{\frac{N(A)}{n}}^{1/2}&\geq \paren*{\P(A)}^{1/2}-C\paren*{\frac{r\log(nd) + u}{n}}^{1/2}\\
            &\geq \paren*{\P(A)}^{1/2}-\paren*{1-\frac{1}{\sqrt{2}}}\paren*{\P(A)}^{1/2}\\
            &\geq \frac{1}{\sqrt{2}}\paren*{\P(A)}^{1/2}.
         \end{split}
     \end{equation*}

    Equation \eqref{eq:cell-conc-3} follows directly from \eqref{eq:cell-conc-2}. We consider the case $\P(A)\leq c_2(r\log(nd)+u)/n$, as the complementary case is analogous. If $N(A)\leq c_1(r\log(nd)+u)$, then \eqref{eq:cell-conc-3} holds trivially; otherwise, \eqref{eq:cell-conc-3} follows immediately from \eqref{eq:cell-conc-2}. 
\end{proof}

\subsection{Proof of Lemma~\ref{lemma:unif_conc_id}}

\begin{lemma}[Uniform weighted cell means]
\label{lem:uniform-weighted-cell-means}
Assume sub-Gaussian response (Assumption~\ref{assum:subgaussian-response}).
For $B\in\mathcal R_d$, write
\[
p_B:=\P(B),\qquad \hat p_B:=\frac{N(B)}n,
\qquad
\mu_B:=\E[f^*(\bX)\mid \bX\in B],
\]
with $\mu_B=0$ when $p_B=0$, and set $\overline Y_B=0$ when
$N(B)=0$. For any $u>0$, with probability at least $1-e^{-u}$,
uniformly for every $r\in[d]$ and $B\in\mathcal A_{d,r}$,
\begin{equation}\label{eq:uniform-weighted-cell-means}
\left|\hat p_B^{1/2}\overline Y_B-p_B^{1/2}\mu_B\right|
\lesssim
K\left(\frac{r\log(nd)+u}{n}\right)^{1/2}.
\end{equation}
On the same event, uniformly for every
$c\in[-\norm{f^*}_\infty,\norm{f^*}_\infty]$,
\begin{equation}\label{eq:translated-weighted-cell-means}
\left|
\hat p_B^{1/2}(\overline Y_B-c)
-p_B^{1/2}(\mu_B-c)
\right|
\lesssim
K\left(\frac{r\log(nd)+u}{n}\right)^{1/2}.
\end{equation}
\end{lemma}

\begin{proof}
Let $u_1:=u+\log 3$.
Apply Lemmas~\ref{lemma:unif_conc_bounded_cell_ave} and
\ref{lemma:uniform_cell_count} with confidence parameter $u_1$. Write
\[
\varepsilon_{n,r,u}:=\frac{r\log(nd)+u}{n},
\qquad
\overline f_B:=\frac{1}{N(B)}\sum_{i:\bX_i\in B}f^*(\bX_i)
\]
when $N(B)>0$, and set $\overline f_B=0$ otherwise. Since
$\norm{f^*}_\infty\lesssim K$, the two concentration lemmas imply
\begin{equation}\label{eq:weighted-signal-cell-means}
\left|
\hat p_B^{1/2}\overline f_B-p_B^{1/2}\mu_B
\right|
\lesssim K\varepsilon_{n,r,u}^{1/2}
\end{equation}
uniformly over $r$ and $B$. Indeed, when
$p_B\gtrsim\varepsilon_{n,r,u}$, the empirical and population masses
are comparable and the claim follows from
Lemma~\ref{lemma:unif_conc_bounded_cell_ave} and
\eqref{eq:cell-conc-1}. When
$p_B\lesssim\varepsilon_{n,r,u}$, equation~\eqref{eq:cell-conc-1}
also gives $\hat p_B\lesssim\varepsilon_{n,r,u}$, and the claim follows
directly from boundedness.

It remains to control the noise. Conditional on
$\bX_1,\ldots,\bX_n$, the variables $\xi_1,\ldots,\xi_n$ are
independent and centered, with conditional sub-Gaussian norms bounded
by $CK$. For every fixed nonempty sample trace generated by a cell
$B$, write
\[
\overline\xi_B:=\frac{1}{N(B)}\sum_{i:\bX_i\in B}\xi_i,
\]
and set $\overline\xi_B=0$ when $N(B)=0$. Then
\[
\hat p_B^{1/2}\overline\xi_B
=
\frac{1}{\sqrt{nN(B)}}
\sum_{i:\bX_i\in B}\xi_i
\]
is sub-Gaussian with norm at most $CK/\sqrt n$. For each $r$, the
sample-trace counting argument in the proof of
Lemma~\ref{lem:subG_multiplier_sup_cond_exp} gives at most
$\exp\{3r\log(nd)\}$ distinct traces. The conditional tail bound and a
union bound therefore show that, for a sufficiently large universal
constant $D$,
\[
\P\left\{
\sup_{B\in\mathcal A_{d,r}}
\hat p_B^{1/2}|\overline\xi_B|
>
DK\left(\frac{r\log(nd)+u_1}{n}\right)^{1/2}
\,\middle|\,
\bX_1,\ldots,\bX_n
\right\}
\leq 2e^{-u_1}(nd)^{-2r}.
\]
Because $n\ge2$, summing this bound over $r\in[d]$ gives a total
conditional failure probability of at most $e^{-u_1}$. Hence,
after absorbing the additive constant in $u_1$,
\begin{equation}\label{eq:weighted-noise-cell-means}
\hat p_B^{1/2}|\overline\xi_B|
\lesssim K\varepsilon_{n,r,u}^{1/2}
\end{equation}
simultaneously for all $r$ and $B$, with failure probability at most
$e^{-u_1}$. A union bound over this event and the two events invoked
above gives probability at least $1-e^{-u}$. Here and below, the
additive constant in $u_1$ is absorbed into the universal constants.
Combining
\eqref{eq:weighted-signal-cell-means} and
\eqref{eq:weighted-noise-cell-means} proves
\eqref{eq:uniform-weighted-cell-means}.

Finally, \eqref{eq:cell-conc-1},
$|c|\leq\norm{f^*}_\infty$, and
$\norm{f^*}_\infty\lesssim K$ give
\[
|c|\,|\hat p_B^{1/2}-p_B^{1/2}|
\lesssim K\varepsilon_{n,r,u}^{1/2}.
\]
Combining this with \eqref{eq:uniform-weighted-cell-means} proves
\eqref{eq:translated-weighted-cell-means}.
\end{proof}

\begin{proof}[Proof of Lemma~\ref{lemma:unif_conc_id}]
Apply Lemma~\ref{lem:uniform-weighted-cell-means} and work on its event.
Fix $r$, $A\in\mathcal A_{d,r}$, $j$, and $b$, and let $A_L,A_R$ be
the two children. The children belong to
$\mathcal A_{d,\min(d,r+1)}$; since $r\ge1$, this changes the
concentration scale by at most a universal factor. Write
\[
p_B:=\P(B),\qquad \hat p_B:=\frac{N(B)}n,
\qquad \mu_B:=\E[f^*(\bX)\mid\bX\in B]
\]
for $B\in\{A,A_L,A_R\}$, using the zero-mass conventions in
Lemma~\ref{lem:uniform-weighted-cell-means}, and set
\[
\rho:=CK\left(\frac{r\log(nd)+u}{n}\right)^{1/2}.
\]

Take $c=\mu_A$ in \eqref{eq:translated-weighted-cell-means} and define
\[
a_B:=p_B^{1/2}(\mu_B-\mu_A),
\qquad
\hat a_B:=\hat p_B^{1/2}(\overline Y_B-\mu_A).
\]
Since $|\mu_A|\leq\norm{f^*}_\infty\lesssim K$, the weighted cell-mean
bound gives
\begin{equation}\label{eq:weighted-coefficient-comparison}
|a_B-\hat a_B|\leq\rho,
\qquad B\in\{A,A_L,A_R\}.
\end{equation}

The population and empirical variance-decomposition identities are
\begin{equation}\label{eq:three-cell-impurity-identities}
\Delta(A,j,b)=a_{A_L}^2+a_{A_R}^2,
\qquad
\widehat\Delta(A,j,b)
=\hat a_{A_L}^2+\hat a_{A_R}^2-\hat a_A^2.
\end{equation}
These identities also hold when the parent or a child has zero empirical
or population mass under the stated conventions.

Let $R\coloneqq\bigl(\hat a_{A_L}^2+\hat a_{A_R}^2\bigr)^{1/2}.$
By the reverse triangle inequality in $\R^2$ and
\eqref{eq:weighted-coefficient-comparison},
\begin{equation}\label{eq:impurity-vector-comparison}
\left|R-\Delta(A,j,b)^{1/2}\right|
\leq
\left((\hat a_{A_L}-a_{A_L})^2
+(\hat a_{A_R}-a_{A_R})^2\right)^{1/2}
\leq \sqrt{2}\rho.
\end{equation}
Moreover, $a_A=0$, so
$|\hat a_A|\leq\rho$ by
\eqref{eq:weighted-coefficient-comparison}. Since
$R^2=\widehat\Delta(A,j,b)+\hat a_A^2$, we have
\[
0\leq R-\widehat\Delta(A,j,b)^{1/2}
\leq |\hat a_A|
\leq\rho.
\]
Combining this bound with \eqref{eq:impurity-vector-comparison} proves
\[
\left|
\widehat\Delta(A,j,b)^{1/2}-\Delta(A,j,b)^{1/2}
\right|
\lesssim
K\left(\frac{r\log(nd)+u}{n}\right)^{1/2}.
\]
The event and all of its bounds are uniform over $r,A,j$, and $b$, so
the proof is complete.
\end{proof}

\begin{proof}[Proof of Lemma~\ref{lem:mid-child-mass}]
Apply Lemma~\ref{lemma:unif_conc_id} and work on its event. By
\eqref{eq:mid-child-mass-threshold} and the square-root concentration
bound, choosing $C=C(c)$ sufficiently large makes the population
impurity decrease larger than any prescribed universal multiple of $(r\log(nd)+u)/n$.

Suppose without loss of generality that $\P(A_L)\leq\P(A_R)$. Since
$\norm{f^*}_\infty\lesssim K$,
\[
\Delta(A,j,b)
=\frac{\P(A_L)\P(A_R)}{\P(A)}
\bigl(\E[f^*\mid\bX\in A_L]-\E[f^*\mid\bX\in A_R]\bigr)^2
\lesssim \P(A_L)K^2.
\]
Choosing the prescribed multiple to absorb the universal constant in
this upper bound gives
$\P(A_L)\ge c(r\log(nd)+u)/n$. The same conclusion holds for $A_R$
because $\P(A_R)\ge\P(A_L)$.
\end{proof}

\section{Proofs for the one-dimensional upper bound}
\label{app:one-dimensional-upper-proofs}

\subsection{Technical lemmas for Theorem~\ref{thm:cart-mid-one-dimensional}}
\label{sec:one-dimensional-technical-lemmas}

\begin{proof}[Proof of Lemma~\ref{lemma:LH-upper-1d}]
Fix $x\in A$ and write
$\kappa\coloneqq p_{\min}/p_{\max}\in(0,1]$,
$\alpha_0\coloneqq\alpha(x)$,
\[
g(t)\coloneqq f^*(t)-\E[f^*(X)\mid X\in A],
\qquad U_0\coloneqq|g(x)|.
\]
The bound is trivial if $U_0=0$, so assume $U_0>0$. The local
H\"older condition gives, for every $t\in A$,
\begin{equation}
\label{eq:LH-holder}
|f^*(t)-f^*(x)|\le L|t-x|^{\alpha_0}.
\end{equation}

\emph{Step 1: Local lower bound on $|g|$.}
The reverse triangle inequality and \eqref{eq:LH-holder} give, for $t\in A$,
\begin{equation}
\label{eq:LH-rev-tri}
|g(t)|\ge U_0-|g(t)-g(x)|
\ge U_0-L|t-x|^{\alpha_0}.
\end{equation}
Setting $r\coloneqq(U_0/(2L))^{1/\alpha_0}$,
\eqref{eq:LH-rev-tri} yields $|g(t)|\ge U_0/2$ on
$A_1\coloneqq A\cap(x-r,x+r)$.

\emph{Step 2: Lower bound on $\Vol(A_1)$ and variance.}
Writing $A=[a,b]$ and assuming without loss of generality that
$x-a\le b-x$,
\begin{equation*}
\Vol(A_1)
=\min(x-a,r)+\min(b-x,r)
\ge\min\bigl(r,\tfrac{\Vol(A)}2\bigr).
\end{equation*}
Since $\P(A)\le p_{\max}\Vol(A)$ and $p_X\ge p_{\min}$ on $A_1$,
\begin{equation}
\label{eq:LH-var-lb}
\Var[f^*(X)\mid X\in A]
\ge\frac{1}{\P(A)}\int_{A_1}g(t)^2p_X(t)\,dt
\ge\kappa\frac{U_0^2}{8}
\min\bigl(\tfrac{2r}{\Vol(A)},1\bigr).
\end{equation}

\emph{Step 3: Case analysis.}
\emph{Case (i): $2r\le\Vol(A)$.} Equation~\eqref{eq:LH-var-lb}
and the definition of $r$ give
\[
\Vol(A)\Var[f^*(X)\mid X\in A]
\ge
\frac{\kappa U_0^{(1+2\alpha_0)/\alpha_0}}
{4(2L)^{1/\alpha_0}}.
\]
Hence
\begin{align*}
U_0
&\leq
\Bigl(\frac{4}{\kappa}\Bigr)^{\!\alpha_0/(1+2\alpha_0)}\!\!
\cdot 2\,L^{1/(1+2\alpha_0)}
\bigl(\Vol(A)\Var[f^*(X)\mid X\in A]\bigr)^{\!\alpha_0/(1+2\alpha_0)} \\
&\leq
\frac{4}{\sqrt\kappa}\,L^{1/(1+2\alpha_0)}
\bigl(\Vol(A)\Var[f^*(X)\mid X\in A]\bigr)^{\!\alpha_0/(1+2\alpha_0)},
\end{align*}
where the last step uses $\alpha_0/(1+2\alpha_0)\leq 1/2$ and $\kappa\leq 1$.

\emph{Case (ii): $2r>\Vol(A)$.} Equation~\eqref{eq:LH-var-lb} gives
$U_0^2\le 8\Var[f^*(X)\mid X\in A]/\kappa$. Also,
\eqref{eq:LH-holder} gives
\begin{equation*}
\Var[f^*(X)\mid X\in A]
\le\E[(f^*(X)-f^*(x))^2\mid X\in A]
\le L^2\Vol(A)^{2\alpha_0}.
\end{equation*}
Factoring the variance into powers $1/(1+2\alpha_0)$ and
$2\alpha_0/(1+2\alpha_0)$ and using the last display yields
\begin{equation*}
U_0^2
\le\frac{8}{\kappa}\,L^{2/(1+2\alpha_0)}
\bigl(\Vol(A)\Var[f^*(X)\mid X\in A]\bigr)^{2\alpha_0/(1+2\alpha_0)},
\end{equation*}
which gives the same claimed order.

Both cases are dominated by $4\kappa^{-1/2}$. Combining this with
$\Vol(A)\le p_{\min}^{-1}\P(A)$ gives
\[
U_0
\lesssim_{p_{\min},p_{\max}}
L^{1/(1+2\alpha_0)}V(A)^{\alpha_0/(1+2\alpha_0)}.
\]
This is exactly \eqref{eq:LH-upper-1d}.
\end{proof}

\begin{proof}[Proof of Lemma~\ref{lemma:population-good-split-balance-1d}]
Let $A'$ be the parent interval and $(A',1,b)$ the competitive split.
Write $A_{\mathrm{sm}}$ and $A_{\mathrm{lg}}$ for its smaller and larger
children by length, and let $A^x$ be the child containing $x$. Set
$\alpha_0=\alpha(x)$, $\Delta=\Delta(A',1,b)$, and
\[
\theta_0\coloneqq\frac{\alpha_0}{1+2\alpha_0},
\qquad
r_0\coloneqq\frac{1}{1+2\alpha_0}.
\]
Since $\Delta_{\max}(A')\le V(A')$, we may take $\lambda\le1$.
By sufficient impurity decrease and the competitiveness hypothesis,
\[
V(A')\le \frac{2}{\lambda}\Delta.
\]

First suppose $A^x=A_{\mathrm{sm}}$. Write
$p=\P(A_{\mathrm{sm}})/\P(A')$ and
$\rho=\Vol(A_{\mathrm{sm}})/\Vol(A')$.
If $p\ge 1/3$, then density regularity gives
$\Vol(A_{\mathrm{sm}})\gtrsim_{p_{\min},p_{\max}}\Vol(A')$. Since
\[
\Delta\le V(A')\lesssim_{p_{\max}}L^2\Vol(A')^{1+2\alpha_0},
\]
we immediately obtain
\[
\Vol(A_{\mathrm{sm}})
\gtrsim
\lambda^{2\theta_0}
\bigl(L^{-2}\Delta\bigr)^{r_0}.
\]
It remains to consider $p<1/3$. The impurity identity and the tower identity give
\[
\Delta
=
\P(A')p(1-p)
\bigl(\mu_{\mathrm{sm}}-\mu_{\mathrm{lg}}\bigr)^2
=
\P(A')\frac{p}{1-p}
\bigl(\mu'-\mu_{\mathrm{sm}}\bigr)^2,
\]
where 
\[
\begin{aligned}
\mu'&=\E[f^*(X)\mid X\in A'],\\
\mu_{\mathrm{sm}}&=\E[f^*(X)\mid X\in A_{\mathrm{sm}}],\qquad
\mu_{\mathrm{lg}}=\E[f^*(X)\mid X\in A_{\mathrm{lg}}].
\end{aligned}
\]
Using $p/(1-p)\lesssim \rho$ and $(a+b)^2\le2(a^2+b^2)$ gives
\[
\Delta
\lesssim
\P(A')\rho
\Bigl[
\bigl(\E[f^*(X)\mid X\in A']-f^*(x)\bigr)^2
+
\bigl(f^*(x)-\E[f^*(X)\mid X\in A_{\mathrm{sm}}]\bigr)^2
\Bigr].
\]
Lemma~\ref{lemma:LH-upper-1d}, applied to $A'$ and
$A_{\mathrm{sm}}$, together with the inequality $V(A_{\mathrm{sm}})\le V(A')$
implies
\[
\bigl(\E[f^*(X)\mid X\in A']-f^*(x)\bigr)^2
+
\bigl(f^*(x)-\E[f^*(X)\mid X\in A_{\mathrm{sm}}]\bigr)^2
\lesssim
L^{2r_0}V(A')^{2\theta_0}
\lesssim
L^{2r_0}
\left(\frac{\Delta}{\lambda}\right)^{2\theta_0}.
\]
Combining the preceding displays and using
$\P(A')\rho\lesssim_{p_{\max}}\Vol(A_{\mathrm{sm}})$ gives
\[
\Vol(A^x)
=\Vol(A_{\mathrm{sm}})
\gtrsim
\lambda^{2\theta_0}
\bigl(L^{-2}\Delta\bigr)^{r_0}.
\]

It remains to consider $A^x=A_{\mathrm{lg}}$. Then
$\Vol(A_{\mathrm{lg}})\ge\Vol(A')/2$, and
\[
\Delta\le V(A').
\]
The local H\"older condition at $x$ on $A'$ gives
\[
\Var[f^*(X)\mid X\in A']
\le
\E[(f^*(X)-f^*(x))^2\mid X\in A']
\le
L^2\Vol(A')^{2\alpha_0}.
\]
Thus $\Delta\lesssim_{p_{\max}}L^2\Vol(A')^{1+2\alpha_0}$, and hence
\[
\Vol(A^x)
=\Vol(A_{\mathrm{lg}})
\ge \Vol(A')/2
\gtrsim
\lambda^{2\theta_0}
\bigl(L^{-2}\Delta\bigr)^{r_0}.
\]
\end{proof}

\section{Proof of the CART-MID upper bound}
\label{app:cart-mid-upper-proof}

\subsection{Intermediate results for CART-MID}
\label{sec::erroranalysismid}

For each fixed $\bx\in[0,1]^d$, let $A(\bx)$ denote the terminal cell of the CART-MID partition that contains $\bx$. Define the oracle leaf average on the realized terminal cell by
\begin{align}\label{equ::fpdeltax}
f_{\delta}(\bx)
:=
\E[f^*(\bX)\mid \bX \in A(\bx)].
\end{align}
The empirical estimator $\hat f_\delta(\bx)$ replaces this population average by the sample average over the same terminal cell. Thus
\begin{align}\label{equ::errordecomposition}
\bigl|\hat{f}_{\delta}(\bx)-f^*(\bx)\bigr|
\le
\bigl|f_{\delta}(\bx)-f^*(\bx)\bigr|
+
\bigl|\hat{f}_{\delta}(\bx)-f_{\delta}(\bx)\bigr|.
\end{align}
The two terms correspond exactly to the second stage of CART: the first is the bias from averaging over $A(\bx)$, and the second is the sampling fluctuation of the leaf average.

The proof has three ingredients. First, above the impurity noise floor, empirical split maximizers behave like population split maximizers. Second, these competitive splits preserve the sparse structure and keep the terminal side lengths compatible with the signal present at the terminal parent. Third, because CART-MID stops only when the empirical split signal is below $\delta$, sufficient impurity decrease, bias visibility, and pathwise competitive split-scale regularity convert the stopping rule into the terminal bias and volume bounds needed for local averaging. This is where the two stopping requirements meet: the impurity threshold that prevents noise-driven splits also identifies the bandwidth scale needed for the leaf average.

For any $\bx \in [0,1]^d$, let
$\mathcal{A}(\bx)=\{A^i(\bx)\}_{i=0}^{D(\bx)}$
be the ancestor cells on the path to $A(\bx)$, ordered from the root
$A^0(\bx)=[0,1]^d$ to the terminal cell $A^{D(\bx)}(\bx)=A(\bx)$.
The following proposition collects the regularity consequences used throughout the upper-bound proof.

\begin{proposition}[High-probability regularity of the CART-MID path]\label{pro::regularity}
Assume the conditions of Theorem~\ref{thm:main-thm-point-id-cart}. For any $u>0$, choose $\delta$ above the impurity concentration scale as in \eqref{equ::deltathm}. Then, with probability at least $1-e^{-u}$, the following statements hold simultaneously for all $\bx\in[0,1]^d$:
\begin{enumerate}
\item \textbf{Sparsity preservation.} For every $0\le i\le D(\bx)$, the cell $A^i(\bx)$ belongs to $\mathcal{A}_{d,s}$, and for every irrelevant coordinate $j\notin S$, $A_j^i(\bx)=[0,1]$.

\item \textbf{Empirical-population comparison.} For every nonterminal ancestor $A^i(\bx)$,
\[
C_1\,\widehat{\Delta}_{\max}\!\bigl(A^i(\bx)\bigr)
\le
\Delta_{\max}\!\bigl(A^i(\bx)\bigr)
\le
C_2\,\widehat{\Delta}_{\max}\!\bigl(A^i(\bx)\bigr),
\]
for universal constants $0<C_1<C_2<\infty$.

\item \textbf{Competitive split.} If $(\hat j,\hat b)$ is the split selected at a nonterminal ancestor $A^i(\bx)$, then for a constant $C_{\mathrm{rel}}\in[1/2,1)$,
\[
\Delta\bigl(A^i(\bx),\hat j,\hat b\bigr)
\ge
C_{\mathrm{rel}}\,\Delta_{\max}\bigl(A^i(\bx)\bigr).
\]
In particular,
\[
\Delta_{\max}\bigl(A^i(\bx),\hat j\bigr)
\ge
C_{\mathrm{rel}}\,\Delta_{\max}\bigl(A^i(\bx)\bigr).
\]

\item \textbf{Path competitiveness.} Every nonterminal split on the path is competitive, so the pathwise competitive split-scale condition can be applied to the realized root-to-leaf path.
\end{enumerate}
\end{proposition}

The proof of Proposition~\ref{pro::regularity} is the main technical place where high dimensionality enters. The sparsity-separation condition separates relevant from irrelevant population gains, while the concentration inequalities ensure that this separation is preserved by the empirical maximizer whenever the node is split. In the proof below this argument is written as a stable-maximizer calculation; the displayed constants only serve to keep the empirical error below the separation gap.

Once these pathwise regularity properties hold, the terminal cell has the right local size. The stopping rule gives
$\widehat\Delta_{\max}(A(\bx))\le\delta$, and concentration turns this into
$\Delta_{\max}(A(\bx))\lesssim\delta$. Sufficient impurity decrease then bounds the residual variation inside the terminal cell. Conversely, the terminal-parent split has population signal at least a constant multiple of $\delta$, and all earlier splits on the path are competitive. The pathwise competitive split-scale assumption can then be applied to the realized path. This gives
\[
\len_j(A(\bx))
\gtrsim
L^{-1/\alpha_j(x_j)}
\left(\frac{\delta}{\Vol(A(\bx))}\right)^{1/(2\alpha_j(x_j))}
\]
for every coordinate split along the path. Multiplying over these coordinates gives
\[
\Vol(A(\bx))
\gtrsim
\omega^{2s/(s+2)}
L^{-2s/(2\bar\alpha_S(\bx)+s)}
\delta^{s/(2\bar\alpha_S(\bx)+s)}.
\]
This terminal-volume lower bound is the bandwidth statement behind the rate. In the theorem statement we use the weaker but simpler consequence obtained by bounding $\omega^{-s/(s+2)}$ by $\omega^{-1}$.
Together with bounded density, it verifies the mass condition in the bias-visibility assumption, so the residual-variation bound becomes the desired pointwise bias bound.

\begin{proposition}[Approximation error bound]\label{pro::approximation}
Assume the conditions of Theorem~\ref{thm:main-thm-point-id-cart}. Let $f_{\delta}(\bx)$ be defined as in \eqref{equ::fpdeltax}, with $\delta$ chosen as in \eqref{equ::deltathm}. 
Then, for any $u>0$, there exists a constant $C_5>0$, depending only on $p_{\min}$ and $\alpha_{\min}$, such that, with probability at least $1-e^{-u}$,
\[
\bigl|f_{\delta}(\bx)-f^*(\bx)\bigr|
\le
C_5\,
\zeta\,
L^{s/(2\bar\alpha_S(\bx)+s)}
\left(\frac{\delta}{\lambda}\right)^{
\bar\alpha_S(\bx)/(2\bar\alpha_S(\bx)+s)},
\quad \text{for all } \bx\in[0,1]^d.
\]
\end{proposition}

\begin{proposition}[Estimation error bound]\label{pro::estimation}
Assume the conditions of Theorem~\ref{thm:main-thm-point-id-cart}. Let $\hat f_\delta$ be the estimator defined in Section~\ref{section:early-CART}, and let $f_\delta$ be defined in \eqref{equ::fpdeltax}, with $\delta$ chosen as in \eqref{equ::deltathm}. 
Then, for any $u>0$, there exists a constant $C_6>0$, depending only on $p_{\min}$ and $\alpha_{\min}$, such that, with probability at least $1-e^{-u}$,
\[
\bigl|\hat{f}_{\delta}(\bx) - f_{\delta}(\bx)\bigr|
\le 
C_6 K
\omega^{-1}
L^{s/(2\bar\alpha_S(\bx)+s)}
\left(
\frac{s\log(nd)+u}
{n\,\delta^{s/(2\bar\alpha_S(\bx)+s)}}
\right)^{1/2},
\quad \text{for all } \bx \in [0,1]^d.
\]
\end{proposition}

\subsection{Proof of the CART-MID upper bound}

\begin{proof}[Proof of Theorem~\ref{thm:main-thm-point-id-cart}]
Let $u_0:=u+\log 3$.
Choose the fixed universal constant in \eqref{equ::deltathm} sufficiently large so that the lower-bound requirements on $\delta$ used below also hold with $u_0$ in place of $u$.
Let $\mathcal{E}_1$ denote the event in Lemma~\ref{lemma:unif_conc_id},
let $\mathcal{E}_2$ be the event in
Lemma~\ref{lemma:unif_conc_cell_ave},
and let $\mathcal E_3$ be
the event in Lemma~\ref{lem:mid-child-mass}, each applied with $u_0$.
The constant in \eqref{equ::deltathm} is also chosen large enough for
Lemma~\ref{lem:mid-child-mass}, with $c$ equal to the constant required
by Lemma~\ref{lemma:unif_conc_cell_ave}. By a union bound,
\[
\P\{\mathcal{E}_1\cap\mathcal{E}_2\cap\mathcal E_3\}
\ge 1-e^{-u}.
\]
We work on this intersection.

The proofs below establish the deterministic approximation bound in
Proposition~\ref{pro::approximation} on $\mathcal E_1$, and the
deterministic estimation bound in Proposition~\ref{pro::estimation} on
$\mathcal E_1\cap\mathcal E_2\cap\mathcal E_3$.
Hence both bounds hold simultaneously on
the present event.
Using the decomposition \eqref{equ::errordecomposition},
\begin{align*}
\bigl|\hat{f}_{\delta}(\bx)-f^*(\bx)\bigr|
&\le 
\bigl|f_{\delta}(\bx)-f^*(\bx)\bigr|
+
\bigl|\hat{f}_{\delta}(\bx)-f_{\delta}(\bx)\bigr| \\
&\le
C_5\,
\zeta\,
L^{s/(2\bar\alpha_S(\bx)+s)}
\left(\frac{\delta}{\lambda}\right)^{
\bar\alpha_S(\bx)/(2\bar\alpha_S(\bx)+s)}
\\
&\quad+
C_6 K
\omega^{-1}
L^{s/(2\bar\alpha_S(\bx)+s)}
\left(
\frac{s\log(nd)+u_0}{n}
\right)^{1/2}
\delta^{-s/(2(2\bar\alpha_S(\bx)+s))},
\end{align*}
where $C_5$ and $C_6$ are the constants specified in Propositions~\ref{pro::approximation} and~\ref{pro::estimation}.
Since $\Delta_{\max}(A)\le V(A)$, we may take $\lambda\le1$; also $\omega\le1$, $\zeta\ge1$, and $\vartheta\le1$. Thus the estimation term is bounded by the same rate as the approximation term after substituting the chosen $\delta$.
Absorbing the additive $\log 2$ into the universal constant yields \eqref{eq:main-thm-point-eid-cart}.
\end{proof}

\subsection{Proofs of the CART-MID intermediate results}

\begin{proof}[Proof of Proposition~\ref{pro::regularity}]
Let $\mathcal E_1$ be the event in Lemma~\ref{lemma:unif_conc_id}
with confidence parameter $u$. Then $\P\{\mathcal E_1\}\ge 1-e^{-u}$.
Throughout the proof, we work on this event. Set
\[
r_n:=
C K
\sqrt{\frac{s\log(nd)+u}{n}},
\qquad
q:=\max(1-\vartheta,1/\sqrt{2}),
\qquad
K_{\mathrm{sep}}:=1+\frac{2}{1-q}.
\]
Set $C_{\mathrm{rel}}:=q^2$, so that $C_{\mathrm{rel}}\in[1/2,1)$.
Since $K_{\mathrm{sep}}\lesssim\vartheta^{-1}$, the choice of the universal
constant in \eqref{equ::deltathm} gives
$\sqrt{\delta}\ge K_{\mathrm{sep}} r_n$. 

We first record a simple consequence of Lemma~\ref{lemma:unif_conc_id}.
For any coordinate set $S\subseteq[d]$, recall that $
\Delta_{\max}(A,S)=\max_{j\in S}\sup_{b\in[0,1]}\Delta(A,j,b)$ and $
\widehat\Delta_{\max}(A,S)=
\max_{j\in S}\sup_{b\in[0,1]}\widehat\Delta(A,j,b)$ as in Section~\ref{sec:impurity-decrease}.
Since the square-root concentration holds uniformly over all $j$ and
$b$, it also holds after taking the maximum:
\begin{align}\label{equ::root_conc_max}
\left|
\Delta_{\max}(A,S)^{1/2}
-
\widehat\Delta_{\max}(A,S)^{1/2}
\right|
\le r_n .
\end{align}
Indeed, this follows by applying the uniform bound at a maximizer of
one side and then reversing the roles of the two quantities.

\smallskip
\noindent
\textit{Proof of Claim $(i)$:}
Fix $\bx\in[0,1]^d$. We prove the claim by induction on $i$.
For $i=0$, $A^0(\bx)=[0,1]^d\in\mathcal A_{d,s}$, and
$A^0_j(\bx)=[0,1]$ for all $j\notin S$.

Assume the claim holds for some $i\le D(\bx)-1$, and write
$A=A^i(\bx)$. Suppose, for contradiction, that the split at $A$ is made
along an irrelevant coordinate $\hat j\notin S$. Then $
\widehat\Delta_{\max}(A,S^c)=\widehat\Delta_{\max}(A)\ge \delta$,
and by empirical optimality,
$\widehat\Delta_{\max}(A,S)\le \widehat\Delta_{\max}(A,S^c)$.

By the root concentration and the separation condition,
\[
\widehat\Delta_{\max}(A,S)^{1/2}
\ge
\Delta_{\max}(A,S)^{1/2}-r_n,
\]
whereas
\[
\widehat\Delta_{\max}(A,S^c)^{1/2}
\le
\Delta_{\max}(A,S^c)^{1/2}+r_n
\le
q\,\Delta_{\max}(A,S)^{1/2}+r_n.
\]
Combining these two displays with
$\widehat\Delta_{\max}(A,S)\le\widehat\Delta_{\max}(A,S^c)$ gives
\[
(1-q)\Delta_{\max}(A,S)^{1/2}\le 2r_n.
\]
Therefore
\[
\widehat\Delta_{\max}(A,S^c)^{1/2}
\le
\frac{2r_nq}{1-q}+r_n
=
\frac{1+q}{1-q}r_n
<
K_{\mathrm{sep}} r_n
\le
\sqrt{\delta},
\]
which contradicts $\widehat\Delta_{\max}(A,S^c)\ge\delta$.
Hence the selected coordinate must lie in $S$.

Consequently, $A^{i+1}(\bx)\in\mathcal A_{d,s}$. Moreover, by the
induction hypothesis, $A_j^i(\bx)=[0,1]$ for all $j\notin S$, and since
the split occurs along a coordinate in $S$, we also have
$A_j^{i+1}(\bx)=[0,1]$ for all $j\notin S$. This completes the induction.

\smallskip
\noindent
\textit{Proof of Claim $(ii)$:}
Fix a nonterminal ancestor $A=A^i(\bx)$, $0\le i\le D(\bx)-1$.
Since $A$ is split, $\widehat\Delta_{\max}(A)\ge\delta$. The root
concentration for the maximum gives
\[
\Delta_{\max}(A)^{1/2}
\ge
\widehat\Delta_{\max}(A)^{1/2}-r_n
\ge
\left(1-\frac1{K_{\mathrm{sep}}}\right)
\widehat\Delta_{\max}(A)^{1/2},
\]
where we used $\widehat\Delta_{\max}(A)^{1/2}\ge\sqrt\delta\ge K_{\mathrm{sep}} r_n$.
Similarly,
\[
\Delta_{\max}(A)^{1/2}
\le
\widehat\Delta_{\max}(A)^{1/2}+r_n
\le
\left(1+\frac1{K_{\mathrm{sep}}}\right)
\widehat\Delta_{\max}(A)^{1/2}.
\]
Squaring the two inequalities yields
\[
\left(1-\frac1{K_{\mathrm{sep}}}\right)^2\widehat\Delta_{\max}(A)
\le
\Delta_{\max}(A)
\le
\left(1+\frac1{K_{\mathrm{sep}}}\right)^2\widehat\Delta_{\max}(A).
\]
Thus Claim $(ii)$ holds with
\[
C_1:=\left(1-\frac1{K_{\mathrm{sep}}}\right)^2,
\qquad
C_2:=\left(1+\frac1{K_{\mathrm{sep}}}\right)^2.
\]

\smallskip
\noindent
\textit{Proof of Claim $(iii)$:}
Let $(\hat j,\hat b)$ be the empirical split selected at a nonterminal
ancestor $A=A^i(\bx)$.

By the root concentration at the selected split,
\[
\Delta(A,\hat j,\hat b)^{1/2}
\ge
\widehat\Delta(A,\hat j,\hat b)^{1/2}-r_n.
\]
On the other hand, applying the same concentration at a population
maximizer of $\Delta_{\max}(A)$ gives
\[
\Delta_{\max}(A)^{1/2}\le \widehat\Delta(A,\hat j,\hat b)^{1/2}+r_n,
\]
and hence $\widehat\Delta(A,\hat j,\hat b)^{1/2}\ge \Delta_{\max}(A)^{1/2}-r_n$. Therefore
\[
\Delta(A,\hat j,\hat b)^{1/2}
\ge
\Delta_{\max}(A)^{1/2}-2r_n.
\]
Since $\widehat\Delta(A,\hat j,\hat b)\ge\delta$ and $\Delta_{\max}(A)^{1/2}\ge \widehat\Delta(A,\hat j,\hat b)^{1/2}-r_n$, we have
\[
\Delta_{\max}(A)^{1/2}
\ge
\sqrt{\delta}-r_n
\ge
(K_{\mathrm{sep}}-1)r_n
=
\frac{2r_n}{1-q}.
\]
Thus $2r_n\le (1-q)\Delta_{\max}(A)^{1/2}$, and consequently $
\Delta(A,\hat j,\hat b)^{1/2}
\ge
q\Delta_{\max}(A)^{1/2}$.
Squaring gives the stronger statement $
\Delta(A,\hat j,\hat b)
\ge
C_{\mathrm{rel}}\Delta_{\max}(A)$.
Since $\Delta_{\max}(A,\hat j)\ge \Delta(A,\hat j,\hat b)$, this implies
\[
\Delta_{\max}(A,\hat j)
\ge
C_{\mathrm{rel}}\Delta_{\max}(A),
\]
which proves Claim $(iii)$.

\smallskip
\noindent
\textit{Proof of Claim $(iv)$:}
Fix a nonterminal ancestor $A=A^i(\bx)$, and let $(\hat j,\hat b)$ be the selected split. Claim $(iii)$ gives
\[
\Delta(A,\hat j,\hat b)\ge C_{\mathrm{rel}}\Delta_{\max}(A).
\]
Since $C_{\mathrm{rel}}\ge1/2$, every selected split on the path is competitive. This proves Claim $(iv)$ and completes the proof.
\end{proof}

\begin{lemma}\label{lem::convarlowerbound}
Assume the conditions of Theorem~\ref{thm:main-thm-point-id-cart}.
Let $\delta$ be chosen as in \eqref{equ::deltathm}.
Then for any $u>0$, there exists a constant $C_4'>0$, depending only on $\alpha_{\min}$, such that, with probability at least $1-e^{-u}$, the following holds
simultaneously for all $\bx\in[0,1]^d$:
\[
\Vol(A(\bx))
\;\ge\;
C_4'\,
\omega^{2s/(s+2)}
L^{-2s/(2\bar\alpha_S(\bx)+s)}\,
\delta^{s/(2\bar\alpha_S(\bx)+s)},
\]
where $\bar\alpha_S(\bx)\coloneqq \paren*{(1/s)\sum_{j\in S}(1/\alpha_j(x_j))}^{-1}$ for $\bx=(x_1,\ldots,x_d)$.
\end{lemma}

\begin{proof}[Proof of Lemma~\ref{lem::convarlowerbound}]
Let $\mathcal{E}_1$ denote the event on which 
\eqref{eq:id-conc-conclude-1-1} in 
Lemma~\ref{lemma:unif_conc_id} holds.
Then $\P\{\mathcal{E}_1\} \ge 1 - e^{-u}$, and we work on this event.
On this event, the conclusions of Proposition~\ref{pro::regularity} hold.

Define
\begin{equation}
\widetilde C_5
:=
\min\{1,C_{\mathrm{rel}}C_1\}^{1/(2\alpha_{\min})},
\end{equation}
where $C_1$ and $C_{\mathrm{rel}}$ are the constants from Proposition~\ref{pro::regularity}.
Since $C_{\mathrm{rel}}C_1$ is bounded below by a universal constant, $\widetilde C_5\le1$ and $\widetilde C_5\ge c_\alpha$ for a constant $c_\alpha>0$ depending only on $\alpha_{\min}$.
Then for every $j\in[d]$,
\begin{equation}\label{eq:C6choice}
\widetilde C_5
\le
\bigl(C_{\mathrm{rel}}C_1\bigr)^{1/(2\alpha_j(x_j))} .
\end{equation}

Fix $\bx\in[0,1]^d$, and let $\mathcal I=\mathcal I(\bx)$ denote the set of coordinates
that early-stopped CART has split along the path leading to the terminal cell $A(\bx)$.
By Proposition~\ref{pro::regularity}, on $\mathcal E_1$ the tree never splits
on coordinates outside $S$, hence
\begin{align}\label{equ::isubsets}
\mathcal I\subseteq S.
\end{align}

We first consider the nontrivial case $\mathcal I\neq\varnothing$.
Let $A^{\mathrm{Pa}}(\bx)=A^{D(\bx)-1}(\bx)$ be the terminal parent. Since $A^{\mathrm{Pa}}(\bx)$ is split,
\[
\widehat\Delta_{\max}(A^{\mathrm{Pa}}(\bx))\ge\delta.
\]
Let $(j_{\mathrm{Pa}},b_{\mathrm{Pa}})$ be the split selected at $A^{\mathrm{Pa}}(\bx)$ and write
\[
\Delta_{\mathrm{pa}}(\bx)
:=
\Delta(A^{\mathrm{Pa}}(\bx),j_{\mathrm{Pa}},b_{\mathrm{Pa}}).
\]
Proposition~\ref{pro::regularity} gives
\[
\Delta_{\mathrm{pa}}(\bx)
\ge
C_{\mathrm{rel}}\Delta_{\max}(A^{\mathrm{Pa}}(\bx))
\ge
C_{\mathrm{rel}}C_1\widehat\Delta_{\max}(A^{\mathrm{Pa}}(\bx))
\ge C_{\mathrm{rel}}C_1\delta.
\]
Moreover, every split along the path to $A(\bx)$ is competitive by Proposition~\ref{pro::regularity}. Applying Assumption~\ref{assum:reliable-split-balance} to the realized path gives, for every $j\in\mathcal I$,
\begin{align}
\label{eq:AjLower}
\len_j(A(\bx))
&\ge
\omega
L^{-1/\alpha_j(x_j)}
\left(
\frac{\Delta_{\mathrm{pa}}(\bx)}{\Vol(A(\bx))}
\right)^{1/(2\alpha_j(x_j))} \notag\\
	&\ge
	\omega\widetilde C_5L^{-1/\alpha_j(x_j)}
	\Bigl(\frac{\delta}{\Vol(A(\bx))}\Bigr)^{1/(2\alpha_j(x_j))} ,
	\end{align}
	where the last step uses \eqref{eq:C6choice}.

Now define
\[
\alpha_{\mathcal I}(\bx)
:=
\Biggl(
\frac{1}{|\mathcal I|}
\sum_{j\in\mathcal I}\frac{1}{\alpha_j(x_j)}
\Biggr)^{-1}.
\]
Taking the product of \eqref{eq:AjLower} over all $j\in\mathcal I$, we obtain
\begin{align*}
\prod_{j\in\mathcal I}\len_j(A(\bx))
&\ge
(\omega\widetilde C_5)^{|\mathcal I|}
L^{-\sum_{j\in\mathcal I}1/\alpha_j(x_j)}
\biggl(\frac{\delta}{\Vol(A(\bx))}\biggr)^{
(1/2)\sum_{j\in\mathcal I}1/\alpha_j(x_j)
} \\
&=
(\omega\widetilde C_5)^{|\mathcal I|}
L^{-|\mathcal I|/\alpha_{\mathcal I}(\bx)}
\biggl(\frac{\delta}{\Vol(A(\bx))}\biggr)^{
|\mathcal I|/(2\alpha_{\mathcal I}(\bx))
}.
\end{align*}
Since $\len_j(A(\bx))=1$ for $j\notin\mathcal I$, we have
\[
\Vol(A(\bx))
=
\prod_{j=1}^d \len_j(A(\bx))
=
\prod_{j\in\mathcal I}\len_j(A(\bx)).
\]
Therefore,
\[
	\Vol(A(\bx))
	\ge
	(\omega\widetilde C_5)^{|\mathcal I|}
	L^{-|\mathcal I|/\alpha_{\mathcal I}(\bx)}
	\Bigl(\frac{\delta}{\Vol(A(\bx))}\Bigr)^{
	|\mathcal I|/(2\alpha_{\mathcal I}(\bx))
},
\]
that is,
\[
	\Vol(A(\bx))^{1+|\mathcal I|/(2\alpha_{\mathcal I}(\bx))}
	\ge
	(\omega\widetilde C_5)^{|\mathcal I|}
	L^{-|\mathcal I|/\alpha_{\mathcal I}(\bx)}
	\delta^{|\mathcal I|/(2\alpha_{\mathcal I}(\bx))}.
\]
Hence
\begin{equation}\label{eq:AlowerI}
	\Vol(A(\bx))
	\ge
	\bigl((\omega\widetilde C_5)^{|\mathcal I|}\bigr)^{
	2\alpha_{\mathcal I}(\bx)/(2\alpha_{\mathcal I}(\bx)+|\mathcal I|)
	}
L^{-2|\mathcal I|/(2\alpha_{\mathcal I}(\bx)+|\mathcal I|)}
\delta^{|\mathcal I|/(2\alpha_{\mathcal I}(\bx)+|\mathcal I|)}.
\end{equation}

Next we compare the exponent with the target exponent.
Since $\mathcal I\subseteq S$ by \eqref{equ::isubsets},
\[
\sum_{j\in\mathcal I}\frac{1}{\alpha_j(x_j)}
\le
\sum_{j\in S}\frac{1}{\alpha_j(x_j)},
\]
and therefore
\[
\frac{\alpha_{\mathcal I}(\bx)}{|\mathcal I|}
=
\Biggl(\sum_{j\in\mathcal I}\frac{1}{\alpha_j(x_j)}\Biggr)^{-1}
\ge
\Biggl(\sum_{j\in S}\frac{1}{\alpha_j(x_j)}\Biggr)^{-1}
=
\frac{\bar\alpha_S(\bx)}{s}.
\]
It follows that
\[
\frac{|\mathcal I|}{2\alpha_{\mathcal I}(\bx)+|\mathcal I|}
=
\frac{1}{1+2\alpha_{\mathcal I}(\bx)/|\mathcal I|}
\le
\frac{1}{1+2\bar\alpha_S(\bx)/s}
=
\frac{s}{2\bar\alpha_S(\bx)+s}.
\]
By the theorem condition, $\delta\le1$. Hence
\begin{equation}\label{eq:deltacompare}
\delta^{|\mathcal I|/(2\alpha_{\mathcal I}(\bx)+|\mathcal I|)}
\ge
\delta^{s/(2\bar\alpha_S(\bx)+s)}.
\end{equation}
Because $L\ge1$, the preceding comparison of exponents also implies
\begin{equation}\label{eq:Lcompare}
L^{-2|\mathcal I|/(2\alpha_{\mathcal I}(\bx)+|\mathcal I|)}
\gtrsim
L^{-2s/(2\bar\alpha_S(\bx)+s)}.
\end{equation}

The $\omega$ factor in \eqref{eq:AlowerI} is
\[
\omega^{
2\alpha_{\mathcal I}(\bx)|\mathcal I|/
(2\alpha_{\mathcal I}(\bx)+|\mathcal I|)
}.
\]
Since $\alpha_{\mathcal I}(\bx)\le1$ and $|\mathcal I|\le s$, the exponent is at most
\[
2s/(s+2).
\]
Because $\omega\le1$, the $\omega$ factor is bounded below by
\[
\omega^{2s/(s+2)}.
\]
For the remaining constant factor,
\[
\bigl(\widetilde C_5^{|\mathcal I|}\bigr)^{
2\alpha_{\mathcal I}(\bx)/(2\alpha_{\mathcal I}(\bx)+|\mathcal I|)
}
=
\widetilde C_5^{
2\alpha_{\mathcal I}(\bx)|\mathcal I|/
(2\alpha_{\mathcal I}(\bx)+|\mathcal I|)
}
\ge c_\alpha^2,
\]
because the exponent is at most $2$ and $\widetilde C_5\le1$.
Set $C_4':=c_\alpha^2$.
Combining this with \eqref{eq:AlowerI}, \eqref{eq:deltacompare}, and \eqref{eq:Lcompare}, we conclude that
for all $\bx$ with $\mathcal I(\bx)\neq\varnothing$,
\[
	\Vol(A(\bx))
	\ge
	C_4'\,
	\omega^{2s/(s+2)}
	L^{-2s/(2\bar\alpha_S(\bx)+s)}\,
	\delta^{s/(2\bar\alpha_S(\bx)+s)}.
\]

Finally, if $\mathcal I(\bx)=\varnothing$, then no split is made along the path to $\bx$,
so $A(\bx)=[0,1]^d$ and thus $\Vol(A(\bx))=1$.
Since $\delta\le 1$, we trivially have
\[
	\Vol(A(\bx))=1\ge C_4'\,\omega^{2s/(s+2)}
	\delta^{s/(2\bar\alpha_S(\bx)+s)}
	L^{-2s/(2\bar\alpha_S(\bx)+s)}
\]
after possibly replacing $C_4'$ by $\min(C_4',1)$.
This completes the proof.
\end{proof}

\begin{proof}[Proof of Proposition~\ref{pro::approximation}]
	Let $\mathcal{E}_1$ denote the event on which 
    \eqref{eq:id-conc-conclude-1-1} in 
    Lemma~\ref{lemma:unif_conc_id} holds.
    Then $\P\{\mathcal{E}_1\} \ge 1 - e^{-u}$, and we work on this event.
    On this event, the conclusions of Proposition~\ref{pro::regularity} and Lemma~\ref{lem::convarlowerbound} hold.
    Set
    \[
    r_n:=
    C K
    \sqrt{\frac{s\log(nd)+u}{n}},
    \qquad
    q:=\max(1-\vartheta,1/\sqrt{2}),
    \qquad
    K_{\mathrm{sep}}:=1+\frac{2}{1-q}.
    \]
    Since $K_{\mathrm{sep}}\lesssim\vartheta^{-1}$, the choice of the universal
    constant in \eqref{equ::deltathm} gives
    $\sqrt{\delta}\ge K_{\mathrm{sep}}r_n$.
	
	Fix $\bx\in[0,1]^d$ and let $A(\bx)$ be the terminal cell containing $\bx$.
    Since $A(\bx)$ is a terminal node, by construction of the early-stopping rule,
    $\widehat{\Delta}_{\max}(A(\bx))\le \delta$.
    On the event $\mathcal E_1$, applying \eqref{equ::root_conc_max} gives
    \begin{align*}
    \Delta_{\max}(A(\bx))^{1/2}
    \le
    \widehat{\Delta}_{\max}(A(\bx))^{1/2}+r_n
    \le
    \sqrt{\delta}+r_n
    \le
    \left(1+\frac1{K_{\mathrm{sep}}}\right)\sqrt{\delta}.
    \end{align*}
    Therefore,
    \[
    \Delta_{\max}(A(\bx))
    \le
    \left(1+\frac1{K_{\mathrm{sep}}}\right)^2\delta .
    \]
    By sufficient impurity decrease,
    \begin{equation}\label{eq:upper-bound-VA}
        V(A(\bx))=\P(A(\bx))\Var[f^*(\bX)\mid \bX\in A(\bx)]
    \le
    \lambda^{-1}\Delta_{\max}(A(\bx))
    \lesssim \lambda^{-1}\delta .
    \end{equation}
    By Lemma~\ref{lem::convarlowerbound} and bounded density (Assumption~\ref{assum:bounded-density}),
    \[
    \P(A(\bx))
    \gtrsim_{p_{\min},\alpha_{\min}}
    \omega^{2s/(s+2)}
    L^{-2s/(2\bar\alpha_S(\bx)+s)}
    \delta^{
    s/(2\bar\alpha_S(\bx)+s)
    }.
    \]
    Then the compatibility condition on $\chi$ in Theorem~\ref{thm:main-thm-point-id-cart} implies
    \[
    \P(A(\bx))
    \ge
    \chi
    \paren*{L^{-2}V(A(\bx))}^{
    s/(2\bar\alpha_S(\bx)+s)
    },
    \]
    after adjusting the universal constant in the theorem.
    Thus the terminal cell satisfies the mass precondition in Assumption~\ref{assum:variance-oscillation}.
    Applying \eqref{eq:bias-visibility} with the upper bound of $V(A(\bx))$ in \eqref{eq:upper-bound-VA} gives
    \begin{align*}
  	|f_{\delta}(\bx)-f^*(\bx)|
	\leq C_5
    \zeta\,
    L^{s/(2\bar\alpha_S(\bx)+s)}
    \left(\frac{\delta}{\lambda}\right)^{
    \bar\alpha_S(\bx)/(2\bar\alpha_S(\bx)+s)}.
    \end{align*}
	Since $\bx$ was arbitrary, the bound holds uniformly over $[0,1]^d$.
\end{proof}

\begin{proof}[Proof of Proposition~\ref{pro::estimation}]
Let $\mathcal{E}_1$ denote the event on which 
\eqref{eq:id-conc-conclude-1-1} in 
Lemma~\ref{lemma:unif_conc_id} holds.
Let $\mathcal{E}_2$ be the high-probability event in
Lemma~\ref{lemma:unif_conc_cell_ave}, and let $\mathcal E_3$ be
the event in Lemma~\ref{lem:mid-child-mass}. Let $u_0:=u+\log 3$.
By a union bound, applying the three lemmas with parameter $u_0$ gives
\[
\P\{\mathcal{E}_1\cap\mathcal{E}_2\cap\mathcal E_3\}
\ge 1-e^{-u}.
\]
We work on this intersection, choosing the universal constant in
\eqref{equ::deltathm} large enough for Lemma~\ref{lem:mid-child-mass}
with $c$ equal to the constant in
Lemma~\ref{lemma:unif_conc_cell_ave}.

On $\mathcal E_1$, Proposition~\ref{pro::regularity} ensures that every
cell and split on a root-to-leaf path uses only coordinates in $S$.
Thus every such cell belongs to $\mathcal A_{d,s}$. If $A(\bx)$ is not
the root, it was created by a split with empirical impurity decrease at
least $\delta$. Since $\vartheta\le1$, the choice
\eqref{equ::deltathm} and $\mathcal E_3$ give
\[
\P(A(\bx))\ge c\frac{s\log(nd)+u_0}{n},
\]
so Lemma~\ref{lemma:unif_conc_cell_ave} applies. If $A(\bx)$ is the
root, the root-cell case of Lemma~\ref{lemma:unif_conc_cell_ave} gives
the same bound. In either case, bounded density yields
\[
\bigl|f_\delta(\bx) - \hat{f}_{\delta}(\bx)\bigr|
\le 
C K\left(\frac{s\log(nd)+u_0}{n\,\Vol(A(\bx))}\right)^{1/2}.
\]
On $\mathcal{E}_1$, Lemma~\ref{lem::convarlowerbound} provides a lower bound on $\Vol(A(\bx))$. 
Substituting this bound into the inequality above gives
\[
	\bigl|f_\delta(\bx) - \hat{f}_{\delta}(\bx)\bigr|
	\le 
		\frac{2C K}{C_4'^{1/2}}
	\omega^{-s/(s+2)}
	L^{s/(2\bar\alpha_S(\bx)+s)}
	\left(\frac{s\log(nd)+u}{n\,\delta^{s/(2\bar\alpha_S(\bx)+s)}}\right)^{1/2},
\]
after enlarging the constant to absorb the additive $\log 3$.
Since $0<\omega\le1$, this is bounded above by the same display with $\omega^{-1}$ in place of $\omega^{-s/(s+2)}$.
The result follows by setting $C_6 := 2C/C_4'^{1/2}$.
\end{proof}

\section{Proofs for the sufficient conditions}
\label{app:sufficient-condition-proofs}

Throughout this appendix, implicit constants in $\lesssim$, $\gtrsim$, and $\asymp$ may depend on the parameters explicitly allowed in the corresponding statement.

\subsection{Auxiliary density-comparison lemmas}

\begin{lemma}\label{lemma:equiv-proj-var}
    Assume there exist constants $0 < c_1 < c_2$ such that $c_1 \leq p_{\bX}(\bx) \leq c_2$ for all $\bx \in [0,1]^s$ for the distribution of $\bX$ on $[0,1]^s$.
    Then, for any cell $A = \prod_{j=1}^s I_j \subseteq [0,1]^s$, any measurable univariate function $g : [0,1] \to \mathbb{R}$, and any $k \in [s]$,
    \begin{equation}
    \begin{aligned}
        \Var\bracks*{g(X_k) \mid X_k \in I_k} \asymp \Var\bracks*{g(X_k) \mid \bX \in A}.
    \end{aligned}
    \end{equation}
\end{lemma}

\begin{proof}
    Fix a cell $A = \prod_{j=1}^s I_j \subseteq [0,1]^s$ and an index $k \in [s]$.

    Let $P_A$ and $P_k$ denote the conditional distributions of $X_k$ given $\bX \in A$ and $X_k \in I_k$, respectively. Let $q_A$ and $q_k$ be their corresponding densities on $I_k$ with respect to the Lebesgue measure. Noting that the marginal density $p_{X_k}(x_k)$ is bounded within $[c_1, c_2]$ (since the domain volume of $[0,1]^{s-1}$ is 1), applying the uniform bounds of the joint and marginal densities to the conditional formulations yields that for any $x_k \in I_k$,
    \begin{equation}
    \begin{aligned}
        \frac{c_1}{c_2 |I_k|} \leq q_A(x_k) \leq \frac{c_2}{c_1 |I_k|}, \quad \text{and} \quad \frac{c_1}{c_2 |I_k|} \leq q_k(x_k) \leq \frac{c_2}{c_1 |I_k|}.
    \end{aligned}
    \end{equation}
    Therefore, the density ratio is uniformly bounded on $I_k$:
    \begin{equation}\label{eq:RN}
    \begin{aligned}
        (c_1/c_2)^2 \leq \frac{q_k(x_k)}{q_A(x_k)} \leq (c_2/c_1)^2.
    \end{aligned}
    \end{equation}

    By the variational characterization of variance,
    $\Var_P[g]=\inf_{a\in\mathbb R}\E_P[(g-a)^2]$ for every measurable $g$.
    Applying the uniform bound \eqref{eq:RN} to $(g(x_k)-a)^2$ and taking the infimum over $a\in\mathbb R$ gives
    \begin{equation}
    \begin{aligned}
        (c_1/c_2)^2 \Var_{P_A}(g) \leq \Var_{P_k}(g) \leq (c_2/c_1)^2 \Var_{P_A}(g).
    \end{aligned}
    \end{equation}
\end{proof}

\begin{lemma}\label{lemma:equiv-unif-P}
    Let $P$ be the distribution of $\bX$ on $[0,1]^s$ with density $p_{\bX}$, and assume there exist constants $0 < c_1 < c_2$ such that $c_1 \leq p_{\bX}(\bx) \leq c_2$ for all $\bx \in [0,1]^s$. Let $\nu$ denote the uniform distribution on $[0,1]^s$. For any cell $A \subseteq [0,1]^s$, let $\E_\nu\bracks*{\cdot}$ and $\Var_\nu\bracks*{\cdot}$ denote the moments under the conditional measure $\nu$ restricted to $A$, and analogously define $\E_P\bracks*{\cdot}$ and $\Var_P\bracks*{\cdot}$ under $P$ restricted to $A$. Then, for any measurable function $\phi$ on $A$, we have
    \begin{equation*}
         \Var_P\bracks*{\phi} \asymp \Var_\nu\bracks*{\phi}.
    \end{equation*}
    Furthermore, if $\phi(\bx) \geq 0$ for all $\bx \in [0,1]^s$, then
    \begin{equation*}
        \E_P\bracks*{\phi} \asymp \E_\nu\bracks*{\phi}.
    \end{equation*}
\end{lemma}

\begin{proof}
    Let $\kappa= c_2/c_1$. Let $P_A$ denote $P$ conditioned
    on~$A$ and $\nu_A$ denote the uniform (Lebesgue-normalized) product measure
    on~$A$, so that under~$\nu_A$ the coordinates
    $X_1, \ldots, X_s$ are \emph{mutually independent} with
    $X_k \sim \mathrm{Uniform}(I_k)$.
    Since $c_1 \leq p_{\bX}(\bx) \leq c_2$, the conditional density on~$A$
    satisfies
    \begin{equation*}
      \kappa^{-1}
      \leq \frac{dP_A}{d\nu_A}(\bx)
      = \frac{p_{\bX}(\bx)\,\Vol(A)}{\P\{\bX \in A\}}
      \leq \kappa,
      \qquad \bx \in A.
    \end{equation*}

    The variance admits the product-pair representation
    \begin{equation*}
      \Var_P[\phi] = \frac{1}{2}\,
      \E_P^{\otimes 2}\!\left[
        (\phi(\bX) - \phi(\bX'))^2
      \right],
    \end{equation*}
    where $(\bX, \bX')$ are i.i.d.\ under~$P_A$. Since
    $dP_A^{\otimes 2}/d\nu_A^{\otimes 2}$ is bounded between
    $\kappa^{-2}$ and $\kappa^2$, for any non-negative integrand~$\Phi$,
    \begin{equation}\label{eq:sandwich}
      \kappa^{-2}\,\E_\nu^{\otimes 2}[\Phi]
      \leq \E_P^{\otimes 2}[\Phi]
      \leq \kappa^2\,\E_\nu^{\otimes 2}[\Phi].
    \end{equation}

    Applying~\eqref{eq:sandwich} to
    $\Phi=(\phi(\bX)-\phi(\bX'))^2$ gives
    \begin{equation*}
      \kappa^{-2}\, \Var_\nu[\phi]
      \leq \Var_P[\phi]
      \leq \kappa^2\, \Var_\nu[\phi].
    \end{equation*}
    If $\phi\ge0$, applying the one-sample density-ratio bound directly to
    $\phi$ gives the asserted expectation comparison.
\end{proof}

\subsection{Closure under independent blocks}

We first prove Proposition~\ref{prop:additive}. For a cell $A=\prod_{j=1}^sI_j$, write $\E_P[\,\cdot\,] = \E[\,\cdot \mid \bX \in A]$,
$\Var_P[\cdot] = \Var[\cdot \mid \bX \in A]$, and
$\Cov_P[\cdot, \cdot] = \Cov[\cdot, \cdot \mid \bX \in A]$.
Let $P_A$ denote the design law conditioned on~$A$ and $\nu_A$ denote the uniform (Lebesgue-normalized) product measure on~$A$, so that under~$\nu_A$ the coordinates $X_1, \ldots, X_s$ are \emph{mutually independent} with $X_k \sim \mathrm{Uniform}(I_k)$. We write $\E_\nu$ and $\Var_\nu$ for moments under~$\nu_A$.

\begin{proof}[Proof of Proposition~\ref{prop:additive}]
    \textit{Step 1: Proof of sufficient impurity decrease.}
    Let $A = A_{\mathcal{J}_1} \times \cdots \times A_{\mathcal{J}_M}$
    be any cell consistent with the partition $\mathcal{J}_1, \dots, \mathcal{J}_M$. Fix $k \in [M]$. Since the pair $(g_k, P_{\mathcal{J}_k})$ satisfies sufficient impurity decrease, there exists a split of $A_{\mathcal{J}_k}$ into two child cells $A_{\mathcal{J}_k,L}$ and $A_{\mathcal{J}_k,R}$ such that
    \begin{equation*}
        \begin{aligned}
            &\frac{P_{\mathcal{J}_k}(A_{\mathcal{J}_k,L})P_{\mathcal{J}_k}(A_{\mathcal{J}_k,R})}{P_{\mathcal{J}_k}(A_{\mathcal{J}_k})}
            \abs*{ \E\bracks*{g_k(\bX_{\mathcal{J}_k}) \mid \bX_{\mathcal{J}_k} \in A_{\mathcal{J}_k,L}}
            - \E\bracks*{g_k(\bX_{\mathcal{J}_k}) \mid \bX_{\mathcal{J}_k} \in A_{\mathcal{J}_k,R}} }^2 \\
            &\quad \gtrsim P_{\mathcal{J}_k}(A_{\mathcal{J}_k}) \Var\bracks*{g_k(\bX_{\mathcal{J}_k}) \mid \bX_{\mathcal{J}_k} \in A_{\mathcal{J}_k}}.
        \end{aligned}
    \end{equation*}

    Multiplying both sides by $\prod_{j \neq k} P_{\mathcal{J}_j}(A_{\mathcal{J}_j})$ and using the independence of the coordinate blocks, we obtain
    \begin{equation*}
        \begin{split}
            &\frac{\P(A_{k,L})\P(A_{k,R})}{\P(A)}
            \abs*{ \E\bracks*{g_k(\bX_{\mathcal{J}_k}) \mid \bX_{\mathcal{J}_k} \in A_{\mathcal{J}_k,L}}
            - \E\bracks*{g_k(\bX_{\mathcal{J}_k}) \mid \bX_{\mathcal{J}_k} \in A_{\mathcal{J}_k,R}} }^2\\
            &\gtrsim \P(A) \Var\bracks*{g_k(\bX_{\mathcal{J}_k}) \mid \bX_{\mathcal{J}_k} \in A_{\mathcal{J}_k}},
        \end{split}
    \end{equation*}
    where $A_{k,L} = A_{\mathcal{J}_1} \times \cdots \times A_{\mathcal{J}_{k-1}} \times A_{\mathcal{J}_k,L} \times A_{\mathcal{J}_{k+1}} \times \cdots \times A_{\mathcal{J}_M}$
    and $A_{k,R} = A_{\mathcal{J}_1} \times \cdots \times A_{\mathcal{J}_{k-1}} \times A_{\mathcal{J}_k,R} \times A_{\mathcal{J}_{k+1}} \times \cdots \times A_{\mathcal{J}_M}$.

    Moreover, by the additive structure of $f$, we have
    \begin{equation*}
        \E\bracks*{f(\bX) \mid \bX \in A_{k,L}}
        = \E\bracks*{g_k(\bX_{\mathcal{J}_k}) \mid \bX_{\mathcal{J}_k} \in A_{\mathcal{J}_k,L}}
        + \sum_{j \neq k} \E\bracks*{g_j(\bX_{\mathcal{J}_j}) \mid \bX_{\mathcal{J}_j} \in A_{\mathcal{J}_j}},
    \end{equation*}
    and
    \begin{equation*}
        \E\bracks*{f(\bX) \mid \bX \in A_{k,R}}
        = \E\bracks*{g_k(\bX_{\mathcal{J}_k}) \mid \bX_{\mathcal{J}_k} \in A_{\mathcal{J}_k,R}}
        + \sum_{j \neq k} \E\bracks*{g_j(\bX_{\mathcal{J}_j}) \mid \bX_{\mathcal{J}_j} \in A_{\mathcal{J}_j}}.
    \end{equation*}
    Taking the difference of these two equations yields
    \begin{align*}
        &\E\bracks*{f(\bX) \mid \bX \in A_{k,L}}
        - \E\bracks*{f(\bX) \mid \bX \in A_{k,R}} \\
        &\qquad =
        \E\bracks*{g_k(\bX_{\mathcal{J}_k}) \mid \bX_{\mathcal{J}_k} \in A_{\mathcal{J}_k,L}}
        - \E\bracks*{g_k(\bX_{\mathcal{J}_k}) \mid \bX_{\mathcal{J}_k} \in A_{\mathcal{J}_k,R}}.
    \end{align*}
    Substituting this identity into the previous bound gives
    \begin{equation*}
        \frac{\P(A_{k,L})\P(A_{k,R})}{\P(A)}
        \abs*{ \E\bracks*{f(\bX) \mid \bX \in A_{k,L}} - \E\bracks*{f(\bX) \mid \bX \in A_{k,R}} }^2
        \gtrsim \P(A) \Var\bracks*{g_k(\bX_{\mathcal{J}_k}) \mid \bX_{\mathcal{J}_k} \in A_{\mathcal{J}_k}}.
    \end{equation*}

    Finally, again since the subvectors $\bX_{\mathcal{J}_1}, \dots, \bX_{\mathcal{J}_M}$ are mutually independent, which implies that
    \begin{equation*}
        \Var\bracks*{f(\bX) \mid \bX \in A} = \sum_{k=1}^M \Var\bracks*{g_k(\bX_{\mathcal{J}_k}) \mid \bX_{\mathcal{J}_k} \in A_{\mathcal{J}_k}}.
    \end{equation*}
    Consequently,
    \begin{equation*}
        \begin{aligned}
            \sup_{k \in [M]} \frac{\P(A_{k,L})\P(A_{k,R})}{\P(A)} &\abs*{ \E\bracks*{f(\bX) \mid \bX \in A_{k,L}} - \E\bracks*{f(\bX) \mid \bX \in A_{k,R}} }^2 \\
            &\gtrsim \sup_{k \in [M]} \P(A) \Var\bracks*{g_k(\bX_{\mathcal{J}_k}) \mid \bX_{\mathcal{J}_k} \in A_{\mathcal{J}_k}} \\
            &\geq \frac{\P(A)}{M} \sum_{k=1}^M \Var\bracks*{g_k(\bX_{\mathcal{J}_k}) \mid \bX_{\mathcal{J}_k} \in A_{\mathcal{J}_k}} \\
            &= \frac{\P(A)}{M} \Var\bracks*{f(\bX) \mid \bX \in A}.
        \end{aligned}
    \end{equation*}
    This establishes sufficient impurity decrease for $(f, P)$.

    \textit{Step 2: Proof of local norm equivalence.} Let $\mu_0 = \sum_{k=1}^M \E\bracks*{g_k(\bX_{\mathcal{J}_k}) \mid \bX_{\mathcal{J}_k} \in A_{\mathcal{J}_k}}$. By the blockwise local norm-equivalence assumption,
    \begin{equation}
    \begin{aligned}
        \sup_{\bx \in A} \abs*{f(\bx) - \mu_0}
        &\leq \sum_{k=1}^M \sup_{\bx_{\mathcal{J}_k} \in A_{\mathcal{J}_k}}\abs*{g_k(\bx_{\mathcal{J}_k}) - \E\bracks*{g_k(\bX_{\mathcal{J}_k}) \mid \bX_{\mathcal{J}_k} \in A_{\mathcal{J}_k}}} \\
        &\leq \sum_{k=1}^M \operatorname{Osc}(g_k;A_{\mathcal{J}_k})\\
        &\leq C_1 \sum_{k=1}^M \Var\bracks*{g_k(\bX_{\mathcal{J}_k}) \mid \bX_{\mathcal{J}_k} \in A_{\mathcal{J}_k}}^{1/2}.
    \end{aligned}
    \end{equation}
    Then by the Cauchy--Schwarz inequality and the independence between different blocks, we have
    \begin{equation}
        \begin{split}
            \sup_{\bx \in A} \abs*{f(\bx) - \mu_0}^2
            &\leq C_1 M \sum_{k=1}^M \Var\bracks*{g_k(\bX_{\mathcal{J}_k}) \mid \bX_{\mathcal{J}_k} \in A_{\mathcal{J}_k}} \\
            &\leq C_1 s \Var\bracks*{f(\bX) \mid \bX \in A}.
        \end{split}
    \end{equation}
    Since $\E_P\bracks*{f(\bX) \mid \bX \in A} \in [\inf_A f, \sup_A f]$, the triangle inequality gives
    \begin{equation}
        \sup_{\bx \in A} \abs*{f(\bx) - \E_P\bracks*{f(\bX) \mid \bX \in A}}^2 \leq \operatorname{osc}(f, A)^2 \leq 4 \sup_{\bx \in A} \abs*{f(\bx) - \mu_0}^2 \leq 4 C_2 \Var\bracks*{f(\bX) \mid \bX \in A},
    \end{equation}
    where $\operatorname{osc}(f, A) \coloneqq \sup_{\bx, \by \in A} \abs*{f(\bx) - f(\by)}$, $C_2$ depends only on $s$, establishing local norm equivalence.
\end{proof}

\subsection{LRP witness lemmas}

\begin{lemma}[Cellwise LRP witnesses imply sufficient impurity decrease]
\label{lemma:sufficient-SID}
Assume bounded density (Assumption~\ref{assum:bounded-density}).
Suppose there are constants $R,C_s>0$ such that, for every
cell $A=\prod_{j=1}^s I_j$, one can find a coordinate
$k\in[s]$ and an absolutely continuous function $g\colon I_k\to\mathbb R$
satisfying \eqref{eq:assum-TV} on every subinterval of $I_k$ and
\(
\Var(g(X_k)\mid X_k\in I_k)>0
\)
whenever \(\Var(f(\bX)\mid\bX\in A)>0\), and
\begin{equation}\label{eq:conditions-SID-0}
    \Cov\bigl(f(\bX),g(X_k)\mid\bX\in A\bigr)^2
    \ge
    C_s\Var\bigl(f(\bX)\mid\bX\in A\bigr)
    \Var\bigl(g(X_k)\mid X_k\in I_k\bigr).
\end{equation}
Then $(f,P)$ satisfies sufficient impurity decrease (Assumption~\ref{assum:sid}), with a constant depending only
on $(s,p_{\min},p_{\max},R,C_s)$.
\end{lemma}

Condition~\eqref{eq:conditions-SID-0} says that, within every
cell, the target has a nontrivial correlation with some
univariate LRP direction.  The LRP requirement is essential:
correlation with an arbitrarily oscillatory univariate function
need not be captured by a single threshold split.

\begin{remark}[Intuition and related work]
On every cell with $\Var(f(\bX)\mid\bX\in A)>0$, Lemma~\ref{lemma:equiv-proj-var} shows that condition \eqref{eq:conditions-SID-0} is equivalent, up to constants depending only on the density bounds, to requiring that $\rho(f,g\mid \bX\in A)^2$ be uniformly bounded below by a positive constant,
where $\rho(f,g\mid \bX\in A)$ denotes the population Pearson product--moment correlation coefficient, as used in the CART analysis of \citet{klusowski2020sparse}. Cells with zero conditional variance satisfy SID trivially. Thus, \eqref{eq:conditions-SID-0} shows that if there exists a univariate function $g$ in the LRP function class such that $f$ and $g$ are sufficiently strongly correlated within every nondegenerate cell, then $f$ satisfies the SID condition. We refer to \citet[see, e.g., Lemma A.4]{klusowski2020sparse} for further discussion of the role of Pearson product--moment correlation in the theoretical analysis of tree-based methods.
\end{remark}

\begin{lemma}\label{lemma:help-sufficient-sid}
    Let $q_A^{(k)}(t) \coloneqq \P\{X_k \leq t\mid\bX \in A\}$.
    Under the setting of Lemma~\ref{lemma:sufficient-SID}, for
    the coordinate $k$ and witness $g$ associated with a cell
    $A=\prod_{j=1}^s[\ell_j,u_j]$, we have
    \begin{equation*}
        \max_{ b \in \mathbb{R}} \sqrt{\Delta(A, k, b)} \gtrsim \frac{\sqrt{\P(A)}\,\Var[f(\bX)|\bX \in A]^{1/2}\paren*{\Var[g(X_k) \mid X_k \in [\ell_k,u_k]]}^{1/2}}{\int_{\ell_k}^{u_k} \sqrt{q_A^{(k)}(t)(1 - q_A^{(k)}(t))}\,|g'(t)|\,\mathrm{d}t}.
    \end{equation*}
\end{lemma}

\begin{proof}
    Center $f$ and $g$ under the conditional law given
    $\bX\in A$.  This does not change either the impurity
    decrease or the total variation of $g$.  By averaging the
    split signal against the variation measure of $g$,
    \begin{equation}\label{eq:proof-help-sid-special-case-1}
    \max_{b \in \mathbb{R}} \sqrt{\Delta(A, k, b)} \geq \frac{\int_{\ell_k}^{u_k} \sqrt{q_A^{(k)}(t)(1 - q_A^{(k)}(t))}\sqrt{\Delta(A, k, t)}\,|g'(t)|\,\mathrm{d}t}{\int_{\ell_k}^{u_k} \sqrt{q_A^{(k)}(t)(1 - q_A^{(k)}(t))}\,|g'(t)|\,\mathrm{d}t}.
    \end{equation}
    Lemma~B.1 of \citet{Mazumder2024} (one may also refer to calculations in \citet{klusowski2020sparse}) and integration by parts
    give
    \begin{equation}
    \begin{aligned}\label{eq:proof-help-sid-special-case-2}
        & \int_{\ell_k}^{u_k} \sqrt{q_A^{(k)}(t)(1 - q_A^{(k)}(t))}\sqrt{\Delta(A, k, t)}\,|g'(t)|\,\mathrm{d}t \\
        &\geq \frac{1}{\sqrt{\P(A)}}
        \left|\E\left[f(\bX)g(X_k)\1_{\{\bX\in A\}}\right]\right| \\
        &=\sqrt{\P(A)}\,
        \left|\Cov\bigl(f(\bX),g(X_k)\mid\bX\in A\bigr)\right|.
    \end{aligned}
    \end{equation}
    Combining \eqref{eq:proof-help-sid-special-case-1}--\eqref{eq:proof-help-sid-special-case-2}
    with \eqref{eq:conditions-SID-0} proves the claim.
\end{proof}

\begin{proof}[Proof of Lemma~\ref{lemma:sufficient-SID}]
    For any cell $A = \prod_{j=1}^s [\ell_j, u_j] \subseteq [0, 1]^s$ and the fixed $k$, it holds that
    \begin{equation}
    \begin{aligned}\label{eq:proof-suff-sid-1}
        \left( \int_{\ell_k}^{u_k} \sqrt{q_A^{(k)}(t)(1 - q_A^{(k)}(t))}\,|g'(t)|\,\mathrm{d}t \right)^2 & \leq \frac{1}{4} \left( \int_{\ell_k}^{u_k} |g'(t)| \, \mathrm{d}t \right)^2 \\
        & \leq \frac{R/4}{u_k - \ell_k} \inf_{w \in \mathbb{R}} \int_{\ell_k}^{u_k} |g(t) - w|^2 \, dt,
    \end{aligned}
    \end{equation}
    where the first inequality uses $q_A^{(k)}(t)(1-q_A^{(k)}(t))\leq 1/4$, and the second follows because $g$ satisfies \eqref{eq:assum-TV} on $I_k$.

    Let $p_{X_k}$ denote the marginal density of $X_k$. By the variational characterization of variance, we have
    \begin{equation}\label{eq:proof-suff-sid-2}
    \begin{aligned}
        \Var[g(X_k) \mid X_k \in [\ell_k,u_k]]
        &= \inf_{w \in \mathbb{R}} \E\left[|g(X_k) - w|^2 \mid X_k \in [\ell_k, u_k]\right] \\
        &= \inf_{w \in \mathbb{R}} \int_{\ell_k}^{u_k} |g(t) - w|^2 \frac{p_{X_k}(t)}{\P\{X_k \in [\ell_k,u_k]\}} \, dt \\
        &\asymp \frac{1}{u_k - \ell_k} \inf_{w \in \mathbb{R}} \int_{\ell_k}^{u_k} |g(t) - w|^2 \, dt,
    \end{aligned}
    \end{equation}
    where the asymptotic equivalence follows from bounded density (Assumption~\ref{assum:bounded-density}). Combining \eqref{eq:proof-suff-sid-1} and \eqref{eq:proof-suff-sid-2} yields
    \begin{equation}\label{eq:proof-suff-sid-3}
    \begin{aligned}
        \left( \int_{\ell_k}^{u_k} \sqrt{q_A^{(k)}(t)(1 - q_A^{(k)}(t))}\,|g'(t)|\,\mathrm{d}t \right)^2
        \lesssim \Var[g(X_k) \mid X_k \in [\ell_k,u_k]].
    \end{aligned}
    \end{equation}
    Finally, squaring the conclusion of
    Lemma~\ref{lemma:help-sufficient-sid} and using
    \eqref{eq:proof-suff-sid-3} yields
    \[
        \max_{b\in\mathbb R}\Delta(A,k,b)
        \gtrsim
        \P(A)\Var[f(\bX)\mid\bX\in A].
    \]
    Taking the maximum over coordinates proves sufficient impurity decrease.
\end{proof}

\subsection{LRP-based multivariate models}

\begin{proof}[Proof of Proposition~\ref{prop:sid-suffi} (i)]
    \textit{Step 1: Proof of sufficient impurity decrease.}
    We verify the cellwise witness condition with one of the
    component functions.  The claim is immediate on cells where
    $\Var_P[f(\bX)]=0$, so assume this variance is positive.
    By Jensen's inequality and the triangle inequality,
    \begin{equation}\label{eq:sid-add-1}
    \begin{aligned}
        &\sum_{k=1}^s \E_P\bracks*{ \left| \left( f(\bX) - \E_P\bracks*{f(\bX)} \right) \left( f_k(X_k) - \E_P\bracks*{f_k(X_k)} \right) \right| } \\
        &\quad \geq \sum_{k=1}^s \abs*{ \E_P\bracks*{ \left( f(\bX) - \E_P\bracks*{f(\bX)} \right) \left( f_k(X_k) - \E_P\bracks*{f_k(X_k)} \right) } } \\
        &\quad \geq \abs*{ \E_P\bracks*{ \left( f(\bX) - \E_P\bracks*{f(\bX)} \right) \sum_{k=1}^s \left( f_k(X_k) - \E_P\bracks*{f_k(X_k)} \right) } } \\
        &\quad = \Var_P[f(\bX)].
    \end{aligned}
    \end{equation}
    Furthermore, by Lemma~\ref{lemma:equiv-unif-P},
    \begin{equation}\label{eq:sid-add-2}
        \Var_P[f(\bX)]^{1/2} \asymp\Var_{\nu}[f(\bX)]^{1/2}
        = \bracks*{\sum_{k=1}^s \Var_{\nu}[f_k(X_k)]}^{1/2}
        \asymp \bracks*{\sum_{k=1}^s \Var_P[f_k(X_k)]}^{1/2}.
    \end{equation}
    By Lemma~\ref{lemma:equiv-proj-var},
    $\Var_P[f_k(X_k)]\asymp
    \Var[f_k(X_k)\mid X_k\in I_k]$.  Write
    $c_k=\Cov_P[f(\bX),f_k(X_k)]$ and
    $v_k=\Var[f_k(X_k)\mid X_k\in I_k]$.  Equations
    \eqref{eq:sid-add-1}--\eqref{eq:sid-add-2} imply
    \[
        \sum_{k=1}^s|c_k|\ge \Var_P[f(\bX)],
        \qquad
        \sum_{k=1}^s v_k\lesssim \Var_P[f(\bX)].
    \]
    Cauchy--Schwarz therefore gives
    $\sum_{k:v_k>0}c_k^2/v_k\gtrsim\Var_P[f(\bX)]$.
    Hence some $k\in[s]$ satisfies
    $c_k^2\gtrsim\Var_P[f(\bX)]v_k$.  Since
    the restriction of $f_k$ to $I_k$ satisfies
    \eqref{eq:assum-TV}, the sufficient-impurity-decrease
    conclusion follows from Lemma~\ref{lemma:sufficient-SID}.

    \textit{Step 2: Proof of local norm equivalence.} Let $\mu_0 = \sum_{k=1}^s \E\bracks*{f_k(X_k) \mid X_k \in I_k}$. By the Cauchy--Schwarz inequality and the LRP condition,
    \begin{equation}\label{eq:eq-vd-add-1}
    \begin{aligned}
        \sup_{\bx \in A} \abs*{f(\bx) - \mu_0}^2
        &\leq s \sum_{k=1}^s \sup_{x_k \in I_k} \abs*{f_k(x_k) - \E\bracks*{f_k(X_k) \mid X_k \in I_k}}^2 \\
        &\leq s \sum_{k=1}^s \left(\int_{I_k}|f_k'(t)|\,\mathrm{d}t\right)^2 \\
        &\leq C_1 s \sum_{k=1}^s \Var_\nu[f_k].
    \end{aligned}
    \end{equation}
    Under the uniform distribution $\nu_A$, the coordinates $X_1, \ldots, X_s$ are independent, so
    \begin{equation}\label{eq:eq-vd-add-2}
        \sum_{k=1}^s \Var_\nu[f_k] = \Var_\nu[f].
    \end{equation}
    By Lemma~\ref{lemma:equiv-unif-P}, $\Var_\nu[f] \asymp \Var_P[f]$. Combining this with \eqref{eq:eq-vd-add-1} and \eqref{eq:eq-vd-add-2} yields
    \begin{equation}\label{eq:eq-vd-add-3}
        \sup_{\bx \in A} \abs*{f(\bx) - \mu_0}^2 \leq C_2 \Var_P[f],
    \end{equation}
    where $C_2 > 0$ depends only on $c_1, c_2, s$, and $C_1$. Since $\E_P\bracks*{f(\bX) \mid \bX \in A} \in [\inf_A f, \sup_A f]$, the triangle inequality gives
    \begin{equation}
        \sup_{\bx \in A} \abs*{f(\bx) - \E_P\bracks*{f(\bX) \mid \bX \in A}}^2 \leq \operatorname{osc}(f, A)^2 \leq 4 \sup_{\bx \in A} \abs*{f(\bx) - \mu_0}^2 \leq 4 C_2 \Var_P[f],
    \end{equation}
    where $\operatorname{osc}(f, A) \coloneqq \sup_{\bx, \by \in A} \abs*{f(\bx) - f(\by)}$, establishing local norm equivalence.
\end{proof}

\begin{proof}[Proof of Proposition~\ref{prop:sid-suffi} (ii)]
    Fix a cell $A=\prod_{k=1}^s I_k$. Because the coordinates are
    mutually independent, they remain mutually independent under $P_A$, and
    the conditional law of $X_k$ is the law of $X_k$ given $X_k\in I_k$.
    Define
    \begin{equation}\label{eq:mean-sigma}
      \mu_k \coloneqq \E_P\bracks*{h_k(X_k)}, \qquad
      \sigma_k^2 \coloneqq \Var_P[h_k(X_k)], \qquad k\in[s].
    \end{equation}
    Thus $h_{\min}\leq\mu_k\leq h_{\max}$ and
    $0\leq\sigma_k^2\leq(h_{\max}-h_{\min})^2$.
    Set $\underline h\coloneqq\min\{1,h_{\min}\}$ and
    $\overline h\coloneqq\max\{1,h_{\max}\}$.

    \textit{Sufficient impurity decrease.}
    Set $f_j(\bx)=\prod_{\ell\in\mathcal J_j}h_\ell(x_\ell)$ and
    $c_k=\Cov_P[f(\bX),h_k(X_k)]$. Conditional independence gives
    \begin{equation}\label{eq:cov-products-P}
      \Cov_P[f_j(\bX),h_k(X_k)]
      =
      \begin{cases}
        \displaystyle \sigma_k^2
        \prod_{\ell\in\mathcal J_j\setminus\{k\}}\mu_\ell,
        & k\in\mathcal J_j,\\[6pt]
        0, & k\notin\mathcal J_j.
      \end{cases}
    \end{equation}
    Since every $k$ belongs to at least one $\mathcal J_j$,
    \begin{equation}\label{eq:cov-products-lower}
      c_k
      =\sigma_k^2\sum_{j:k\in\mathcal J_j}
        \prod_{\ell\in\mathcal J_j\setminus\{k\}}\mu_\ell
      \geq \underline h^{s-1}\sigma_k^2.
    \end{equation}

    We next bound the variance of $f$. For every $j$, independence yields
    \begin{equation*}
      \Var_P[f_j]
      =\prod_{k\in\mathcal J_j}(\sigma_k^2+\mu_k^2)
       -\prod_{k\in\mathcal J_j}\mu_k^2
      \leq C_1\sum_{k\in\mathcal J_j}\sigma_k^2,
    \end{equation*}
    where $C_1$ depends only on $s,h_{\min},h_{\max}$. Indeed, this
    follows by factoring out $\prod_k\mu_k^2$ and using
    $\prod_k(1+r_k)-1\leq e^{\sum_k r_k}\sum_k r_k$ with
    $r_k=\sigma_k^2/\mu_k^2$. Consequently,
    \begin{equation}\label{eq:var-products-upper}
      \Var_P[f]
      \leq M\sum_{j=1}^M\Var_P[f_j]
      \leq C_2\sum_{k=1}^s\sigma_k^2,
    \end{equation}
    where $C_2$ depends only on $s,M,h_{\min},h_{\max}$.
    If $\sum_k\sigma_k^2=0$, then \eqref{eq:var-products-upper} gives
    $\Var_P[f]=0$, so SID is immediate. Otherwise,
    \begin{equation*}
      \sum_{k:\sigma_k>0}\frac{c_k^2}{\sigma_k^2}
      \geq
      \frac{(\sum_k c_k)^2}{\sum_k\sigma_k^2}
      \geq \underline h^{2s-2}\sum_k\sigma_k^2
      \geq \frac{\underline h^{2s-2}}{C_2}\Var_P[f].
    \end{equation*}
    Hence some $k$ satisfies
    $c_k^2\gtrsim \Var_P[f]\sigma_k^2$. The restriction of $h_k$ to
    $I_k$ is an LRP witness, so Lemma~\ref{lemma:sufficient-SID} proves SID.

    \textit{Local norm equivalence.}
    Since $\E_P\bracks*{f} \in [\inf_A f,\, \sup_A f]$,
    \begin{equation}\label{eq:sup-by-osc}
      \sup_{\bx \in A}\abs{f(\bx) - \E_P\bracks*{f}}
      \leq \operatorname{osc}(f, A)
      \coloneqq \sup_{\bx, \by \in A}\abs{f(\bx) - f(\by)}.
    \end{equation}

    Fix $\bx, \by \in A$. Define $\bz^{(0)} = \by$ and $\bz^{(s)} = \bx$, where $\bz^{(k)}$ agrees with~$\bx$ on coordinates $1, \ldots, k$ and with~$\by$ on coordinates $k+1, \ldots, s$. Then
    \begin{equation}
      f(\bx) - f(\by)
      = \sum_{k=1}^{s} \bigl(f(\bz^{(k)}) - f(\bz^{(k-1)})\bigr).
    \end{equation}
    The $k$-th summand alters only coordinate~$k$ from~$y_k$ to~$x_k$. For each $j \in [M]$:
    \begin{equation*}
      f_j(\bz^{(k)}) - f_j(\bz^{(k-1)})
      = \begin{cases}
        \bigl(h_k(x_k) - h_k(y_k)\bigr)\displaystyle\prod_{\ell \in \mathcal{J}_j
        \setminus \{k\}} h_\ell(z_\ell^{(k)}),
        & k \in \mathcal{J}_j, \\[4pt]
        0, & k \notin \mathcal{J}_j,
      \end{cases}
    \end{equation*}
    where each factor $h_\ell(z_\ell^{(k)}) \in [h_{\min}, h_{\max}]$.
    Hence
    \begin{equation*}
      \abs{f(\bz^{(k)}) - f(\bz^{(k-1)})}
      \leq M\, \overline h^{s-1}\, \abs{h_k(x_k) - h_k(y_k)}.
    \end{equation*}
    Taking the supremum over $x_k, y_k \in I_k$ and summing over $k$ yields
    \begin{equation}\label{eq:osc-f}
      \operatorname{osc}(f, A) \leq M\, \overline h^{s-1}
      \sum_{k=1}^{s} \operatorname{osc}(h_k, I_k).
    \end{equation}

    For each $k \in [s]$,
    \begin{equation}\label{eq:osc-sigma}
      \operatorname{osc}(h_k, I_k)
      \leq \int_{I_k}|h_k'(t)|\,\mathrm dt
      \lesssim_{p_{\min},p_{\max},R} \sigma_k.
    \end{equation}
    Here we used absolute continuity, LRP on $I_k$, the variational
    characterization of variance, and the upper and lower density bounds for
    $X_k$ conditional on $X_k\in I_k$.

    By Cauchy--Schwarz and \eqref{eq:cov-products-lower}, for $\sigma_k>0$,
    \begin{equation}\label{eq:Sigma-upper}
      \sigma_k \leq \underline h^{1-s}\Var_P[f]^{1/2}.
    \end{equation}
    This bound holds trivially when $\sigma_k = 0$. Substituting into \eqref{eq:osc-sigma} and summing over $k$:
    \begin{equation}\label{eq:sid-2-1}
        \sum_{k=1}^{s} \operatorname{osc}(h_k, I_k)
        \lesssim_{s,p_{\min},p_{\max},R,h_{\min}}
        \Var_P[f]^{1/2}.
    \end{equation}

    Combining \eqref{eq:sup-by-osc}, \eqref{eq:osc-f}, and
    \eqref{eq:sid-2-1}, and squaring both sides, gives
    \begin{equation}
      \sup_{\bx \in A}\abs{f(\bx) - \E_P\bracks*{f}}^2 \leq C_4\, \Var_P[f],
    \end{equation}
    where $C_4>0$ depends only on
    $s,M,p_{\min},p_{\max},R,h_{\min},h_{\max}$. This proves local
    norm equivalence.
\end{proof}

\subsection{Bias visibility from local norm equivalence}

\begin{proof}[Proof of Proposition~\ref{prop:lne-implies-bias-visibility}]
Fix a relevant cell $A\in\mathcal A_S$ and $\bx\in A$.
For brevity, write
\[
B(\bx;A):=\bigl|f^*(\bx)-\E[f^*(\bX)\mid \bX\in A]\bigr|.
\]
If $V(A)=0$, then local norm equivalence gives $\Osc(A)=0$, and hence $B(\bx;A)=0$.
Assume $V(A)>0$ and write
\[
q(\bx):=\frac{s}{2\bar\alpha_S(\bx)+s}.
\]
Since
\[
B(\bx;A)^2
\le
\Osc(A)^2
\le
\rho\,\Var[f^*(\bX)\mid\bX\in A]
=
\rho\frac{V(A)}{\P(A)},
\]
whenever
\[
\P(A)
\ge
\chi
\paren*{\frac{V(A)}{L^2}}^{
s/(2\bar\alpha_S(\bx)+s)},
\]
we have
\[
B(\bx;A)^2
\le
\rho\chi^{-1}
L^{2q(\bx)}
V(A)^{1-q(\bx)}.
\]
If $\zeta^2\ge\rho\chi^{-1}$, then
\[
B(\bx;A)
\le
\zeta\,
L^{s/(2\bar\alpha_S(\bx)+s)}
V(A)^{
\bar\alpha_S(\bx)/(2\bar\alpha_S(\bx)+s)},
\]
which is Assumption~\ref{assum:variance-oscillation}.
\end{proof}

\subsection{Pathwise competitive split-scale regularity}

\begin{proof}[Proof of Proposition~\ref{prop:pathwise-split-scale-from-primitives}]
All constants in this proof are independent of $\lambda$ unless the dependence is displayed.
Since $\Delta_{\max}(B)\le V(B)$ for every cell $B$, we may replace $\lambda$ by $\lambda\wedge1$ and assume $0<\lambda\le1$.
Fix a relevant path
\[
A^0\supset A^1\supset\cdots\supset A^D=A
\]
whose nonterminal splits are competitive.

Write
\[
\Delta_{\mathrm{pa}}
:=
\Delta(A^{D-1},j_{D-1},b_{D-1})
\]
for the terminal-parent split signal.
If $\Delta_{\mathrm{pa}}=0$, then \eqref{eq:reliable-split-scale} is immediate. Otherwise, let $A^{\mathrm{sib}}$ be the sibling of $A$ in the split of $A^{D-1}$ and write
\[
\mu_A:=\E[f^*(\bX)\mid \bX\in A],
\qquad
\mu_{\mathrm{sib}}:=\E[f^*(\bX)\mid \bX\in A^{\mathrm{sib}}].
\]
The impurity identity gives
\[
\frac{\Delta_{\mathrm{pa}}}{\P(A)}
\le
(\mu_A-\mu_{\mathrm{sib}})^2.
\]
By local norm equivalence, the child-mean contrast is bounded by the oscillation on $A^{D-1}$, and hence
\[
(\mu_A-\mu_{\mathrm{sib}})^2
\le
\Osc(A^{D-1})^2
\le
\rho\,\Var[f^*(\bX)\mid\bX\in A^{D-1}].
\]
Using bounded density (Assumption~\ref{assum:bounded-density}) to compare $\P(A)$ and $\Vol(A)$, we obtain
\[
\Var[f^*(\bX)\mid\bX\in A^{D-1}]
\gtrsim_{\rho,p_{\max}}
\frac{\Delta_{\mathrm{pa}}}{\Vol(A)}.
\]

Now fix a coordinate $j$ split somewhere along the path, and let $t_j$ be its last split time. Then $A^{D-1}\subseteq A^{t_j}$. By Lemma~\ref{lem:pathwise-variance-comparison-from-lne},
\[
\Var[f^*(\bX)\mid\bX\in A^{t_j}]
\ge
\rho^{-1}
\Var[f^*(\bX)\mid\bX\in A^{D-1}]
\gtrsim_{\rho,p_{\max}}
\frac{\Delta_{\mathrm{pa}}}{\Vol(A)}.
\]
The split at $A^{t_j}$ is competitive and occurs along $j$, so $\Delta_{\max}(A^{t_j},j)\ge \Delta_{\max}(A^{t_j})/2$. Combining sufficient impurity decrease with Lemma~\ref{lem:smooth-split-signal-upper} gives
\[
\Var[f^*(\bX)\mid\bX\in A^{t_j}]
\lesssim_{p_{\min},p_{\max},\|f^*\|_\infty}
\lambda^{-1}
L^2\len_j(A^{t_j})^{2\alpha_j(x_j)}.
\]
Therefore
\[
\len_j(A^{t_j})
\gtrsim_{p_{\min},p_{\max},\rho,\|f^*\|_\infty}
\lambda^{1/(2\alpha_j(x_j))}
L^{-1/\alpha_j(x_j)}
\left(
\frac{\Delta_{\mathrm{pa}}}{\Vol(A)}
\right)^{1/(2\alpha_j(x_j))}.
\]
The preceding variance lower bound and sufficient impurity decrease imply $\Delta_{\max}(A^{t_j})>0$, so Lemma~\ref{lem:reliable-split-balance-from-lne} applies.
By bounded density, and because no later split occurs along coordinate $j$, this gives
\[
\len_j(A)
\gtrsim_{p_{\min},p_{\max},\rho}
\lambda
\len_j(A^{t_j}).
\]
Combining the last two displays yields
\[
\len_j(A)
\gtrsim_{p_{\min},p_{\max},\rho,\|f^*\|_\infty}
\lambda^{1+1/(2\alpha_j(x_j))}
L^{-1/\alpha_j(x_j)}
\left(
\frac{\Delta_{\mathrm{pa}}}{\Vol(A)}
\right)^{1/(2\alpha_j(x_j))}.
\]
Since $\alpha_j(x_j)\ge\alpha_{\min}$ and $\lambda\le1$, the displayed lower bound is at least
\[
C_{p_{\min},p_{\max},\rho,\alpha_{\min},\|f^*\|_\infty}
\lambda^{1+1/(2\alpha_{\min})}
L^{-1/\alpha_j(x_j)}
\left(
\frac{\Delta_{\mathrm{pa}}}{\Vol(A)}
\right)^{1/(2\alpha_j(x_j))}
\]
with $C_{p_{\min},p_{\max},\rho,\alpha_{\min},\|f^*\|_\infty}$ as defined in the proposition statement.
Thus Assumption~\ref{assum:reliable-split-balance} holds with
\[
\omega
=
\min\left\{1,
C_{p_{\min},p_{\max},\rho,\alpha_{\min},\|f^*\|_\infty}
\lambda^{1+1/(2\alpha_{\min})}
\right\}.
\]
\end{proof}

\begin{proof}[Proof of Lemma~\ref{lem:reliable-split-balance-from-lne}]
Let $(A,j,b)$ be a competitive split, with children $A_L$ and $A_R$.
By sufficient impurity decrease and competitiveness,
\[
\Delta(A,j,b)
\ge
\frac{\lambda}{2}\P(A)\Var[f^*(\bX)\mid\bX\in A].
\]
On the other hand, the child-mean contrast is bounded by $\Osc(A)$, so local norm equivalence gives
\[
\Delta(A,j,b)
\le
\rho
\frac{\P(A_L)\P(A_R)}{\P(A)}
\Var[f^*(\bX)\mid\bX\in A].
\]
The positive-signal assumption and sufficient impurity decrease imply $\Var[f^*(\bX)\mid\bX\in A]>0$, so cancellation gives
\[
\frac{\P(A_L)\P(A_R)}{\P(A)^2}
\ge
\frac{\lambda}{2\rho}.
\]
Writing $x=\P(A_L)/\P(A)$, the left hand side is $x(1-x)$, which is bounded above by $x\wedge(1-x)$.
Thus $\P(A_L)\wedge\P(A_R)\ge \lambda\P(A)/(2\rho)$.
\end{proof}

\begin{proof}[Proof of Lemma~\ref{lem:pathwise-variance-comparison-from-lne}]
Since $A'\subseteq A$,
\[
\Var[f^*(\bX)\mid\bX\in A']
\le
\Osc(A')^2
\le
\Osc(A)^2
\le
\rho\,\Var[f^*(\bX)\mid\bX\in A].
\]
\end{proof}

\begin{proof}[Proof of Lemma~\ref{lem:smooth-split-signal-upper}]
Fix $A\in\mathcal A_S$, $j\in S$, and $b\in A_j$. Write $A=A_j\times A_{-j}$ with $A_j=[\ell_j,u_j]$. Define
\[
N(t):=\int_{A_{-j}} f^*(t,\bx_{-j})\,p(t,\bx_{-j})\,d\bx_{-j},
\qquad
D(t):=\int_{A_{-j}} p(t,\bx_{-j})\,d\bx_{-j},
\]
and let $m(t):=N(t)/D(t)$ for $t\in A_j$. Then $m(t)$ is the conditional mean of $f^*(\bX)$ given $\bX_{-j}\in A_{-j}$ and $X_j=t$.

By bounded density (Assumption~\ref{assum:bounded-density}), there exist constants $0<p_{\min}\le p_{\max}<\infty$ such that $p_{\min}\le p(\bx)\le p_{\max}$ for all $\bx\in[0,1]^d$. Hence for every $t\in A_j$,
\[
p_{\min}\Vol(A_{-j})
\le
D(t)
\le
p_{\max}\Vol(A_{-j}),
\]
so in particular the denominator is uniformly bounded away from zero.

We first bound the oscillation of $N(\cdot)$ on $A_j$. Set
$\alpha_j^+(A_j):=\sup_{v\in A_j}\alpha_j(v)$. Fix
$0<\epsilon<\alpha_j^+(A_j)$ and
choose $t_\epsilon\in A_j$ such that
$\alpha_j(t_\epsilon)\ge\alpha_j^+(A_j)-\epsilon$. For $t,t'\in A_j$,
the triangle inequality gives
\[
|N(t)-N(t')|
\le
|N(t)-N(t_\epsilon)|+|N(t')-N(t_\epsilon)|.
\]
Thus it suffices to bound $|N(t)-N(t_\epsilon)|$. Using
$ab-a'b'=(a-a')b+a'(b-b')$, we have
\begin{align*}
&|f^*(t,\bx_{-j})p(t,\bx_{-j})-f^*(t_\epsilon,\bx_{-j})p(t_\epsilon,\bx_{-j})|\\
&\le
|f^*(t,\bx_{-j})-f^*(t_\epsilon,\bx_{-j})|\,|p(t,\bx_{-j})|
+
|f^*(t_\epsilon,\bx_{-j})|\,|p(t,\bx_{-j})-p(t_\epsilon,\bx_{-j})|.
\end{align*}
Since $f^*$ belongs to the local H\"{o}lder class, applied with base
point $(t_\epsilon,\bx_{-j})$, we get
\[
|f^*(t,\bx_{-j})-f^*(t_\epsilon,\bx_{-j})|
\le
L\,|t-t_\epsilon|^{\alpha_j(t_\epsilon)}
\le
L\,|t-t_\epsilon|^{\alpha_j^+(A_j)-\epsilon}.
\]
Likewise, since $p$ satisfies the same local H\"{o}lder condition,
\[
|p(t,\bx_{-j})-p(t_\epsilon,\bx_{-j})|
\le
L\,|t-t_\epsilon|^{\alpha_j(t_\epsilon)}
\le
L\,|t-t_\epsilon|^{\alpha_j^+(A_j)-\epsilon}.
\]
Combining the last three displays with the bounds $|p|\le p_{\max}$ and $|f^*|\le \|f^*\|_\infty$, we obtain
\[
|f^*(t,\bx_{-j})p(t,\bx_{-j})-f^*(t_\epsilon,\bx_{-j})p(t_\epsilon,\bx_{-j})|
\le
C_1\,|t-t_\epsilon|^{\alpha_j^+(A_j)-\epsilon},
\]
where $C_1:=L(p_{\max}+\|f^*\|_{\infty})$. Integrating over $A_{-j}$ yields
\[
|N(t)-N(t_\epsilon)|
\le
C_1\,\Vol(A_{-j})\,|t-t_\epsilon|^{\alpha_j^+(A_j)-\epsilon}
\le
C_1\,\Vol(A_{-j})\,\len_j(A)^{\alpha_j^+(A_j)-\epsilon}.
\]
Hence
\[
|N(t)-N(t')|
\le
C_1\,\Vol(A_{-j})\,\len_j(A)^{\alpha_j^+(A_j)-\epsilon}
\qquad \forall\, t,t'\in A_j.
\]

Exactly the same argument applied to $D(t)=\int_{A_{-j}}p(t,\bx_{-j})\,d\bx_{-j}$ gives
\[
|D(t)-D(t')|
\le
C_2L\,\Vol(A_{-j})\,\len_j(A)^{\alpha_j^+(A_j)-\epsilon}
\qquad \forall\, t,t'\in A_j
\]
for some constant $C_2>0$.

We now control the oscillation of $m(\cdot)$. For any $t,t'\in A_j$,
\[
|m(t)-m(t')|
=
\left|\frac{N(t)}{D(t)}-\frac{N(t')}{D(t')}\right|
\le
\frac{|N(t)-N(t')|}{D(t)}
+
|N(t')|\frac{|D(t)-D(t')|}{D(t)D(t')}.
\]
Using $D(t),D(t')\ge p_{\min}\Vol(A_{-j})$ and
\[
|N(t')|
\le
\int_{A_{-j}} |f^*(t',\bx_{-j})|\,p(t',\bx_{-j})\,d\bx_{-j}
\le
\|f^*\|_\infty p_{\max}\Vol(A_{-j}),
\]
we obtain
\[
|m(t)-m(t')|
\le
\frac{C_1}{p_{\min}}\,\len_j(A)^{\alpha_j^+(A_j)-\epsilon}
+
\frac{\|f^*\|_\infty p_{\max} C_2}{p_{\min}^2}\,
\len_j(A)^{\alpha_j^+(A_j)-\epsilon}
\le
C_3L\,\len_j(A)^{\alpha_j^+(A_j)-\epsilon}
\]
for some constant $C_3>0$ depending only on $p_{\min},p_{\max}$, and $\|f^*\|_\infty$.

Finally, by Fubini's theorem,
\[
\E[f^*(\bX)\mid \bX\in A,\ X_j\le b]
=
\frac{\int_{A_j\cap(-\infty,b]} m(t)D(t)\,dt}
{\int_{A_j\cap(-\infty,b]} D(t)\,dt},
\]
and similarly
\[
\E[f^*(\bX)\mid \bX\in A,\ X_j>b]
=
\frac{\int_{A_j\cap(b,\infty)} m(t)D(t)\,dt}
{\int_{A_j\cap(b,\infty)} D(t)\,dt}.
\]
Therefore both quantities are weighted averages of $m(t)$ over subsets of $A_j$, with nonnegative weights. Hence their difference is bounded by the oscillation of $m$ on $A_j$, namely
\begin{align*}
&\big|
\E[f^*(\bX)\mid \bX\in A,\ X_j\le b]
-
\E[f^*(\bX)\mid \bX\in A,\ X_j>b]
\big| \\
&\qquad\le
\sup_{t,t'\in A_j}|m(t)-m(t')|
\le
C_3L\,\len_j(A)^{\alpha_j^+(A_j)-\epsilon}.
\end{align*}
Set $\mu_L=\E[f^*(\bX)\mid\bX\in A_L]$ and
$\mu_R=\E[f^*(\bX)\mid\bX\in A_R]$. Since
\[
\Delta(A,j,b)
=
\frac{\P(A_L)\P(A_R)}{\P(A)}
\bigl(\mu_L-\mu_R\bigr)^2
\le
\P(A)C_3^2L^2\len_j(A)^{2(\alpha_j^+(A_j)-\epsilon)}.
\]
Letting $\epsilon\downarrow0$ gives the same bound with exponent
$2\alpha_j^+(A_j)$. Since
$\alpha_j^+(A_j)\ge\alpha_j(x_j)$ for $\bx\in A$, the claim follows
because $\len_j(A)\le1$.
\end{proof}

\section{Proofs for the CART-MLS lower bound}
\label{app:cart-mls-lower-bound}

This appendix proves Theorem~\ref{thm::lowerright}. The geometric part of
the argument is common to honest and ordinary CART-MLS: in one dimension,
the MLS rule fixes the number of observations in each endpoint leaf up to
a constant factor. The stochastic arguments are different. For an honest
tree, the endpoint leaves are fixed before their noises are observed. For an
ordinary tree, the leaves depend on those noises, so we use a persistence
event that holds for every admissible endpoint leaf size.

For the honest estimator, write $\mathcal D^{(1)}$ and $\mathcal D^{(2)}$
for the independent tree-building and averaging samples, each of size $n$.
Let $N^{(\ell)}(A)$ denote the number of covariates from
$\mathcal D^{(\ell)}$ in $A$, and set
\[
 \bar f_A^{(\ell)}
 :=
 \frac{1}{N^{(\ell)}(A)}
 \sum_{i:X_i^{(\ell)}\in A}f^*(X_i^{(\ell)}),
 \qquad
 \bar\xi_A^{(\ell)}
 :=
 \frac{1}{N^{(\ell)}(A)}
 \sum_{i:X_i^{(\ell)}\in A}\xi_i^{(\ell)}
\]
whenever $N^{(\ell)}(A)>0$. For the ordinary estimator we suppress the
superscript.

\subsection{LRP property of the square-root example}

The following proposition places the square-root example within the LRP class
of Definition~\ref{def:LRP}. Together with
Proposition~\ref{prop:sid-suffi}(i), it shows that the lower bound in
Theorem~\ref{thm::lowerright} is not caused by a failure of sufficient impurity
decrease, but by the non-adaptivity of the minimum-leaf-size stopping rule.

\begin{proposition}[Square-root function satisfies LRP]
\label{prop:sqrt-lrp}
The function $g(x)=\sqrt{x}$ satisfies the locally reverse Poincar\'e condition
(Definition~\ref{def:LRP}) with constant $R=18$.
\end{proposition}

\begin{proof}[Proof of Proposition~\ref{prop:sqrt-lrp}]
Fix $0\le a<b\le1$, and let $U\sim\mathrm{Unif}[a,b]$. Write
$u=\sqrt a$, $v=\sqrt b$, and $r=u/v\in[0,1)$. Direct calculation gives
\[
\Var(\sqrt U)
=
v^2\frac{(1-r)^2(r^2+4r+1)}{18(r+1)^2}.
\]
Since $g$ is absolutely continuous and increasing,
\[
\int_a^b|g'(t)|\,\mathrm dt=v-u=v(1-r).
\]
Also,
\[
\frac{1}{b-a}\inf_{w\in\mathbb R}
\int_a^b|g(t)-w|^2\,\mathrm dt
=\Var(\sqrt U).
\]
Consequently,
\[
\frac{
\left(\int_a^b|g'(t)|\,\mathrm dt\right)^2
}{
(b-a)^{-1}\inf_{w\in\mathbb R}
\int_a^b|g(t)-w|^2\,\mathrm dt
}
=
\frac{18(r+1)^2}{r^2+4r+1}
\le18.
\]
This is precisely the LRP inequality with $R=18$.
\end{proof}

\subsection{Endpoint geometry and design}

\begin{lemma}[Endpoint leaf sizes]
\label{lem:mls-endpoint-counts}
Let $A_0$ and $A_1$ be the terminal cells containing $0$ and $1$ in a
one-dimensional CART-MLS tree. If $m_z$ is the number of tree-building
observations in $A_z$, then
\[
 N\le m_z<2N,\qquad z\in\{0,1\}.
\]
This includes the case in which the root is terminal.
\end{lemma}

\begin{proof}
Every non-root leaf was created by an admissible split and therefore contains
at least $N$ observations. A one-dimensional cell containing at least $2N$
observations admits a split with at least $N$ observations in each child, so
it cannot be terminal. If the root is terminal, then $n<2N$; since
$N\le n$, the same inequalities hold.
\end{proof}

\begin{lemma}[Endpoint design]
\label{lem:mls-endpoint-design}
Assume bounded density (Assumption~\ref{assum:bounded-density}) and
$f^*(x)=\sqrt{x}$. Let $A_0$ and $A_1$ be
any endpoint intervals, possibly selected using the first-sample responses,
that satisfy
\[
 N\le N^{(1)}(A_z)<2N,\qquad z\in\{0,1\}.
\]
There are constants $C,c_1,c_2>0$, depending only on
$p_{\min}$ and $p_{\max}$, with the following properties.

If $N\ge C\log n$, then there is an event depending only on two independent
design samples, with probability $1-O(n^{-2})$, on which the following
statements hold simultaneously for every such pair of intervals and every
$z\in\{0,1\}$:
\begin{align}
 c_1\frac{N}{n}
 \le \P(A_z)
 \le c_2\frac{N}{n},
 \qquad
 c_1\frac{N}{n}
 \le \Vol(A_z)
 \le c_2\frac{N}{n},
 \qquad
 c_1N
 \le N^{(2)}(A_z)
 \le c_2N,
 \label{eq:endpoint-design-counts}
\end{align}
and
\begin{align}
 \bar f_{A_0}^{(\ell)}
 &\ge c_1\sqrt{\frac{N}{n}},
 &
 1-\bar f_{A_1}^{(\ell)}
 &\ge c_1\frac{N}{n},
 \qquad \ell\in\{1,2\}.
 \label{eq:endpoint-design-biases}
\end{align}
Here the assertions involving the second sample are omitted when only one
sample is present.

If $N<C\log n$, then, with probability $1-O(n^{-2})$ for a single design
sample, every endpoint interval $A$ containing at least one and fewer than
$2N$ observations
satisfies
\begin{align}
 \frac{1}{N(A)}\sum_{X_i\in A}\sqrt{X_i}
 &\lesssim \sqrt{\frac{\log n}{n}},
 && A=[0,a],
 \label{eq:small-left-signal}\\
 1-\frac{1}{N(A)}\sum_{X_i\in A}\sqrt{X_i}
 &\lesssim \frac{\log n}{n},
 && A=[1-a,1].
 \label{eq:small-right-signal}
\end{align}
Finally, for two independent design samples and every fixed
$N<C\log n$, with probability at least a universal constant $p_D>0$,
the two endpoint leaves selected from the first sample each contain between
$1$ and $C\log n$ observations from the second sample.
\end{lemma}

\begin{proof}
Apply Lemma~\ref{lemma:uniform_cell_count} with $d=r=1$ and
$u=3\log n$ to the first sample and, when present, to the second sample.
If $N\ge C\log n$, Lemma~\ref{lem:mls-endpoint-counts} and the uniform
count comparison give
\[
 \P(A_z)\asymp \frac{N^{(1)}(A_z)}{n}\asymp\frac{N}{n},
 \qquad
 N^{(2)}(A_z)\asymp n\P(A_z)\asymp N.
\]
Bounded density also gives $\Vol(A_z)\asymp\P(A_z)$.
This proves \eqref{eq:endpoint-design-counts}.

We next prove the empirical bias bounds. Write $A_0=[0,a]$ and fix a
sufficiently large integer $K$, depending only on the density bounds. The
count inequalities, first for $A_0$ and then for $[0,a/K]$, give
\[
 N^{(\ell)}([0,a/K])
 \le C\frac{p_{\max}}{Kp_{\min}}N^{(\ell)}(A_0)+C\log n
 \le \frac12 N^{(\ell)}(A_0),
\]
after first taking $K$ sufficiently large and then increasing the constant
in the condition $N\ge C\log n$. Hence at least half of the observations satisfy
$X_i\ge a/K$, and therefore
\[
 \bar f_{A_0}^{(\ell)}\gtrsim\sqrt a
 \asymp\sqrt{\frac{N}{n}}.
\]
The argument applies to both samples because the second sample is independent
of $A_0$ and its cell counts satisfy the same uniform inequalities. For
$A_1=[1-a,1]$, at most half of the observations lie in
$[1-a/K,1]$. For every remaining observation,
\[
 1-\sqrt{X_i}=\frac{1-X_i}{1+\sqrt{X_i}}\ge\frac{a}{2K}.
\]
Thus $1-\bar f_{A_1}^{(\ell)}\gtrsim a\asymp N/n$, proving
\eqref{eq:endpoint-design-biases}.

If $N<C\log n$, the small-cell part of
Lemma~\ref{lemma:uniform_cell_count} gives $\P(A)\lesssim\log n/n$
for every endpoint interval containing fewer than $2N$ observations.
Bounded density then gives $\Vol(A)\lesssim\log n/n$.
Monotonicity of $\sqrt{\cdot}$ then gives
\eqref{eq:small-left-signal} and \eqref{eq:small-right-signal}.

It remains to verify the second-sample occupancy assertion. Let
$X_{(1)}^{(1)}\le\cdots\le X_{(n)}^{(1)}$ be the first-sample order
statistics, and let $F$ be the design distribution function. Standard
binomial tail bounds applied to the uniform order statistics
$F(X_{(1)}^{(1)}),\ldots,F(X_{(n)}^{(1)})$ imply, uniformly over
$1\le N<C\log n$, that with probability bounded below,
\[
 F(X_{(N)}^{(1)})\gtrsim \frac{N}{n},
 \qquad
 1-F(X_{(n-N+1)}^{(1)})\gtrsim\frac{N}{n},
\]
while
\[
 F(X_{(2N)}^{(1)})\lesssim\frac{\log n}{n},
 \qquad
 1-F(X_{(n-2N+1)}^{(1)})\lesssim\frac{\log n}{n}.
\]
On this event, each selected endpoint leaf has probability mass at least
$cN/n$ and at most $C\log n/n$. Conditional on the first design, the two endpoint
counts in the second sample are binomial marginals on disjoint intervals.
The identity
\[
 \P\{N^{(2)}(A_0)>0,\ N^{(2)}(A_1)>0\mid\mathcal D^{(1)}\}
 =
 1-(1-a_0)^n-(1-a_1)^n+(1-a_0-a_1)^n,
\]
where $a_z=\P(A_z)$, is bounded below by a universal positive constant.
A binomial upper-tail bound also gives
$N^{(2)}(A_z)\le C\log n$ for both endpoints with probability tending to
one. Combining these events proves the claim.
\end{proof}

\subsection{Gaussian endpoint-walk lemmas}

Conditional on the ordered covariates, the noises remain i.i.d. Gaussian.
The following two lemmas control the corresponding partial-sum walks near
the left and right endpoints.

\begin{lemma}[Mixed-sign persistence]
\label{lem:mls-mixed-persistence}
Let $\xi_1,\ldots,\xi_n$ be i.i.d.\ $\mathcal N(0,\sigma^2)$. Define
\[
 S_L(k)=\sum_{i=1}^k\xi_i,
 \qquad
 S_R(k)=\sum_{i=n-k+1}^n\xi_i.
\]
There are universal constants $a,p_{\mathrm{per}}>0$ such that, whenever
$4N\le n$,
\[
 \P\left\{
 \inf_{N\le k<2N}S_L(k)\ge a\sigma\sqrt N,\quad
 \sup_{N\le k<2N}S_R(k)\le-a\sigma\sqrt N
 \right\}
 \ge p_{\mathrm{per}}.
\]
\end{lemma}

\begin{proof}
The increments used by the two endpoint walks are disjoint when
$4N\le n$, so the walks are independent. For the left walk,
consider
\[
 S_L(N)\ge2a\sigma\sqrt N,
 \qquad
 \max_{0\le j\le N}|S_L(N+j)-S_L(N)|\le a\sigma\sqrt N.
\]
The two events in this display are independent. By Brownian scaling, their
intersection has probability at least
\[
 \{1-\Phi(2a)\}
 \P\left\{\sup_{0\le t\le1}|B(t)|\le a\right\}>0,
\]
uniformly in $N$, where $B$ is standard Brownian motion. On this
intersection, $S_L(k)\ge a\sigma\sqrt N$ for every $N\le k<2N$.
The same argument applied to $-S_R$ gives the right event. Independence of
the endpoint walks yields the result.
\end{proof}

\begin{lemma}[Small-window anti-concentration]
\label{lem:mls-small-window}
With the notation of Lemma~\ref{lem:mls-mixed-persistence}, fix a universal
constant $C_0>0$. For all sufficiently large $n$ and every
$M\le C_0\log n$,
\[
 \P\left\{
 \begin{array}{l}
 |S_L(k)|\ge
 \dfrac{a\sigma\sqrt{k}}{\log n\,\log\log n},\\[5pt]
 |S_R(k)|\ge
 \dfrac{a\sigma\sqrt{k}}{\log n\,\log\log n}
 \end{array}
 \quad\text{for every }1\le k\le M
 \right\}
 \ge1-\frac{CC_0}{\log\log n}
\]
for a sufficiently small universal constant $a>0$.
\end{lemma}

\begin{proof}
For every fixed $k$, both $S_L(k)/(\sigma\sqrt{k})$ and
$S_R(k)/(\sigma\sqrt{k})$ are standard normal. Since
$\P\{|Z|\le t\}\le Ct$ for $Z\sim\mathcal N(0,1)$, a union bound over the
$2M$ inequalities shows that the complement of the displayed event has
probability at most
\[
 \frac{C M}{\log n\,\log\log n}
 \le\frac{CC_0}{\log\log n}.
\]
\end{proof}

\subsection{Proof of Theorem~\ref{thm::lowerright}}

\begin{proof}
Fix $N\in[n]$, and let $A_0,A_1$ be the endpoint leaves. We prove the two
endpoint inequalities jointly.
For the ordinary estimator, index the observations in increasing order of
their covariates. Conditional on the design, the correspondingly relabeled
noises $\xi_1,\ldots,\xi_n$ remain i.i.d.\ $\mathcal N(0,\sigma^2)$, and
$S_L,S_R$ denote the endpoint walks in
Lemma~\ref{lem:mls-mixed-persistence}.

\medskip\noindent
\textit{Honest CART-MLS, $C\log n\le N\le n/4$.}
Work on the event in Lemma~\ref{lem:mls-endpoint-design}. Conditional on
the two designs, the leaves and their second-sample counts
$m_z=N^{(2)}(A_z)$ are fixed, with $m_z\asymp N$. The tree has at least
one split in this range, so $A_0$ and $A_1$ are disjoint. Consequently,
$\bar\xi_{A_0}^{(2)}$ and $\bar\xi_{A_1}^{(2)}$ are independent centered
Gaussian variables. With probability bounded below by a universal positive
constant, jointly,
\[
 \bar\xi_{A_0}^{(2)}\ge\frac{c\sigma}{\sqrt N},
 \qquad
 \bar\xi_{A_1}^{(2)}\le-\frac{c\sigma}{\sqrt N}.
\]
These signs reinforce the two deterministic biases in
\eqref{eq:endpoint-design-biases}. Therefore
\[
 |\widehat f_N(0)-f^*(0)|
 \ge c\left(\sqrt{\frac{N}{n}}+\frac{\sigma}{\sqrt N}\right),
 \qquad
 |\widehat f_N(1)-f^*(1)|
 \ge c\left(\frac{N}{n}+\frac{\sigma}{\sqrt N}\right).
\]

\medskip\noindent
\textit{Ordinary CART-MLS, $C\log n\le N\le n/4$.}
The observations in $A_0$ are the first $m_0$ order statistics and those
in $A_1$ are the last $m_1$, where
Lemma~\ref{lem:mls-endpoint-counts} gives $N\le m_z<2N$ even though
$m_z$ depends on the responses. On the mixed-sign persistence event,
\[
 \frac{S_L(m_0)}{m_0}\ge\frac{c\sigma}{\sqrt N},
 \qquad
 \frac{S_R(m_1)}{m_1}\le-\frac{c\sigma}{\sqrt N}.
\]
The event has constant probability conditionally on the design, and the
design event in Lemma~\ref{lem:mls-endpoint-design} has probability tending
to one. Again the noise signs reinforce the deterministic biases, proving
\eqref{eq:mls-lower-main-regime}.

\medskip\noindent
\textit{The range $N>n/4$.}
Here the design lemma gives endpoint biases bounded below by positive
constants. For ordinary CART, Doob's maximal inequality gives, for a
sufficiently large universal $K$, a constant-probability event on which
\[
 \max_{1\le k\le n}|S_L(k)|
 +\max_{1\le k\le n}|S_R(k)|
 \le K\sigma\sqrt n.
\]
Since every selected endpoint leaf contains at least $N>n/4$ observations,
both noise averages are $O(\sigma/\sqrt n)$. For honest CART, the same
conclusion follows conditionally from $m_z\asymp N$ and ordinary Gaussian
tail bounds. After increasing $n_0$, these noise terms are smaller
than half of the corresponding biases. Since
$\sigma/\sqrt N=O(\sigma/\sqrt n)$ in this range, the two bounds in
\eqref{eq:mls-lower-main-regime} follow for either estimator.

\medskip\noindent
\textit{Honest CART-MLS, small $N$.}
Work on the positive-probability occupancy event in
Lemma~\ref{lem:mls-endpoint-design}. Conditional on the two designs, the
endpoint leaves are disjoint and their averaging counts satisfy
$1\le m_z\le C\log n$. With constant conditional probability,
\[
 \bar\xi_{A_0}^{(2)}\ge\frac{c\sigma}{\sqrt{\log n}},
 \qquad
 \bar\xi_{A_1}^{(2)}\le-\frac{c\sigma}{\sqrt{\log n}}.
\]
Since $f^*\ge0=f^*(0)$ and $f^*\le1=f^*(1)$, these signs reinforce the
signal errors at the respective endpoints. This proves the honest bound in
\eqref{eq:mls-lower-small-regime}.

\medskip\noindent
\textit{Ordinary CART-MLS, small $N$.}
Set $M=\lceil2C\log n\rceil$ and intersect the design event in
Lemma~\ref{lem:mls-endpoint-design} with the event in
Lemma~\ref{lem:mls-small-window}. Since $N\le m_z<2N\le M$, the selected
noise averages satisfy
\[
 \left|\frac{S_L(m_0)}{m_0}\right|
 \wedge
 \left|\frac{S_R(m_1)}{m_1}\right|
 \ge
 \frac{c\sigma}{\sqrt N\,\log n\,\log\log n}.
\]
By \eqref{eq:small-left-signal} and \eqref{eq:small-right-signal}, the
absolute signal errors are at most
$C\sqrt{\log n/n}$ and $C\log n/n$, respectively. Uniformly over
$N<C\log n$, both are negligible compared with the preceding noise bound
once $n\ge n_0$. The reverse triangle inequality therefore proves
the ordinary bound in \eqref{eq:mls-lower-small-regime}. Taking the minimum
of the positive constants obtained in the five cases completes the proof.
\end{proof}

\section{Proofs for noise-driven splits}
\label{app:noise-driven-splits}

This appendix contains the calculations and proof for
Section~\ref{sec:split-reliability-lower-bound}. The stochastic argument
starts from the same ordered Gaussian noise walk used in
Appendix~\ref{app:cart-mls-lower-bound}. Here, however, the CART criterion
depends on its centered and variance-standardized bridge, rather than
directly on its endpoint averages.

\subsection{Calculation for Example~\ref{ex:relu}}
\label{app:relu-end-cut-calculation}

We verify the scaling claimed in Example~\ref{ex:relu}. Let
$X\sim\mathrm{Unif}[0,1]$ and $f(x)=(2x-1)_+$. Consider a cell
$A_a=[0,a]$ with $a>1/2$, and write $t=a-1/2$. Then the nonconstant part of
$f$ on $A_a$ is the interval $[1/2,a]$, whose length is $t$.
Using the definition $V(A)=\P(A)\Var[f(X)\mid X\in A]$, we have
\begin{align*}
V(A_a)
&=
\int_{1/2}^a (2x-1)^2\,dx
-
\frac{1}{a}\left(\int_{1/2}^a (2x-1)\,dx\right)^2  \\
&=
\frac{4}{3}t^3-\frac{t^4}{a}.
\end{align*}
Thus, as $t\downarrow0$,
\begin{equation}\label{eq:relu-V-tail}
    V(A_a)\asymp t^3,
    \qquad\text{or equivalently}\qquad
    t\asymp V(A_a)^{1/3}.
\end{equation}

We next locate the population-optimal split on $A_a$. Since
\[
\Delta(A_a,1,b)
=V(A_a)-V([0,b])-V([b,a]),
\]
maximizing the population impurity decrease is equivalent to minimizing the
post-split residual variation $V([0,b])+V([b,a])$. Splits with
$b\le 1/2$ leave the entire sloped tail in the right child. For such splits,
\[
V([0,b])+V([b,a])
=V([b,a])
=\frac{4}{3}t^3-\frac{t^4}{a-b}
\ge \frac{1}{3}t^3.
\]
Now consider splits inside the sloped part, $b=1/2+s$ with $0\le s\le t$.
The left child contains a flat part and a sloped tail of length $s$, while the
right child is fully linear on an interval of length $t-s$. Hence
\begin{align*}
V([0,b])+V([b,a])
&=
\left(\frac{4}{3}s^3-\frac{s^4}{1/2+s}\right)
+\frac{(t-s)^3}{3}.
\end{align*}
Writing $s=ut$ gives
\[
V([0,b])+V([b,a])
=t^3\left\{\frac{4}{3}u^3+\frac{1}{3}(1-u)^3\right\}
 +O(t^4),
\]
uniformly over $u\in[0,1]$. The leading term is strictly minimized at
$u=1/3$. Therefore any population-optimal split satisfies
$s/t\to1/3$ as $t\downarrow0$, and in particular
\[
    a-b=t-s\asymp t\asymp V(A_a)^{1/3}.
\]
Because $\Vol(A_a)=a\ge1/2$, this also gives
\[
    \frac{\Vol([b,a])}{\Vol(A_a)}
    \asymp V(A_a)^{1/3}.
\]
Thus the split can have a vanishing child--parent length ratio even though the
smaller child still has length of order $V(A_a)^{1/3}$.

\subsection{Proof of Theorem~\ref{thm:noise-driven-split}}

The proof follows the same logic as the constant-signal construction of \citet{cattaneo2025honest}, but with the deterministic component controlled by a weak-signal certificate. Conditional on the design points in an interval $A_n$, the empirical impurity decrease along a candidate split can be decomposed into a deterministic population contribution and a stochastic term coming from the noise. Under \eqref{eq:weak-signal-variation}, local reverse Poincar\'e control and the smallness of $V(A_n)=\P(A_n)\Var[f^*(X)\mid X\in A_n]$ make the deterministic contribution uniformly smaller than the extreme stochastic fluctuation over the admissible split locations.

Once the deterministic drift is negligible, the maximizing split behaves like the maximizer of a noise process indexed by the in-cell order statistics. Standard order-statistic and extreme-value arguments then imply that, with probability bounded away from zero, the maximizer falls in either of the near-endpoint ranges
\[
\{m^a<\iota<m^b\}
\qquad\text{or}\qquad
\{m-m^b<\iota<m-m^a\}.
\]
This yields the nonvanishing lower bound in Theorem~\ref{thm:noise-driven-split}. The result is therefore a lower bound on the reliability of the greedy split itself, rather than on the bias-variance performance of the final leaf average.

We use the following margin version of the Gaussian boundary-maximizer
argument. It is stated separately to make explicit the perturbation
allowed in the proof below.

\begin{lemma}[Separated maximizer of a Gaussian CART bridge]
\label{lem:separated-gaussian-cart-bridge}
Let $\eta_1,\eta_2,\ldots$ be i.i.d.\ standard Gaussian random
variables and, for $1\le k<m$, define
\[
    G_m(k)
    :=
    \frac{
        m^{-1/2}\sum_{i=1}^k\eta_i
        -
        (k/m)m^{-1/2}\sum_{i=1}^m\eta_i
    }{
        \sqrt{(k/m)(1-k/m)}
    }.
\]
Fix $0<a<b<1$, and let
\[
    I_m^L
    :=
    \{k\in[m-1]:m^a<k<m^b\},
    \qquad
    I_m^R
    :=
    \{k\in[m-1]:m-m^b<k<m-m^a\}.
\]
Let $(q_m)_{m\ge3}$ be any deterministic nonnegative sequence
satisfying
\[
    q_m
    =
    o\bigl((\log\log m)^{-1/2}\bigr).
\]
Then, for each $D\in\{L,R\}$,
\begin{equation}
\label{eq:separated-gaussian-cart-bridge}
    \liminf_{m\to\infty}
    \P
    \left\{
        \max_{k\in I_m^D}|G_m(k)|
        >
        \max_{k\in[m-1]\setminus I_m^D}|G_m(k)|
        +
        q_m
    \right\}
    \ge
    \frac{b-a}{2e}.
\end{equation}
\end{lemma}

\begin{proof}
It suffices to prove the result for $D=L$, since the case $D=R$
follows by reversing the Gaussian sequence. We follow the univariate
argument in the proof of Theorem~SA-1 of \citet{cattaneo2025honest}.

The strong approximation used therein couples the Gaussian CART bridge
and the corresponding standardized Brownian bridge, uniformly over the
target index set and its complement, with error
\[
    \varepsilon_m
    =
    o_{\P}\bigl((\log\log m)^{-1/2}\bigr).
\]
Choose a deterministic sequence
$r_m=o((\log\log m)^{-1/2})$ such that
$\P(\varepsilon_m\le r_m)\to1$. Under the Brownian-bridge-to-O-U
transformation, the threshold argument in equations
(SA-11)--(SA-13) of that proof remains valid with the original margin
$2r_m$ replaced by
\[
    2r_m+q_m
    =
    o\bigl((\log\log m)^{-1/2}\bigr).
\]
Indeed, this perturbation changes the normalized Darling--Erdos
parameter by
\[
    (2r_m+q_m)\sqrt{2\log\log m}=o(1).
\]
Consequently, the Gaussian correlation inequality and the corrected
Darling--Erdos limit yield exactly the same lower bound as in
equation (SA-15):
\[
    \frac{b-a}{2}
    \left(1-\frac{b-a}{2}\right)^{2/(b-a)-1}
    \ge
    \frac{b-a}{2e}.
\]
Transferring the conclusion back through the coupling and using
$\P(\varepsilon_m\le r_m)\to1$ proves
\eqref{eq:separated-gaussian-cart-bridge}.
\end{proof}

\begin{proof}[Proof of Theorem~\ref{thm:noise-driven-split}]
Write $A=A_n$ and $m=m_n=N(A_n)$. Let $p_{\min},p_{\max}$ satisfy
$0<p_{\min}\le p_X\le p_{\max}<\infty$, so
$p_{\min}\Vol(A)\le\P(A)\le p_{\max}\Vol(A)$.
The proof has three steps: a bridge representation of the empirical
CART split, uniform negligibility of the signal bridge at the
Darling--Erd\H{o}s scale, and a perturbation transfer to the noise-only
argmax result of \citet{cattaneo2025honest}.

\medskip\noindent\textit{Step 1: Bridge representation.}
Let $X_A^{(1)}\le\cdots\le X_A^{(m)}$ be the order statistics of the
sample points in $A$, and let $Y_A^{(i)}$, $\xi_A^{(i)}$,
$f_A^{(i)}=f^*(X_A^{(i)})$ be the corresponding responses, noise, and
signal values. Set
\[
    S(k)=\sum_{i=1}^k Y_A^{(i)},\qquad
    S_\xi(k)=\sum_{i=1}^k \xi_A^{(i)},\qquad
    S_f(k)=\sum_{i=1}^k f_A^{(i)}.
\]
Thus $S_\xi$ is the ordered Gaussian random walk, and the corresponding
centered and standardized bridge is obtained by applying the transformation
below to $S_\xi$. A direct computation gives
\[
    n\,\widehat\Delta(A,1,X_A^{(k)})
    =
    Z(k)^2,
    \qquad
    Z(k):=
    \frac{m^{-1/2}S(k)-(k/m)\,m^{-1/2}S(m)}{\sqrt{(k/m)(1-k/m)}},
\]
so the CART split index satisfies
$\widehat\iota\in\arg\max_{1\le k<m}|Z(k)|$.
Defining $Z_\xi$ and $Z_f$ by replacing $S$ with $S_\xi$ and $S_f$,
respectively, in the same formula yields the additive decomposition
$Z=Z_\xi+Z_f$.

\medskip\noindent\textit{Step 2: The signal bridge is negligible.}
We first prove that the signal bridge is uniformly negligible on a high-probability event. Let $\mathcal E_n^{\mathrm{cnt}}$ be the event in Lemma~\ref{lemma:uniform_cell_count} with $r=d=1$ and $u=\log n$. Then $\P\{\mathcal E_n^{\mathrm{cnt}}\}\ge 1-\frac{1}{n}.$
Since $n\P(A_n)/\log n\to\infty$, the threshold condition in Lemma~\ref{lemma:uniform_cell_count} holds for all sufficiently large $n$. Hence, on $\mathcal E_n^{\mathrm{cnt}}$,
\[
    \frac{1}{2}
    \le\frac{n\P(A)}{m}\le2,
    \qquad
    \frac{m}{\log n}\to\infty .
\]

Using $k(m-k)/m\le m/4$ and the elementary bound
$|\frac{S_f(k)}{k}-\frac{S_f(m)-S_f(k)}{m-k}|\le\Osc(A)$,
\begin{equation}\label{eq:Zf-osc-bound}
    \max_{1\le k<m} Z_f(k)^2
    \;\le\;
    \frac{m}{4}\,\Osc(A)^2.
\end{equation}

By the LRP inequality applied to $A$ with $w=\E[f^*(X)\mid X\in A]$
and $p_X\ge p_{\min}$,
\begin{align*}
    \Osc(A)^2
    &\le
    \frac{R}{\Vol(A)}
    \int_A|f^*(x)-w|^2\,dx \\
    &\le
    \frac{R\P(A)}{p_{\min}\Vol(A)}
    \Var[f^*(X)\mid X\in A] \\
    &\le
    \frac{R\,p_{\max}}{p_{\min}}
    \Var[f^*(X)\mid X\in A].
\end{align*}
Combining this bound with \eqref{eq:Zf-osc-bound}, and using $m\le 2n\P(A)$ on $\mathcal E_n^{\mathrm{cnt}}$, we obtain
\begin{equation}\label{eq:rn-LRP}
    \max_{1\le k<m}Z_f(k)^2
    \le
    \frac{Rp_{\max}}{2p_{\min}}\,
    nV(A_n)\coloneqq r_{n}^2.
\end{equation}
By the weak-signal condition \eqref{eq:weak-signal-variation}, $r_{n}^2=o((\log\log n)^{-1})$.
We then obtain, on $\mathcal E_n^{\mathrm{cnt}}$,
\begin{equation}\label{eq:Zf-negligible}
    \max_{1\le k<m}|Z_f(k)|\le r_n,
    \qquad
    r_n=o\bigl((\log\log n)^{-1/2}\bigr).
\end{equation}

\medskip\noindent\textit{Step 3: Perturbation transfer.}
We now transfer the noise-only split-location result to the full response
process. The deterministic principle is simple: if the maximum of the noise
bridge on an index set exceeds its maximum on the complement by more than
$2\|Z_f\|_\infty$, then the full bridge is maximized on that index set.
For any $I\subset\{1,\ldots,m-1\}$, define
\[
    M(I)=\max_{k\in I}|Z(k)|,\qquad
    M_\xi(I)=\max_{k\in I}|Z_\xi(k)|.
\]
By \eqref{eq:Zf-negligible}, on $\mathcal E_n^{\mathrm{cnt}}$,
\begin{equation}\label{eq:restricted-max-perturb-thm35}
    |M(I)-M_\xi(I)|\le r_n
    \qquad
    \text{for every }I\subset\{1,\ldots,m-1\}.
\end{equation}

Fix $0<a<b<1$ and define the left near-boundary index set $I_m^L=\{k\in\{1,\ldots,m-1\}:m^a<k<m^b\}.$
Let
\[
    \mathcal G_n^L
    =
    \{M_\xi(I_m^L)>M_\xi((I_m^L)^c)+2r_n\}.
\]
Conditional on the design points, $Z_\xi$ is the same standardized Gaussian
noise bridge as in Lemma~\ref{lem:separated-gaussian-cart-bridge}, multiplied
by $\sigma$, and its conditional law depends on the design only through
$m=m_n$. Set $q_n=2r_n/\sigma$. On $\mathcal E_n^{\mathrm{cnt}}$,
$m_n\to\infty$ and, since $m_n\le n$,
\[
q_n=o\bigl((\log\log m_n)^{-1/2}\bigr).
\]
Lemma~\ref{lem:separated-gaussian-cart-bridge} therefore gives
\begin{equation}\label{eq:noise-left-thm35}
    \liminf_{n\to\infty}\P\{\mathcal G_n^L\}
    \ge
    \frac{b-a}{2e}.
\end{equation}
To justify the random index explicitly, the sequential formulation of the
lemma implies the following uniform version: for every $\epsilon>0$, there
are $\eta>0$ and $M<\infty$ such that, for every $m\ge M$ and every
$0\le q\le\eta(\log\log m)^{-1/2}$, the probability in
\eqref{eq:separated-gaussian-cart-bridge} is at least
$(b-a)/(2e)-\epsilon$. Otherwise, a violating sequence would contradict
the lemma. Conditional on the design, apply this bound with
$m=m_n$ and $q=q_n$, then use
$\P\{\mathcal E_n^{\mathrm{cnt}}\}\to1$. This proves
\eqref{eq:noise-left-thm35}. The same conditioning argument applies to
the right boundary range.

On $\mathcal E_n^{\mathrm{cnt}}\cap\mathcal G_n^L$, \eqref{eq:restricted-max-perturb-thm35} gives
\[
    M(I_m^L)
    \ge
    M_\xi(I_m^L)-r_n
    >
    M_\xi((I_m^L)^c)+r_n
    \ge
    M((I_m^L)^c).
\]
Therefore $\widehat\iota\in I_m^L$ on $\mathcal E_n^{\mathrm{cnt}}\cap\mathcal G_n^L$. Hence
\begin{equation}
    \P\{m^a<\widehat\iota<m^b\}
    \ge
    \P\{\mathcal E_n^{\mathrm{cnt}}\cap\mathcal G_n^L\}
    \ge
    \P\{\mathcal G_n^L\}-\P\{(\mathcal E_n^{\mathrm{cnt}})^c\}.
\end{equation}
Since $\P\{(\mathcal E_n^{\mathrm{cnt}})^c\}\le 1/n$ and \eqref{eq:noise-left-thm35} holds, we obtain
\begin{equation}
    \liminf_{n\to\infty}
    \P\{m^a<\widehat\iota<m^b\}
    \ge
    \frac{b-a}{2e}.
\end{equation}
Repeating the perturbation argument on $\mathcal E_n^{\mathrm{cnt}}\cap\mathcal G_n^R$ yields
\begin{equation}
    \liminf_{n\to\infty}
    \P\{m-m^b<\widehat\iota<m-m^a\}
    \ge
    \frac{b-a}{2e}.
\end{equation}
This completes the proof.
\end{proof}

\section{General local H\"older function class}
\label{appendix:proof-of-holder}


\begin{definition}[Bounded local anisotropic H\"older regularity]
\label{def:local-holder-general}
Let $\alpha_j,\gamma_j:[0,1]\to(0,1]$, $j\in[d]$, be measurable
functions satisfying $0<\alpha_{\min}\le\alpha_j(t)\le\alpha_{\max}\le1$
for all $j$ and $t$. Let $L\ge1$ be a H\"older radius. We say that a
bounded regression function $f^*$ is \emph{locally anisotropic
H\"older with exponents $\{\alpha_j\}_{j=1}^d$, locality radii
$\{\gamma_j\}_{j=1}^d$, and radius $L$} if, for all
$\bx,\bx'\in[0,1]^d$ such that $|x_j-x_j'|<\gamma_j(x_j)$ for all
$j\in[d]$,
\begin{equation}\label{eq:general-local-holder}
|f^*(\bx)-f^*(\bx')|
\;\le\;
L\sum_{j=1}^d |x_j-x_j'|^{\alpha_j(x_j)} .
\end{equation}
\end{definition}

\begin{lemma}[The locality radii can be taken to be one]
\label{lem:radius-one}
Let $f^*$ satisfy Definition~\ref{def:local-holder-general} with
exponents $\{\alpha_j\}_{j=1}^d$, locality radii
$\{\gamma_j\}_{j=1}^d$, and radius $L$. Assume in addition that
\begin{enumerate}
\item[(i)] each $\alpha_j$ is H\"older continuous: there exist
  $\nu\in(0,1]$ and $L_\alpha\ge0$ such that
  $\lvert\alpha_j(t)-\alpha_j(t')\rvert\le
  L_\alpha\lvert t-t'\rvert^{\nu}$
  for all $t,t'\in[0,1]$ and $j\in[d]$;
\item[(ii)] the locality radii are uniformly bounded away from zero:
  there exists $c\in(0,1]$ such that $\gamma_j(t)\ge c$ for all
  $t\in[0,1]$ and $j\in[d]$.
\end{enumerate}
Then, for \emph{all} $\bx,\bx'\in[0,1]^d$,
\[
|f^*(\bx)-f^*(\bx')|
\;\le\;
L'\sum_{j=1}^d |x_j-x_j'|^{\alpha_j(x_j)},
\qquad
L' \;:=\; \Bigl(\frac{2}{c}\Bigr)^{1+L_\alpha} L .
\]
That is, $f^*$ is locally anisotropic H\"older with the same
exponents, radius $L'$, and locality radii $\gamma_j\equiv 1$; in
particular, it satisfies the unrestricted form of the local H\"older
condition used in the main analysis, with $L$ replaced by $L'$.
\end{lemma}

\begin{proof}
The proof consists of a reduction to one dimension followed by a
chaining argument; assumption (i) enters only through the two
exponent-comparison displays in the proof of the Claim below, and
assumption (ii) only through the admissibility of the chain steps and
the bound $m\le 2/c$ on their number.

For the one-dimensional reduction we write, for measurable functions
$\alpha:[0,1]\to(0,1]$ and $\gamma:[0,1]\to(0,1]$ and a constant
$M>0$,
\begin{equation}\label{eq:1d-class}
\begin{aligned}
\mathcal{H}_1(\alpha,\gamma,M)
\coloneqq\bigl\{g:[0,1]\to\mathbb{R}:\;&
\lvert g(t')-g(t)\rvert
\le M\lvert t-t'\rvert^{\alpha(t)}\\
&\text{for all $t,t'\in[0,1]$ with
$\lvert t-t'\rvert<\gamma(t)$}\bigr\}.
\end{aligned}
\end{equation}

\medskip
\noindent\textit{Step 1: Reduction to one dimension.} Fix
$\bx,\bx'\in[0,1]^d$ and define the coordinate path
$\bz^{(0)}:=\bx$ and
\[
  \bz^{(j)} \;:=\; (x_1',\ldots,x_j',\,x_{j+1},\ldots,x_d),
  \qquad j\in[d],
\]
so that $\bz^{(d)}=\bx'$ and consecutive points differ only in one
coordinate. By the triangle inequality,
\begin{equation}\label{eq:coordinate-path}
  |f^*(\bx)-f^*(\bx')|
  \;\le\; \sum_{j=1}^d
  \bigl|f^*(\bz^{(j-1)})-f^*(\bz^{(j)})\bigr| .
\end{equation}
For each $j$, define the univariate section
$g_j(t):=f^*(x_1',\ldots,x_{j-1}',\,t,\,x_{j+1},\ldots,x_d)$,
$t\in[0,1]$, so that the $j$th summand in \eqref{eq:coordinate-path}
equals $|g_j(x_j)-g_j(x_j')|$. We first record that
$g_j\in\mathcal{H}_{1}(\alpha_j,\gamma_j,L)$: if
$t,t'\in[0,1]$ satisfy $|t-t'|<\gamma_j(t)$, the two points of
$[0,1]^d$ at which $g_j(t)$ and $g_j(t')$ evaluate $f^*$ agree in
every coordinate $k\neq j$, so the admissibility condition
$|z_k-z_k'|=0<\gamma_k(z_k)$ holds off coordinate $j$, and
\eqref{eq:general-local-holder} yields
$|g_j(t')-g_j(t)|\le L\,|t-t'|^{\alpha_j(t)}$, the terms with
$k\neq j$ vanishing. It therefore suffices to prove the following
one-dimensional claim; applying it to $g_j$ with base point $x_j$ and
increment $x_j'-x_j$, and summing \eqref{eq:coordinate-path} over $j$,
gives the lemma with the dimension-free constant
$L'=(2/c)^{1+L_\alpha}L$.

\medskip
\noindent\emph{Claim.} Let
$g\in\mathcal{H}_{1}(\alpha,\gamma,L)$ with
$\gamma(\cdot)\ge c$ on $[0,1]$ and
$\lvert\alpha(t)-\alpha(t')\rvert\le L_\alpha\lvert t-t'\rvert^{\nu}$
for all $t,t'\in[0,1]$. Then, for every $x\in[0,1]$ and every $h\neq0$
with $x+h\in[0,1]$,
\[
  \lvert g(x+h)-g(x)\rvert
  \;\le\; \Bigl(\frac{2}{c}\Bigr)^{1+L_\alpha} L\,
  \lvert h\rvert^{\alpha(x)} .
\]

\medskip
\noindent\emph{Proof of the Claim.} If $\lvert h\rvert<\gamma(x)$, the
bound holds with constant $L\le (2/c)^{1+L_\alpha}L$, directly from
\eqref{eq:1d-class}. Assume therefore $\lvert h\rvert\ge\gamma(x)$.
Without loss of generality, we assume $h>0$: the reflected function
$\tilde g(u):=g(1-u)$ belongs to
$\mathcal{H}_1\bigl(\alpha(1-\cdot),\gamma(1-\cdot),L\bigr)$,
whose exponent and radius functions satisfy (i) and (ii) with the same
$(\nu,L_\alpha,c)$, and
$g(x+h)-g(x)=\tilde g\bigl((1-x)+\lvert h\rvert\bigr)-\tilde g(1-x)$
with $\alpha\bigl(1-(1-x)\bigr)=\alpha(x)$, so the case $h<0$ follows
from the case $h>0$ applied to $\tilde g$.

Let $m=\lfloor 2h/c\rfloor$. Since $c\le\gamma(x)\le h\le 1$, we have
$2\le m\le 2/c$. Let $x_i:=x+ic/2$ for $0\le i\le m-1$; from
$m\le 2h/c<m+1$, we have
\[
  c/2 \;\le\; x+h-x_{m-1} \;<\; c .
\]
By the triangle inequality, we have
\begin{align*}
|g(x+h)-g(x)|&\leq \sum_{i=1}^{m-1} |g(x_i)-g(x_{i-1})|
  + |g(x+h)-g(x_{m-1})|.
\end{align*}
The local
H\"older condition \eqref{eq:1d-class} gives
\begin{align*}
|g(x+h)-g(x)|\leq L\sum_{i=1}^{m-1}|x_i-x_{i-1}|^{\alpha(x_i)}
  + L|x+h-x_{m-1}|^{\alpha(x+h)}.
\end{align*}
Using the H\"older condition of $\alpha$ in (i), we have
\begin{align*}
|x_i-x_{i-1}|^{\alpha(x_i)}\leq
|x_i-x_{i-1}|^{\alpha(x)-L_\alpha|x_i-x|^{\nu}}
\end{align*}
and
\begin{align*}
|x+h-x_{m-1}|^{\alpha(x+h)}\leq
|x+h-x_{m-1}|^{\alpha(x)-L_\alpha h^{\nu}}.
\end{align*}
Since $|x_i-x|^{\nu}\le 1$ and $h^{\nu}\le 1$, while every increment
lies in $[c/2,1]$, we have
\begin{align*}
|x_i-x_{i-1}|^{L_\alpha|x_i-x|^{\nu}}
\;\geq\; |x_i-x_{i-1}|^{L_\alpha}
\;\geq\; (c/2)^{L_\alpha}
\end{align*}
and
\begin{align*}
|x+h-x_{m-1}|^{L_\alpha h^{\nu}}
\;\geq\; |x+h-x_{m-1}|^{L_\alpha}
\;\geq\; (c/2)^{L_\alpha}.
\end{align*}
Therefore, since every increment is at most $c\le h$ and there are
$m\le 2/c$ of them, we get
\begin{align*}
|g(x+h)-g(x)|& \leq
\Bigl(\frac{2}{c}\Bigr)^{L_\alpha}
L\sum_{i=1}^{m-1}|x_i-x_{i-1}|^{\alpha(x)}
  + \Bigl(\frac{2}{c}\Bigr)^{L_\alpha} L\,
  |x+h-x_{m-1}|^{\alpha(x)}\\
& \leq \Bigl(\frac{2}{c}\Bigr)^{L_\alpha} L m\, h^{\alpha(x)}
\;\leq\; \Bigl(\frac{2}{c}\Bigr)^{1+L_\alpha} L\, h^{\alpha(x)}.
\end{align*}
This proves the Claim.

\medskip
Finally, since $|x_j-x_j'|\le 1$ and $\alpha_j(x_j)>0$ imply
$|x_j-x_j'|^{\alpha_j(x_j)}\le1$, the conclusion of the lemma also
gives $\sup_{\bx,\bx'}|f^*(\bx)-f^*(\bx')|\le L'd$, so the boundedness
postulated in Definition~\ref{def:local-holder-general} holds
automatically.
\end{proof}

\end{document}